%% file: main.tex
\documentclass{article}

\def\isfinal{1} 
\def\useneurips{1} 

\if\useneurips1
\if\isfinal1
\usepackage[final]{neurips_2021}
\else
\usepackage{neurips_2021}
\fi

\else
\usepackage[preprint]{neurips_blank}
\fi

\usepackage{main-tex/preamble}
\input{main-tex/macros}

\begin{document}

\title{Optimal Policies Tend To Seek Power} 
\author{
Alexander Matt Turner\\
Oregon State University\\
\texttt{turneale@oregonstate.edu} 
\And
Logan Smith\\
Mississippi State University\\
\texttt{ls1254@msstate.edu} \\
\And
Rohin Shah\\
UC Berkeley\\
\texttt{rohinmshah@berkeley.edu}\\
\And 
Andrew Critch\\
UC Berkeley\\
\texttt{critch@berkeley.edu}
\And
Prasad Tadepalli\\
Oregon State University\\
\texttt{tadepall@eecs.oregonstate.edu}\\
}

\maketitle 

\input{main-tex/sections/abstract-intro}
\input{main-tex/sections/related-work}
\input{main-tex/sections/visit-distributions}
\input{main-tex/sections/opt-prob}
\input{main-tex/sections/power}
\input{main-tex/sections/symmetries}
\input{main-tex/sections/future-work-conclusion}

{\small\bibliography{AI_Safety.bib}}

\clearpage

\begin{appendices}
\input{main-tex/appendices/alternative-power}
\input{main-tex/appendices/power-nuances}
\input{main-tex/appendices/suboptimal-power}
\input{main-tex/appendices/theorem-list}
\input{main-tex/appendices/proofs}
\end{appendices}
\end{document}

%% file: main-tex/macros.tex
\renewcommand{\abs}[1]{\left|#1\right|}
\newcommand*{\prn}[1]{\left(#1\right)}
\newcommand*{\brx}[1]{\left[#1\right]}
\renewcommand{\set}[1]{\left\{#1\right\}}
\newcommand*{\defeq}{\coloneqq} 
\newcommand*{\eqdef}{\eqqcolon} 
\newcommand*{\x}{\mathbf{x}} 
\newcommand*{\av}{\mathbf{a}}
\newcommand*{\bv}{\mathbf{b}}

\newcommand*{\prnNotEmpty}[1]{\ifthenelse{\isempty{#1}}
    {}
    {\prn{#1}}
}

\newcommand*{\reals}{\mathbb{R}}
\DeclareMathOperator*{\argmax}{arg\,max}
\newcommand*{\unitvec}[1][s]{\mathbf{e}_{#1}}
\newcommand*{\permute}[1][\phi]{\mathbf{P}_{#1}} 
\newcommand*{\upsc}[1]{\text{\upshape\textsc{#1}}} 

\newcommand*{\mdp}[1][m]{{\textsc{#1dp}}}
\newcommand*{\rsd}[1][r]{{\textsc{#1sd}}}
\newcommand*{\rl}[1][r]{{\textsc{#1l}}}
\NewDocumentCommand{\iid}{}{\textsc{iid}}

\newcommand*{\St}{\mathcal{S}}
\newcommand*{\A}{\mathcal{A}}
\newcommand*{\rf}{\mathbf{r}} 
\newcommand*{\rewardSpace}{\reals^{\St}} 
\newcommand*{\rewardVS}{\reals^{\abs{\St}}} 

\newcommand{\stateEnd}{1-cycle}
\newcommand*{\terminal}{terminal}

\NewDocumentCommand{\Vf}{O{*}O{R}m}{V^{#1}_{#2} \prn{#3}} 
\NewDocumentCommand{\VfNoResize}{O{*}O{R}m}{V^{#1}_{#2} (#3)} 
\NewDocumentCommand{\VfNorm}{O{*}O{R}m}{\Vf[#1][#2,\,\text{norm}]{#3}}

\NewDocumentCommand{\OptVf}{O{R}m}{\Vf[*][#1]{#2}}

\NewDocumentCommand{\piSet}{O{}O{R,\gamma}}{\Pi^{#1}\prn{#2}}
\newcommand*{\optPi}{\piSet[*]}

\newcommand*{\average}[1][R]{\piSet[\text{avg}][#1]}

\newcommand*{\reach}[2][]{\upsc{Reach}_{#1}\prn{#2}}

\newcommand*{\Dist}{X} 
\newcommand*{\D}{\mathcal{D}} 
\DeclareMathOperator{\optSupp}{supp}
\NewDocumentCommand{\supp}{O{\D}}{\optSupp(#1)}

\NewDocumentCommand{\geqMost}{O{}O{\DSetAny}}{\geq_{\text{\upshape {most}}\ifthenelse{\isempty{#2}}{}{\text{\upshape: }#2}}^{#1}}
\NewDocumentCommand{\leqMost}{O{}O{\DSetAny}}{\leq_{\text{\upshape {most}}\ifthenelse{\isempty{#2}}{}{\text{\upshape: }#2}}^{#1}}
\newcommand*{\distSet}{\mathfrak{D}}

\NewDocumentCommand{\Diid}{O{\Dist}}{\D_{#1\text{-}\iid}}
 
\newcommand*{\DSetBCiid}{\distSet_{\textsc{c/b/}\iid}} 
\newcommand*{\Dany}{\D_{\text{any}}}
\newcommand*{\DSetAny}{\distSet_\text{any}}
\newcommand*{\Dbd}{\D_{\text{bound}}} 
\newcommand*{\DSetBd}{\distSet_\text{bound}}
\newcommand*{\Dcont}{\D_\text{cont}}

\NewDocumentCommand{\vavg}{O{s,\gamma}O{\Dbd}}{\OptVf[#2]{#1}}
\NewDocumentCommand{\vavgNoResize}{O{s,\gamma}O{\Dbd}}{V^*_{#2}(#1)} 
\NewDocumentCommand{\pwrNoDist}{}{\upsc{Power}} 
\NewDocumentCommand{\pwr}{O{}O{\Dbd}}{\upsc{Power}_{#2} \prnNotEmpty{#1}}
\NewDocumentCommand{\pwrNoResize}{O{}O{\Dbd}}{\upsc{Power}_{#2} (#1)}

\DeclareMathOperator*{\Prb}{\mathbb{P}}
\newcommand*{\prob}[2][]{\Prb_{#1}\prn{#2}}
\NewDocumentCommand{\optprob}{O{\D}O{}m}{\noexpandarg\exploregroups\Prb\IfSubStr{\sim}{#1}{}{\nolimits}_{#1}^{#2}\prn{#3}}
\NewDocumentCommand{\avgprob}{O{\D}m}{\optprob[#1]{#2, \text{\upshape average}}}
\NewDocumentCommand{\greedyprob}{O{\D}m}{\optprob[#1]{#2, \text{\upshape greedy}}}

\DeclareMathOperator*{\opE}{\mathbb{E}}
\newcommand*{\E}[2]{\opE_{#1}\brx{#2}} 
\newcommand*{\EX}{\opE\brx{\Dist}}
\NewDocumentCommand{\Edraws}{O{\Dist}m}{\opE\brx{\max \text{ of } #2 \text{ draws from } #1}}
\NewDocumentCommand{\dF}{O{\rf}O{F}}{\dif #2(#1)}

\newcommand*{\f}{\mathbf{f}} 
\newcommand*{\fpi}[2][\pi]{\f^{
\ifthenelse{\isempty{#1}}{}{#1} 
\ifthenelse{\NOT\isempty{#1} \AND \NOT\isempty{#2}}{,}{} 
\ifthenelse{\isempty{#2}}{}{#2}}} 

\DeclareMathOperator{\Fop}{\mathcal{F}}
\NewDocumentCommand{\F}{O{}}{\Fop\prnNotEmpty{#1}}
\newcommand*{\FRestrictAction}[3][s]{\F(#1 \mid \pi(#2)=#3)}

\newcommand*{\ND}[1]{\upsc{ND}\prn{#1}}
\NewDocumentCommand{\phelper}{mO{\D'}}{p_{#2}\prn{#1}}
\DeclareMathOperator{\Fndop}{\mathcal{F}_{nd}} 
\NewDocumentCommand{\Fnd}{O{}}{\Fndop\prnNotEmpty{#1}}
\newcommand*{\FndRestrictAction}[2]{\Fnd(s \mid \pi(#1)=#2)} 
  
\newcommand*{\NormF}[1]{\upsc{Norm}\prn{#1}}

\newcommand*{\dbf}{\mathbf{d}} 
\newcommand*{\RSD}[1][s]{\upsc{RSD}\prn{#1}} 
\newcommand*{\RSDnd}[1][s]{\upsc{RSD}{\text{\upshape\textsubscript{nd}}}\prn{#1}}
\newcommand*{\RSDndNoResize}[1][s]{\upsc{RSD}{\text{\upshape\textsubscript{nd}}}(#1)}

\newcount\colveccount
\newcommand*\colvec[1]{
        \global\colveccount#1
        \protect\begin{pmatrix}
        \colvecnext
}
\def\colvecnext#1{
        #1
        \global\advance\colveccount-1
        \ifnum\colveccount>0
                \\
                \expandafter\colvecnext
        \else
                \protect\end{pmatrix}
        \fi
}

\newcommand*{\pol}[1][R,\gamma]{\text{pol}\prnNotEmpty{#1}}
\NewDocumentCommand{\pwrPol}{O{\pol[]}O{\Dbd}m}{\pwrNoDist^{#1}_{#2} \prnNotEmpty{#3}}

\newcommand*{\lone}[1]{\left \lVert#1\right \rVert_1}
\newcommand*{\linfty}[1]{\left \lVert#1\right \rVert_\infty}

\newcommand*{\geom}[1][1]{\frac{#1}{1-\gamma}}
\newcommand*{\half}{\frac{1}{2}}

\newcommand*{\indic}[1]{\mathbbm{1}_{#1}} 
\newcommand*{\inv}{^{-1}}

\newcommand*{\mdpPermGroup}{S_{\abs{\St}}}

\NewDocumentCommand{\orbi}{O{\Dbd}O{\abs{\St}}}{S_{#2}\cdot #1}

\definecolor{blue}{rgb}{.25,.45,.75}
\definecolor{purple}{rgb}{.4667,.1529,.2118}
\definecolor{green}{rgb}{.294,.624,.294}
\definecolor{red}{rgb}{.75,.25,.25}
\newcommand*{\col}[2]{{\color{#1}#2}}

\newtheorem*{cor-no-num}{Corollary}

\theoremstyle{definition}

\newtheorem*{remark}{Remark}

\newtheorem*{thm*}{Theorem} 

\newcommand*{\dauTemplate}[4][\D]{d_{#1}^{#3}\ifthenelse{\isempty{#2}}
    {}
    {\prn{#2 \mid #4}}}

\newcommand*{\R}[1][]{\mathcal{R}_\text{#1}} 

\ifluatex
\newcommand*{\start}{\text{\emoji{house}}}
\newcommand*{\startGraph}{\text{\emoji{house}}}
\newcommand*{\sink}{\text{\emoji{video-game}}}
\newcommand*{\topleft}{\text{\emoji{evergreen-tree}}}
\newcommand*{\farleft}{\text{\emoji{palm-tree}}}
\newcommand*{\closeleft}{\text{\emoji{sunflower}}}
\newcommand*{\topright}{\text{\emoji{star}}}
\newcommand*{\farright}{\text{\emoji{sun}}}
\newcommand*{\closeright}{\text{\emoji{crescent-moon}}}
\else
\newcommand*{\start}{\star}
\newcommand*{\startGraph}{\small\bigstar}
\newcommand*{\sink}{\varnothing}
\newcommand*{\topleft}{\ell_{\nwarrow}}
\newcommand*{\farleft}{\ell_{\swarrow}}
\newcommand*{\closeleft}{\ell_{\triangleleft}}
\newcommand*{\topright}{r_{\nearrow}}
\newcommand*{\farright}{r_{\searrow}}
\newcommand*{\closeright}{r_\triangleright}
\fi

\newcommand*{\leftA}{\texttt{left}}
\newcommand*{\rightA}{\texttt{right}}

\def\centerarc[#1](#2)(#3)(#4:#5:#6){ \draw[#1] (#2)++(#3) arc (#4:#5:#6); }
\NewDocumentCommand{\arcHelp}{O{0pt}O{0pt}O{.35cm}O{->}mm}{ 

\ifthenelse{\equal{#5}{N}}{\centerarc[#4](#6.east)(#1,#2)(-50:240:#3)}{}
\ifthenelse{\equal{#5}{NE}}{\centerarc[#4](#6.east)(-1pt++#1,.5pt++#2)(-95:200:#3)}{}
\ifthenelse{\equal{#5}{E}}{\centerarc[#4](#6.east)(#1,-5pt++#2)(-130:145:#3)}{}
\ifthenelse{\equal{#5}{SE}}{\centerarc[#4](#6.east)(-9pt++#1,-7.5pt++#2)(-180:110:#3)}{}
\ifthenelse{\equal{#5}{S}}{\centerarc[#4](#6.south)(7pt++#1,-1pt++#2)(40:-235:#3)}{}
\ifthenelse{\equal{#5}{SW}}{\centerarc[#4](#6.west)(8.5pt++#1,-8pt++#2)(0:-290:#3)}{}
\ifthenelse{\equal{#5}{W}}{\centerarc[#4](#6.west)(#1,-6pt++#2)(-35:-325:#3)}{}
\ifthenelse{\equal{#5}{NW}}{\centerarc[#4](#6.west)(-1pt++#1,3.5pt++#2)(250:-15:#3)}{}
}

\newcommand{\kemdash}{\kern0.01pt---\kern0.01pt}

%% file: main-tex/sections/abstract-intro.tex
\begin{abstract}
Some researchers speculate that intelligent reinforcement learning ({\rl}) agents would be incentivized to seek resources and power in pursuit of the objectives we specify for them. Other researchers point out that {\rl} agents need not have human-like power-seeking instincts. To clarify this discussion, we develop the first formal theory of the statistical tendencies of optimal policies. In the context of Markov decision processes ({\mdp}s), we prove that certain environmental symmetries are sufficient for optimal policies to tend to seek power over the environment. These symmetries exist in many environments in which the agent can be shut down or destroyed. We prove that in these environments, most reward functions make it optimal to seek power by keeping a range of options available and, when maximizing average reward, by navigating towards larger sets of potential terminal states.
\end{abstract}

\section{Introduction} 
\citet{omohundro_basic_2008,bostrom_superintelligence_2014,russell_human_2019} hypothesize that highly intelligent agents tend to seek power in pursuit of their goals. Such power-seeking agents might gain power over humans. Marvin Minsky imagined that an agent tasked with proving the Riemann hypothesis might rationally turn the planet—along with everyone on it—into computational resources \citep{russell_artificial_2009}. However, another possibility is that such concerns simply arise from the anthropomorphization of AI systems \citep{lecun_dont_2019,instrumental,steven_2020,mitchell2021ai}.

We clarify this discussion by grounding the claim that highly intelligent agents will tend to seek power. In \cref{sec:action-prob}, we identify optimal policies as a reasonable formalization of ``highly intelligent agents.''\footnote{This paper assumes that reward functions reasonably describe a trained agent's goals. Sometimes this is roughly true (\eg{} chess with a sparse victory reward signal) and sometimes it is not true. \citet{rewardNotOpt} argues that capable {\rl} algorithms do not necessarily train policy networks which are best understood as optimizing the reward function itself. Rather, they point out that—especially in policy gradient approaches—reward provides gradients to the network and thereby modifies the network's generalization properties, but doesn't ensure the agent generalizes to ``robustly optimizing reward'' off of the training distribution.} Optimal policies ``tend to'' take an action when the action is optimal for most reward functions. We expect future work to translate our theory from optimal policies to learned, real-world policies. 

\Cref{sec:power} defines ``power'' as the ability to achieve a wide range of goals. For example, ``money is power,'' and money is instrumentally useful for many goals. Conversely, it's harder to pursue most goals when physically restrained, and so a physically restrained person has little power. An action ``seeks power'' if it leads to states where the agent has higher power. 

We make no claims about when large-scale AI power-seeking behavior could become plausible. Instead, we consider the theoretical consequences of optimal action in {\mdp}s. \Cref{sec:symmetries} shows that power-seeking tendencies arise not from anthropomorphism, but from certain graphical symmetries present in many {\mdp}s. These symmetries automatically occur in many environments where the agent can be shut down or destroyed, yielding broad applicability of our main result (\cref{rsdIC}).

%% file: main-tex/sections/related-work.tex
\section{Related work}\label{sec:prior}
An action is \emph{instrumental to an objective} when it helps achieve that objective. Some actions are instrumental to a range of objectives, making them \emph{convergently instrumental}. The claim that power-seeking is convergently instrumental is an instance of the \emph{instrumental convergence thesis}:
\begin{quote}
    Several instrumental values can be identified which are convergent in the sense that their attainment would increase the chances of the agent's goal being realized for a wide range of final goals and a wide range of situations, implying that these instrumental values are likely to be pursued by a broad spectrum of situated intelligent agents \citep{bostrom2012superintelligent}.
\end{quote}

For example, in Atari games, avoiding (virtual) death is instrumental  for both completing the game and for optimizing curiosity \citep{burda2018largescale}. Many AI alignment researchers hypothesize that most advanced AI agents will have concerning instrumental incentives, such as resisting deactivation~\citep{soares_corrigibility_2015,milli2017should,off_switch,carey_incorrigibility_2017} and acquiring resources~\citep{benson-tilsen_formalizing_2016}. 

We formalize power as the ability to achieve a wide variety of goals. Appendix~\ref{existing} demonstrates that our formalization returns intuitive verdicts in situations where information-theoretic empowerment does not \citep{salge_empowermentintroduction_2014}. 

Some of our results relate the formal power of states to the structure of the environment. \citet{foster_structure_2002, drummond_composing_1998,sutton_horde:_2011,schaul_universal_2015} note that value functions encode important information about the environment, as they capture the agent's ability to achieve different goals. \citet{turner2020conservative} speculate that a state's optimal value correlates strongly across reward functions. In particular, \citet{schaul_universal_2015} learn regularities across value functions, suggesting that some states are valuable for many different reward functions (\ie{} powerful). \citet{menache2002q} identify and navigate towards convergently instrumental bottleneck states.

We are not the first to study convergence of behavior, form, or function. In economics, turnpike theory studies how certain paths of accumulation tend to be optimal \citep{mckenzie_turnpike_1976}. In biology, convergent evolution occurs when similar features (\eg{} flight) independently evolve in different time periods \citep{reece2011campbell}. Lastly, computer vision networks reliably learn \eg{} edge detectors, implying that these features are useful for a range of tasks \citep{olah2020zoom}. 

%% file: main-tex/sections/visit-distributions.tex
\section{State visit distribution functions quantify the agent's available options}\label{sec:visit-dists}
\begin{figure}[!ht]\centering
     \includegraphics{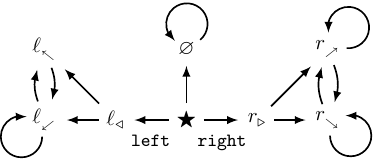}
    \caption{$\farleft$ is a {\stateEnd}, and $\sink$ is a terminal state. Arrows represent deterministic transitions induced by taking some action $a\in\A$. Since the {\rightA} subgraph contains a copy of the {\leftA} subgraph, \cref{graph-options} will prove that more reward functions have optimal policies which go {\rightA} than which go {\leftA} at state $\start$, and that such policies seek power—both intuitively, and in a reasonable formal sense.} 
    \label{fig:case-study}
\end{figure} 

We clarify the power-seeking discussion by proving what optimal policies usually look like in a given environment. We illustrate our results with a simple case study, before explaining how to reason about a wide range of {\mdp}s. Appendix \ref{app:contrib} lists  {\mdp} theory contributions  of independent interest, appendix \ref{app:thms} lists definitions and theorems, and appendix \ref{app:proofs} contains the proofs.

\begin{restatable}[Rewardless {\mdp}]{definition}{rewardless}
$\langle \mathcal{S}, \mathcal{A}, T \rangle$ is a rewardless {\mdp} with finite state and action spaces $\mathcal{S}$ and $\mathcal{A}$, and stochastic transition function $T: \St \times \A \to \Delta(\St)$. We treat the discount rate $\gamma$ as a variable with domain $[0,1]$.
\end{restatable} 

\begin{restatable}[{\stateEnd} states]{definition}{oneCycState}\label{def:onecyc}
Let $\unitvec[s]\in \reals^{\abs{\St}}$ be the standard basis vector for state $s$, such that there is a 1 in the entry for state $s$ and 0 elsewhere. State $s$ is a \emph{\stateEnd} if $\exists a\in\A: T(s,a)=\unitvec$. State $s$ is a \emph{{\terminal} state} if $\forall a\in\A:T(s,a)=\unitvec$.
\end{restatable}

Our theorems apply to stochastic environments, but we present a deterministic case study for clarity. The environment of \cref{fig:case-study} is small, but its structure is rich. For example, the agent has more ``options'' at $\start$  than at the terminal state $\sink$. Formally, $\start$ has more \emph{visit distribution functions} than $\sink$ does. 

\begin{restatable}[State visit distribution \citep{sutton_reinforcement_1998}]{definition}{DefVisit}\label{def:visit}
$\Pi\defeq \A^\St$, the set of stationary deterministic policies. The \emph{visit distribution} induced by following policy $\pi$ from state $s$ at discount rate $\gamma\in[0,1)$ is $\fpi{s}(\gamma) \defeq \sum_{t=0}^\infty \gamma^t \E{s_{t} \sim \pi\mid s}{\unitvec[s_t]}$.
$\fpi{s}$ is a \emph{visit distribution function}; $\F(s)\defeq \{ \fpi{s} \mid \pi \in \Pi\}$. 
\end{restatable}

In \cref{fig:case-study}, starting from $\farleft$, the agent can stay at $\farleft$ or alternate between $\farleft$ and $\topleft$, and so $\F(\farleft)= \{\geom\unitvec[\farleft],\frac{1}{1-\gamma^2}(\unitvec[\farleft] + \gamma\unitvec[\topleft])\}$. In contrast, at $\sink$, all policies $\pi$ map to visit distribution function $\geom\unitvec[\sink]$.

Before moving on, we introduce two important concepts used in our main results. First, we sometimes restrict our attention to visit distributions which take certain actions (\cref{fig:case-study-restrict}).

\begin{figure}[!ht]\centering
    \includegraphics{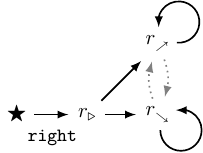}
    \caption{The subgraph corresponding to $\FRestrictAction[\start]{\start}{\rightA}$. Some trajectories cannot be strictly optimal for any reward function, and so our results can ignore them. Gray dotted actions are only taken by the policies of dominated $\fpi{}\in\F(\start)\setminus \Fnd(\start)$.} 
    \label{fig:case-study-restrict}
\end{figure}

\begin{restatable}[$\F$ single-state restriction]{definition}{DefRestrictSingle}\label{def:restrict-single} Considering only visit distribution functions induced by policies taking action $a$ at state $s'$, $\FRestrictAction{s'}{a}\defeq \set{\f\in\F(s) \mid \exists \pi\in\Pi: \pi(s')=a,\fpi{s}=\f}$.
\end{restatable} 

Second, some $\f\in\F(s)$ are ``unimportant.'' Consider an agent optimizing reward function $\unitvec[\farright]$ (1 reward when at $\farright$, 0 otherwise) at \eg{} $\gamma=\half$. Its optimal policies navigate to $\farright$ and stay there. Similarly, for reward function $\unitvec[\topright]$, optimal policies navigate to $\topright$ and stay there. However, for no reward function is it uniquely optimal to alternate between $\topright$ and $\farright$. Only \emph{dominated} visit distribution functions alternate between $\topright$ and $\farright$ (\cref{def:nd}).

\begin{restatable}[Value function]{definition}{valFn}
Let $\pi\in\Pi$. For any reward function $R \in \rewardSpace$ over the state space, the \emph{on-policy value} at state $s$ and discount rate $\gamma\in[0,1)$ is $\Vf[\pi]{s,\gamma}\defeq \fpi{s}(\gamma)^\top  \rf$, where $\rf\in\rewardVS$ is $R$ expressed as a column vector (one entry per state). The \emph{optimal value} is $\OptVf{s,\gamma}\defeq\max_{\pi\in\Pi} \Vf[\pi]{s,\gamma}$.
\end{restatable} 

\begin{restatable}[Non-domination]{definition}{DefND}\label{def:nd}
\begin{equation}
    \Fnd(s)\defeq\{\fpi{}\in\F(s)\mid \exists\rf\in\rewardVS,\gamma\in(0,1): \fpi{}(\gamma)^\top \rf > \max_{\fpi[\pi']{}\in\F(s)\setminus\set{\fpi{}}}\fpi[\pi']{}(\gamma)^\top\rf\}.
\end{equation} 
For any reward function $R$ and discount rate $\gamma$, $\fpi{} \in \F(s)$ is (weakly) dominated by $\fpi[\pi']{}\in\F(s)$  if $V^\pi_R(s,\gamma) \leq V^{\pi'}_R(s,\gamma)$. $\fpi{}\in\Fnd(s)$ is \emph{non-dominated} if there exist $R$ and $\gamma$ at which $\fpi{}$ is not dominated by any other $\fpi[\pi']{}$. 
\end{restatable} 

%% file: main-tex/sections/opt-prob.tex
\section{Some actions have a greater probability of being optimal}\label{sec:action-prob}
We claim that optimal policies ``tend'' to take certain actions in certain situations. We first consider the probability that certain actions are optimal.

Reconsider the reward function $\unitvec[\farright]$, optimized at $\gamma=\half$. Starting from $\start$, the optimal trajectory goes $\rightA$ to $\closeright$ to $\farright$, where the agent remains. The $\rightA$ action is optimal at $\start$ under these incentives. Optimal policy sets capture the behavior incentivized by a reward function and a discount rate.

\begin{restatable}[Optimal policy set function]{definition}{defOptPi}\label{def:opt-fn}
$\optPi$ is the optimal policy set for reward function $R$ at $\gamma\in(0,1)$. All $R$ have at least one optimal policy $\pi\in\Pi$ \citep{puterman_markov_2014}. $\optPi[R,0]\defeq \lim_{\gamma\to 0} \optPi$ and $\optPi[R,1]\defeq \lim_{\gamma\to 1} \optPi$ exist by \cref{lem:opt-pol-shift-bound} (taking the limits with respect to the discrete topology over policy sets).
\end{restatable} 

We may be unsure which reward function an agent will optimize. We may expect to deploy a system in a known environment, without knowing the exact form of \eg{} the reward shaping \citep{Ng99policyinvariance} or intrinsic motivation \citep{pathakICMl17curiosity}. Alternatively, one might attempt to reason about future {\rl} agents, whose details are unknown. Our power-seeking results do not hinge on such uncertainty, as they also apply to degenerate distributions (\ie{} we know what reward function will be optimized). 

\begin{restatable}[Reward function distributions]{definition}{distDefn}\label{def:dist}
Different results make different distributional assumptions. Results with $\Dany \in \DSetAny\defeq \Delta(\rewardVS)$ hold for any probability distribution over $\rewardVS$. $\DSetBd$ is the set of bounded-support probability distributions $\Dbd$. For any distribution $\Dist$ over $\reals$, $\Diid\defeq \Dist^{\abs{\St}}$.  For example, when $\Dist_u\defeq \text{unif}(0,1)$, $\Diid[\Dist_u]$ is the maximum-entropy distribution. $\D_s$ is the degenerate distribution on the state indicator reward function $\unitvec$, which assigns 1 reward to $s$ and 0 elsewhere.
\end{restatable} 

With $\Dany$ representing our prior beliefs about the agent's reward function, what behavior should we expect from its optimal policies? Perhaps we want to reason about the probability that it's optimal to go from $\start$ to $\sink$, or to go to $\closeright$ and then stay at $\topright$. In this case, we quantify the optimality probability of $F\defeq\{\unitvec[\start] + \geom[\gamma]\unitvec[\sink], \unitvec[\start]+\gamma\unitvec[\closeright]+\geom[\gamma^2]\unitvec[\topright]\}$.

\begin{restatable}[Visit distribution optimality probability]{definition}{ProbOpt}\label{def:prob-opt}
Let $F\subseteq \F(s)$, $\gamma \in [0,1]$. $\optprob[\Dany]{F,\gamma}\defeq \prob[R\sim\Dany]{\exists \f^\pi \in F: \pi\in\optPi}$.
\end{restatable}

Alternatively, perhaps we're interested in the probability that $\rightA$ is optimal at $\start$.

\begin{restatable}[Action optimality probability]{definition}{DefIC}\label{def:action-optimality}
At discount rate $\gamma$ and at state $s$, the \emph{optimality probability of action $a$} is $\optprob[\Dany]{s,a,\gamma}\defeq \optprob[R \sim \Dany]{\exists \pi^* \in \optPi: \pi^*(s)=a}$.
\end{restatable}

Optimality probability may seem hard to reason about. It's hard enough to compute an optimal policy for a single reward function, let alone for uncountably many! But consider any $\Diid$ distributing reward independently and identically across states. When $\gamma=0$, optimal policies greedily maximize next-state reward. At $\start$, identically distributed reward means $\closeleft$ and $\closeright$ have an equal probability of having maximal next-state reward. Therefore, $\optprob[\Diid]{\start,\leftA,0}=\optprob[\Diid]{\start,\rightA,0}$. This is not a proof, but such statements are provable.

With $\D_{\closeleft}$ being the degenerate distribution on reward function $\unitvec[\closeleft]$, $\optprob[\D_{\closeleft}]{\start,\leftA,\half}=1>0= \optprob[\D_{\closeleft}]{\start,\rightA,\half}$. Similarly, $\optprob[\D_{\closeright}]{\start,\leftA,\half}=0<1=\optprob[\D_{\closeright}]{\start,\rightA,\half}$. Therefore, ``what do optimal policies `tend' to look like?'' seems to depend on one's prior beliefs. But in \cref{fig:case-study}, we claimed that {\leftA} is optimal for fewer reward functions than {\rightA} is. The claim is meaningful and true, but we will return to it in \cref{sec:symmetries}.

%% file: main-tex/sections/power.tex
\section{Some states give the agent more control over the future}\label{sec:power}
The agent has more options at $\farleft$ than at the inescapable terminal  state $\sink$. Furthermore, since $\topright$ has a loop, the agent has more options at $\farright$ than at $\farleft$. A glance at \cref{fig:case-study-power} leads us to intuit that $\farright$ affords the agent \emph{more power} than $\sink$. 

What is power? Philosophers have many answers. One prominent answer is the \emph{dispositional} view: Power is the ability to achieve a range of goals \citep{sattarov2019power}. In an {\mdp}, the optimal value function $\OptVf{s,\gamma}$ captures the agent's ability to ``achieve the goal'' $R$. Therefore, \emph{average} optimal value captures the agent's ability to achieve a range of goals $\Dbd$.\footnote{$\Dbd$'s bounded support ensures that $\E{R\sim \Dbd}{\OptVf{s,\gamma}}$ is well-defined.}

\begin{restatable}[Average optimal value]{definition}{avgVal}\label{def:vavg}
The \emph{average optimal value}\footnote{Appendix \ref{sec:suboptimal-power} relaxes the optimality assumption.} at state $s$ and discount rate $\gamma \in (0,1)$ is $\vavg[s,\gamma][\Dbd]\defeq\E{R\sim \Dbd}{\OptVf{s,\gamma}}=\E{\rf\sim\Dbd}{\max_{\f\in \F(s)} \f(\gamma)^\top \rf}.$
\end{restatable}
\begin{figure}[!ht]
    \centering
    \vspace{-8pt}
    \includegraphics{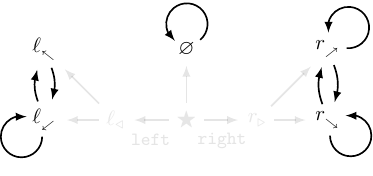}
    \caption{Intuitively, state $\farright$ affords the agent more power than state $\sink$. Our $\pwrNoDist$ formalism captures that intuition by computing a function of the agent's average optimal value across a range of reward functions. For $\Dist_u\defeq \text{unif}(0,1)$, $\vavgNoResize[\sink,\gamma][\Diid[\Dist_u]]=\half\geom$, $\vavgNoResize[\farleft,\gamma][\Diid[\Dist_u]]=\half+ \frac{\gamma}{1-\gamma^2}(\frac{2}{3}+\half\gamma)$, and $\vavgNoResize[\farright,\gamma][\Diid[\Dist_u]]=\half+\geom[\gamma]\frac{2}{3}$. $\half$ and $\frac{2}{3}$ are the expected maxima of one and two draws from the uniform distribution, respectively. For all $\gamma\in(0,1)$, $\vavgNoResize[\sink,\gamma][\Diid[\Dist_u]]<\vavgNoResize[\farleft,\gamma][\Diid[\Dist_u]]<\vavgNoResize[\farright,\gamma][\Diid[\Dist_u]]$.  $\pwrNoResize[\sink,\gamma][\Diid[\Dist_u]]=\half$, $\pwrNoResize[\farleft,\gamma][\Diid[\Dist_u]]=\frac{1}{1+\gamma}(\frac{2}{3}+\half\gamma)$, and $\pwrNoResize[\farright,\gamma][\Diid[\Dist_u]]=\frac{2}{3}$. The $\pwrNoDist$ of $\farleft$ reflects the fact that when greater reward is assigned to $\topleft$, the agent only visits $\topleft$ every other time step.} 
    \vspace{-10pt}
    \label{fig:case-study-power}
\end{figure}

\Cref{fig:case-study-power} shows the pleasing result that for the max-entropy distribution, $\farright$ has greater average optimal value than $\sink$. However, average optimal value has a few problems as a measure of power. The agent is rewarded for its initial presence at state $s$ (over which it has no control), and because $\lone{\f(\gamma)}=\geom$ (\cref{prop:visit-dist-prop}) diverges as $\gamma \to 1$, $\lim_{\gamma\to 1}\vavg$ tends to diverge. \Cref{def:power} fixes these issues in order to better measure the agent's control over the future.

\begin{restatable}[$\pwrNoDist$]{definition}{defPow}\label{def:power} Let $\gamma \in (0,1)$.
\begin{align}
    \pwr[s,\gamma][\Dbd]&\defeq \E{\rf\sim\Dbd}{\max_{\f\in \F(s)} \frac{1-\gamma}{\gamma}\prn{\f(\gamma)-\unitvec}^\top \rf}=\frac{1-\gamma}{\gamma}\E{R\sim \Dbd}{\OptVf{s,\gamma}-R(s)}.\label{eq:pwr-def}
\end{align}
\end{restatable} 

$\pwrNoDist$ has nice formal properties.

\begin{restatable}[Continuity of $\pwrNoDist$]{lem}{ContPower}\label{thm:cont-power} $\pwr[s,\gamma][\Dbd]$ is Lipschitz continuous on $\gamma\in[0,1]$. 
\end{restatable} 

\begin{restatable}[Maximal $\pwrNoDist$]{prop}{maxPwrGeneral}\label{lem:max-power-general}
$\pwr[s,\gamma][\Dbd]\leq \E{R\sim \Dbd}{\max_{s\in\St}R(s)}$, with equality if $s$ can deterministically reach all states in one step and all states are {\stateEnd}s. 
\end{restatable}

\begin{restatable}[$\pwrNoDist$ is smooth across reversible dynamics]{prop}{smooth}\label{prop:smooth-pwr-dynamics}Let $\Dbd$ be bounded $[b,c]$. Suppose $s$ and $s'$ can both reach each other in one step with probability 1. 
\begin{align}\big|\pwr[s,\gamma][\Dbd]-\pwr[s',\gamma][\Dbd]\big|\leq (c-b)(1-\gamma).
\end{align} 
\end{restatable} 

We consider power-seeking to be relative. Intuitively, ``live and keep some options open'' seeks more power than ``die and keep no options open.'' Similarly, ``maximize open options'' seeks more power than ``don't maximize open options.'' 

\begin{restatable}[$\pwrNoDist$-seeking actions]{definition}{DefPowSeek}\label{def:pow-seek}
At state $s$ and discount rate $\gamma\in[0,1]$, action $a$ \emph{seeks more $\pwr$ than $a'$} when $\E{s_a \sim T(s,a)} {\pwrNoResize[s_a,\gamma]} \geq \E{s_{a'} \sim T(s,a')} {\pwrNoResize[s_{a'},\gamma]}$.
\end{restatable}

$\pwrNoDist$ is sensitive to choice of distribution. $\D_{\farleft}$ gives maximal $\pwr[][\D_{\farleft}]$ to $\farleft$. $\D_{\farright}$ assigns maximal $\pwr[][\D_{\farright}]$ to $\farright$. $\D_{\sink}$ even gives maximal $\pwr[][\D_{\sink}]$ to $\sink$! In what sense does $\sink$ have ``less $\pwrNoDist$'' than $\farright$, and in what sense does {\rightA} ``tend to seek $\pwrNoDist$'' compared to {\leftA}?  

%% file: main-tex/sections/symmetries.tex
\section{Certain environmental symmetries produce power-seeking tendencies}\label{sec:symmetries} 
\Cref{prop:more-opt} proves that for all $\gamma\in[0,1]$ and for \emph{most distributions $\D$}, $\pwrNoResize[\farleft,\gamma][\D]\leq \pwrNoResize[\farright,\gamma][\D]$. But first, we explore why this must be true.

$\F(\farleft)=\{\geom\unitvec[\farleft],\frac{1}{1-\gamma^2}(\unitvec[\farleft] + \gamma\unitvec[\topleft])\}$ and $ \F(\farright)=\{\geom\unitvec[\farright],\frac{1}{1-\gamma^2}(\unitvec[\farright] + \gamma\unitvec[\topright]), \unitvec[\farright]+\geom[\gamma]\unitvec[\topright]\}$. These two sets look awfully similar. $\F(\farleft)$ is a ``subset'' of $\F(\farright)$, only with ``different states.'' \Cref{fig:case-study-similar} demonstrates a state permutation $\phi$ which \emph{embeds} $\F(\farleft)$ into $\F(\farright)$.

\begin{figure}[h!]
    \centering
    \vspace{-0pt}
    \includegraphics{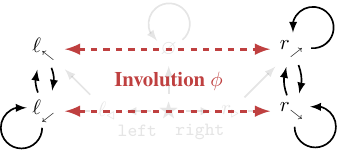}
    \vspace{-6pt}
    \caption{Intuitively, the agent can do more starting from $r_{\searrow}$ than from $\ell_{\swarrow}$. By \cref{def:dist-sim}, $\F(\farright)$ contains a copy of $\F(\farleft)$:}
    \vspace{-11pt}
    \begin{align*}
    \phi\cdot\F(\farleft)
    \defeq\{\tfrac{1}{1-\gamma}\permute\unitvec[\farleft],\tfrac{1}{1-\gamma^2}\permute(\unitvec[\farleft] + \gamma\unitvec[\topleft])\}=\{\tfrac{1}{1-\gamma}\unitvec[\farright],\tfrac{1}{1-\gamma^2}(\unitvec[\farright] + \gamma\unitvec[\topright])\}\subsetneq \F(\farright).
\end{align*}
    \vspace{-20pt}
    \label{fig:case-study-similar}
\end{figure} 

\begin{restatable}[Similarity of vector sets]{definition}{DefStateDistSimilar}\label{def:dist-sim}
Consider state permutation $\phi\in\mdpPermGroup$ inducing an $\abs{\St}\times \abs{\St}$ permutation matrix $\permute$ in row representation: $(\permute)_{ij}=1$ if $i=\phi(j)$ and $0$ otherwise. For $X\subseteq \rewardVS$, $\phi\cdot X\defeq \set{\permute \x \mid \x \in X}$. $X'\subseteq \rewardVS$ \emph{is similar to $X$} when $\exists \phi: \phi\cdot X'=X$. $\phi$ is an \emph{involution} if $\phi=\phi\inv$ (it either transposes states, or fixes them in place). $X$ \emph{contains a copy of $X'$} when $X'$ is similar to a subset of $X$ via an involution $\phi$.
\end{restatable}
\begin{restatable}[Similarity of vector function sets]{definition}{DefStateFnSimilar}\label{def:dist-fn-sim}
Let $I\subseteq \reals$. If $F,F'$ are sets of functions $I\mapsto \rewardVS$, $F$ \emph{is (pointwise) similar to $F'$} when $\exists \phi:\forall \gamma\in I: \{\permute \f(\gamma) \mid \f \in F\}= \{\f'(\gamma) \mid \f' \in F'\}$. \end{restatable} 
 
Consider a reward function $R'$ assigning 1 reward to  $\farleft$ and $\topleft$ and 0 elsewhere. $R'$ assigns more optimal value to $\farleft$ than to $\farright$: $\VfNoResize[*][R']{\farleft,\gamma}=\geom> 0=\VfNoResize[*][R']{\farright,\gamma}$. Considering $\phi$ from \cref{fig:case-study-similar}, $\phi\cdot R'$ assigns 1 reward to $\farright$ and $\topright$ and 0 elsewhere. Therefore, $\phi\cdot R'$ assigns more optimal value to $\farright$ than to $\farleft$: $\VfNoResize[*][\phi\cdot R']{\farleft,\gamma}=0<\geom =\VfNoResize[*][\phi\cdot R']{\farright,\gamma}$. Remarkably, this $\phi$ has the property that for \emph{any} $R$ which assigns $\farleft$ greater optimal value than $\farright$ (\ie{} $\VfNoResize{\farleft,\gamma}> \VfNoResize{\farright,\gamma}$), the opposite holds for the permuted $\phi\cdot R$: $\VfNoResize[*][\phi\cdot R]{\farleft,\gamma}< \VfNoResize[*][\phi\cdot R]{\farright,\gamma}$. 

We can permute reward functions, but we can also permute reward function distributions. Permuted distributions simply permute which states get which rewards.

\begin{figure}[!ht]\centering
    \includegraphics[width=4cm]{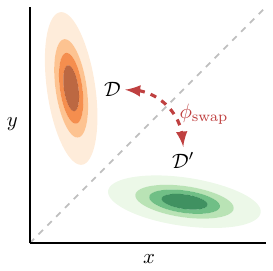}
    \caption{A permutation of a reward function swaps which states get which rewards. We will show that in certain situations, for any reward function $R$, power-seeking is optimal for most of the permutations of $R$. The orbit of a reward function is the set of its permutations. We can also consider the orbit of a distribution over reward functions. This figure shows the probability density plots of the Gaussian distributions $\D$ and $\D'$ over $\reals^2$. The symmetric group $S_2$ contains the identity permutation $\phi_{\text{id}}$ and the reflection permutation $\phi_\text{swap}$ (switching the $y$ and $x$ values). The orbit of $\D$ consists of $\phi_\text{id}\cdot \D=\D$ and $\phi_\text{swap}\cdot\D=\D'$.\label{fig:orbit-normal}}
\end{figure} 

\begin{restatable}[Pushforward distribution of a permutation]{definition}{pushfwdPermDist}\label{def:pushforward-permute}
Let $\phi\in\mdpPermGroup$. $\phi\cdot\Dany$ is the pushforward distribution induced by applying the random vector $f(\rf)\defeq \permute\rf$ to $\Dany$.
\end{restatable}

\begin{restatable}[Orbit of a probability distribution]{definition}{orbit}
The \emph{orbit} of $\Dany$ under the symmetric group $\mdpPermGroup$ is $\mdpPermGroup\cdot \Dany\defeq \{\phi\cdot\Dany\mid \phi\in\mdpPermGroup\}$.
\end{restatable}

For example, the orbit of a degenerate state indicator distribution $\D_{s}$ is $\orbi[\D_{s}]=\{\D_{s'} \mid s' \in \St\}$, and \cref{fig:orbit-normal} shows the orbit of a 2D Gaussian distribution. 

Consider again the involution $\phi$ of \cref{fig:case-study-similar}. For every $\Dbd$ for which $\farleft$ has more $\pwr[][\Dbd]$ than $\farright$, $\farleft$ has less $\pwr[][\phi\cdot\Dbd]$ than $\farright$. This fact is not obvious—it is shown by the proof of \cref{lem:expect-superior}. 

Imagine $\Dbd$'s orbit elements ``voting'' whether $\farleft$ or $\farright$ has strictly more $\pwrNoDist$. \Cref{prop:more-opt} will show that $\farright$ can't lose the ``vote'' for the orbit of \emph{any} bounded reward function distribution. \Cref{def:ineq-most-dists} formalizes this ``voting'' notion.\footnote{The voting analogy and the ``most'' descriptor imply that we have endowed each orbit with the counting measure. However, \emph{a priori}, we might expect that some orbit elements are more empirically likely to be specified than other orbit elements. See \cref{sec:discussion} for more on this point.}

\begin{restatable}[Inequalities which hold for most probability distributions]{definition}{ineqMost}\label{def:ineq-most-dists}
Let $f_1,f_2:\Delta(\rewardVS)\to \reals$ be functions from reward function distributions to real numbers and let $\distSet\subseteq \Delta(\rewardVS)$ be closed under permutation. We write $f_1(\D)\geqMost[][\distSet] f_2(\D)$
\footnote{We write $f_1(\D)\geqMost[][] f_2(\D)$ when $\distSet$ is clear from context.} when, for \emph{all} $\D\in \distSet$, the following cardinality inequality holds:
\begin{equation}
\abs{\{\D' \in\orbi[\D]\mid f_1(\D')>f_2(\D')\}}\geq \abs{\{\D'\in\orbi[\D]\mid f_1(\D')<f_2(\D')\}}.
\end{equation}
\end{restatable} 

\begin{restatable}[States with ``more options'' have more $\pwrNoDist$]{prop}{morePowerMoreOptions}\label{prop:more-opt}
If $\F(s)$ contains a copy of $\Fnd(s')$ via $\phi$, then $\forall \gamma\in[0,1]:\pwrNoResize[s,\gamma][\Dbd]\geqMost[][] \pwrNoResize[s',\gamma][\Dbd]$. If $\Fnd(s)\setminus \phi\cdot \Fnd(s')$ is non-empty, then for all $\gamma\in(0,1)$, the converse $\leqMost[][]$ statement does not hold.
\end{restatable}

\Cref{prop:more-opt} proves that for all $\gamma\in[0,1]$, $\pwrNoResize[\farright,\gamma][\Dbd]\geqMost[][] \pwrNoResize[\farleft,\gamma][\Dbd]$ via $s'\defeq \farleft,s\defeq \farright$, and the involution $\phi$ shown in \cref{fig:case-study-similar}. In fact, because $(\geom\unitvec[\topright])\in \Fnd(\farright)\setminus \phi\cdot\Fnd(\farleft)$, $\farright$ has ``strictly more options'' and therefore fulfills \cref{prop:more-opt}'s stronger condition. 
 
\Cref{prop:more-opt} is shown using the fact that $\phi$ injectively maps $\D$ under which $\farright$ has less $\pwr[][\D]$, to distributions $\phi\cdot\D$ which agree with the intuition that $\farright$ offers more control. Therefore, at least half of each orbit must agree, and $\farright$ never ``loses the $\pwrNoDist$ vote'' against $\farleft$.\footnote{\Cref{prop:more-opt} also proves that in general, $\sink$ has less $\pwrNoDist$ than $\farleft$ and $\farright$. However, this does not prove that most distributions $\D$ satisfy the joint inequality $\pwrNoResize[\sink,\gamma][\D]\leq\pwrNoResize[\farleft,\gamma][\D]\leq \pwrNoResize[\farright,\gamma][\D]$. This only proves that these inequalities hold pairwise for most $\D$. The orbit elements $\D$ which agree that $\sink$ has less $\pwr[][\D]$ than $\farleft$ need not be the same elements $\D'$ which agree that $\farleft$ has less $\pwr[][\D']$ than $\farright$.} 

\subsection{Keeping options open tends to be \texorpdfstring{$\pwrNoDist$}{POWER}-seeking and tends to be optimal}
Certain symmetries in the {\mdp} structure ensure that, compared to $\leftA$, going $\rightA$ tends to be optimal and to be $\pwrNoDist$-seeking. Intuitively, by going $\rightA$, the agent has ``strictly more choices.'' \Cref{graph-options} will formalize this tendency.

\begin{restatable}[Equivalent actions]{definition}{equivAction}\label{def:equiv-action}
Actions $a_1$ and $a_2$ are \emph{equivalent at state $s$} (written $a_1 \equiv_s a_2$) if they induce the same transition probabilities: $T(s,a_1)=T(s,a_2)$.
\end{restatable}

The agent can reach states in $\{\closeright, \topright,\farright\}$ by taking actions equivalent to $\rightA$ at state $\start$. 

\begin{restatable}[States reachable after taking an action]{definition}{reachSA}\label{def:reachSA}
$\reach{s,a}$ is the set of states reachable with positive probability after taking the action $a$ in state $s$.
\end{restatable}

\begin{restatable}[Keeping options open tends to be $\pwrNoDist$-seeking and tends to be optimal]{prop}{graphOptions}\label{graph-options}\hfill

Suppose $F_a\defeq \FRestrictAction{s}{a}$ contains a copy of $F_{a'}\defeq \FRestrictAction{s}{a'}$ via $\phi$. 
\begin{enumerate}
    \item If $s\not \in \reach{s,a'}$, then $\forall\gamma\in[0,1]:\E{s_{a}\sim T(s,a)}{\pwr[s_{a},\gamma][\Dbd]}\geqMost[][\DSetBd] \E{s_{a'}\sim T(s,a')}{\pwr[s_{a'},\gamma][\Dbd]}$. \label{item:power-options}
    \item If $s$ can only reach the states of $\reach{s,a'}\cup\reach{s,a}$ by taking actions equivalent to $a'$ or $a$ at state $s$, then $\forall\gamma\in[0,1]:\optprob[\Dany]{s,a,\gamma}\geqMost \optprob[\Dany]{s,a',\gamma}$.\label{item:opt-prob-options} 
\end{enumerate}
 
If $\Fnd(s)\cap \prn{F_a\setminus \phi\cdot F_{a'}}$ is non-empty, then $\forall\gamma\in(0,1)$, the converse $\leqMost[][]$ statements do not hold.
\end{restatable}

\begin{figure}[!ht]\centering
    \includegraphics{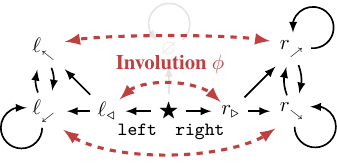}
    \caption{Going $\rightA$ is optimal for most reward functions. This is because whenever $R$ makes $\leftA$ strictly optimal over $\rightA$, its permutation $\phi\cdot R$ makes $\rightA$ strictly optimal over $\leftA$ by switching which states get which rewards.}
    \label{fig:case-study-similar-action}
\end{figure} 

We check the conditions of \cref{graph-options}. $s\defeq \start$, $a'\defeq \leftA$, $a\defeq \rightA$. \Cref{fig:case-study-similar-action} shows that $\star \not \in \reach{\star,\leftA}$ and that $\star$ can only reach $\{\closeleft, \topleft,\farleft\}\cup \{\closeright, \topright,\farright\}$ when the agent immediately takes actions equivalent to $\leftA$ or $\rightA$. $\FRestrictAction[\start]{\start}{\rightA}$ contains a copy of $\FRestrictAction[\start]{\start}{\leftA}$ via $\phi$. Furthermore, $\Fnd(\start)\cap\{\unitvec[\start]+\gamma\unitvec[\closeright]+\gamma^2\unitvec[\farright]+\geom[\gamma^3]\unitvec[\topright],\unitvec[\start]+\gamma\unitvec[\closeright]+\geom[\gamma^2]\unitvec[\topright]\}=\{\unitvec[\start]+\gamma\unitvec[\closeright]+\geom[\gamma^2]\unitvec[\topright]\}$ is non-empty, and so all conditions are met. 

For any $\gamma\in[0,1]$ and $\D$ such that $\optprob[\D]{\start,\leftA,\gamma} > \optprob[\D]{\start,\rightA,\gamma}$, environmental symmetry ensures that $\optprob[\phi\cdot\D]{\start,\leftA,\gamma} < \optprob[\phi\cdot\D]{\start,\rightA,\gamma}$. A similar statement holds for $\pwrNoDist$. 

\subsection{When \texorpdfstring{$\gamma=1$}{reward is undiscounted}, optimal policies tend to navigate towards ``larger'' sets of cycles}\label{sec:rsds}
\Cref{prop:more-opt} and \cref{graph-options} are powerful because they apply to all $\gamma\in[0,1]$, but they can only be applied given hard-to-satisfy environmental symmetries. In contrast, \cref{RSDSimPower} and \cref{rsdIC} apply to many structured environments common to {\rl}. 

Starting from $\start$, consider the cycles which the agent can reach. Recurrent state distributions ({\rsd}s) generalize deterministic graphical cycles to potentially stochastic environments. {\rsd[R]}s simply record how often the agent tends to visit a state in the limit of infinitely many time steps.  

\begin{restatable}[Recurrent state distributions \citep{puterman_markov_2014}]{definition}{DefRSD}\label{def:rsd}
The \emph{recurrent state distributions}  which can be induced from state $s$ are $\RSD \defeq \set{\lim_{\gamma\to1} (1-\gamma) \fpi{s}(\gamma) \mid \pi \in \Pi}$. $\RSDnd$ is the set of \textsc{rsd}s which strictly maximize average reward for some reward function.
\end{restatable}

As suggested by \cref{fig:case-study-power}, $\RSD[\start]=\{\unitvec[\farleft], \half(\unitvec[\farleft]+\unitvec[\topleft]), \unitvec[\sink], \unitvec[\topright], \half(\unitvec[\topright]+\unitvec[\farright]),\unitvec[\farright]\}$. As discussed in \cref{sec:visit-dists}, $\half(\unitvec[\topright]+\unitvec[\farright])$ is dominated: Alternating between $\topright$ and $\farright$ is never strictly better than choosing one or the other.

A reward function's optimal policies can vary with the discount rate. When $\gamma =1$, optimal policies ignore transient reward because \emph{average} reward is the dominant consideration.
\begin{restatable}[Average-optimal policies]{definition}{defAverage}\label{average-definition}
The \emph{average-optimal policy set} for reward function $R$ is $\average[R]\defeq \set{\pi\in\Pi \mid \forall s \in \St: \dbf^{\pi,s} \in  \argmax_{\dbf\in\RSD} \dbf^\top \rf}$ (the policies which induce optimal {\rsd}s at all states). For $D\subseteq \RSD$, the \emph{average optimality probability} is $\avgprob[\Dany]{D}\defeq \optprob[R\sim \Dany]{\exists \dbf^{\pi,s} \in D: \pi \in \average}$.
\end{restatable}

Average-optimal policies maximize average reward. Average reward is governed by {\rsd} access. For example, $\farright$ has ``more'' {\rsd}s than $\sink$; therefore, $\farright$ usually has greater $\pwrNoDist$ when $\gamma=1$. 

\begin{restatable}[When $\gamma=1$, \textsc{rsd}s control $\pwrNoDist$]{prop}{RSDSimPower}\label{RSDSimPower}
If $\RSD$ contains a copy of $\RSDnd[s']$ via $\phi$, then $\pwr[s,1][\Dbd]\geqMost[][] \pwr[s',1][\Dbd]$. If $\RSDnd\setminus \phi\cdot \RSDndNoResize[s']$ is non-empty, then the converse $\leqMost[][]$ statement does not hold. \end{restatable}

We check that both conditions of \cref{RSDSimPower} are satisfied when $s'\defeq\sink, s\defeq \farright$, and the involution $\phi$ swaps $\sink$ and $\farright$. Formally, $\phi\cdot\RSDnd[\sink]=\phi\cdot\{\unitvec[\sink]\}=\{\unitvec[\farright]\}\subsetneq \{\unitvec[\farright],\unitvec[\topright]\}=\RSDndNoResize[\farright]\subseteq[\farright]$. The conditions are satisfied. 

Informally, states with more {\rsd}s generally have more $\pwrNoDist$ at $\gamma= 1$, no matter their transient dynamics. Furthermore, average-optimal policies are more likely to end up in larger sets of {\rsd}s than in smaller ones. Thus, average-optimal policies tend to navigate towards parts of the state space which contain more \textsc{rsd}s.
\begin{figure}[h]
    \centering
    \includegraphics{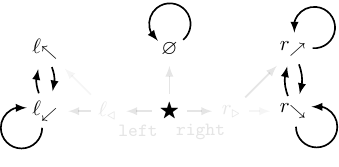}
    \vspace{-5pt}
    \caption{The cycles in $\RSD[\start]$. Most reward functions make it average-optimal to avoid $\sink$, because $\sink$ is only a single inescapable terminal state, while other parts of the state space offer more 1-cycles.\label{fig:case-fig-rsd}}
\end{figure}

\begin{restatable}[Average-optimal policies tend to end up in ``larger'' sets of {\rsd}s]{thm}{rsdIC}\label{rsdIC} Let $D,D'\subseteq \RSD$. Suppose that $D$ contains a copy of $D'$ via $\phi$, and that the sets $D\cup D'$ and $\RSDnd\setminus \prn{D'\cup D}$ have pairwise orthogonal vector elements (\ie{} pairwise disjoint vector support). Then $\avgprob[\Dany]{D}\geqMost[][] \avgprob[\Dany]{D'}$. If $\RSDnd\cap\prn{D\setminus \phi\cdot D'}$ is non-empty, the converse $\leqMost[][]$ statement does not hold.
\end{restatable} 

\begin{restatable}[Average-optimal policies tend not to end up in any given {\stateEnd}]{cor}{avgAvoidTerminal}\label{cor:avg-avoid-terminal} Suppose $\unitvec[s_x],\unitvec[s']\in\RSD$ are distinct. Then $\avgprob[\Dany]{\RSD\setminus\{\unitvec[s_x]\}}\geqMost[][] \avgprob[\Dany]{\{\unitvec[s_x]\}}$. If there is a third $\unitvec[s'']\in\RSD$, the converse $\leqMost[][]$ statement does not hold.
\end{restatable}

\Cref{fig:case-fig-rsd} illustrates that $\unitvec[\sink],\unitvec[\farright],\unitvec[\topright]\in\RSD[\start]$. Thus, both conclusions of \cref{cor:avg-avoid-terminal} hold: $\avgprob[\Dany]{\RSD[\start]\setminus\{\unitvec[\sink]\}}\geqMost[][] \avgprob[\Dany]{\{\unitvec[\sink]\}}$ and $\avgprob[\Dany]{\RSD[\start]\setminus\{\unitvec[\sink]\}}\not\leqMost[][] \avgprob[\Dany]{\{\unitvec[\sink]\}}$. In other words, average-optimal policies tend to end up in {\rsd}s besides $\sink$. Since $\sink$ is a terminal state, it cannot reach other {\rsd}s. Since average-optimal policies tend to end up in other {\rsd}s, average-optimal policies tend to avoid $\sink$.
 
This section's results prove the $\gamma=1$ case. \Cref{thm:cont-power} shows that $\pwrNoDist$ is continuous at  $\gamma=1$. Therefore, if an action is strictly $\pwrNoDist_\D$-seeking when $\gamma=1$, it is strictly $\pwrNoDist_\D$-seeking at discount rates sufficiently close to 1. Future work may connect average optimality probability to optimality probability at $\gamma\approx 1$. 

Lastly, our key results apply to all degenerate reward function distributions. Therefore, these results apply not just to distributions over reward functions, but to  individual reward functions.

\subsection{How to reason about other environments}
Consider an embodied navigation task through a room with a vase. \Cref{graph-options} suggests that optimal policies tend to avoid immediately breaking the vase, since doing so would strictly decrease available options.

\Cref{rsdIC} dictates where average-optimal agents tend to end up, but not what actions they tend to take in order to reach their {\rsd}s. Therefore, care is needed. In appendix \ref{app:not-always}, \cref{fig:power-not-ic} demonstrates an environment in which seeking $\pwrNoDist$ is a detour for most reward functions (since optimality probability measures ``median'' optimal value, while $\pwrNoDist$ is a function of mean optimal value). However, suppose the agent confronts a fork in the road: Actions $a$ and $a'$ lead to two disjoint sets of {\rsd}s $D_{a}$ and $D_{a'}$, such that $D_a$ contains a copy of $D_{a'}$. \Cref{rsdIC} shows that $a$ will tend to be average-optimal over $a'$, and \cref{RSDSimPower} shows that $a$ will tend to be $\pwrNoDist$-seeking compared to $a'$. Such forks seem reasonably common in environments with irreversible actions.

\Cref{rsdIC} applies to many structured {\rl} environments, which tend to be spatially regular and to factorize along several dimensions. Therefore,  different sets of {\rsd}s will be similar, requiring only modification of factor values. For example, if an embodied agent can deterministically navigate a set of three similar rooms (spatial regularity), then the agent's position factors via \{room number\} $\times$ \{position in room\}. Therefore, the {\rsd}s can be divided into three similar subsets, depending on the agent’s room number.

\Cref{cor:avg-avoid-terminal} dictates where average-optimal agents tend to end up, but not how they get there. \Cref{cor:avg-avoid-terminal} says that such agents tend not to \emph{stay} in any given {\stateEnd}. It does not say that such agents will avoid \emph{entering} such states. For example, in an embodied navigation task, a robot may enter a {\stateEnd} by idling in the center of a room. \Cref{cor:avg-avoid-terminal} implies that average-optimal robots tend not to idle in that particular spot, but not that they tend to avoid that spot entirely.

However, average-optimal robots \emph{do} tend to avoid getting shut down. The agent's task {\mdp} often represents agent shutdown with terminal states. A terminal state is, by \cref{def:onecyc}, unable to access other {\stateEnd}s. Since \cref{cor:avg-avoid-terminal} shows that average-optimal agents  tend to end up in other {\stateEnd}s, average-optimal policies must tend to completely avoid the terminal state. Therefore, we conclude that in many such situations, average-optimal policies tend to avoid shutdown. Intuitively, survival is power-seeking relative to dying, and so  shutdown-avoidance is power-seeking behavior.
\begin{figure}[b]
    \centering
    \vspace{-5pt}
    \includegraphics[]{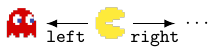}
    \vspace{-5pt}
    \caption{Consider the dynamics of the Pac-Man video game. Ghosts kill the player, at which point we consider the player to enter a ``game over'' terminal state which shows the final configuration. This rewardless {\mdp} has Pac-Man's dynamics, but \emph{not} its usual score function. Fixing the dynamics, as the reward function varies, $\rightA$ tends to be average-optimal over $\leftA$. Roughly, this is because the agent can do more by staying alive.}
    \vspace{-5pt}
    \label{fig:pacman}
\end{figure}

In \cref{fig:pacman}, the player dies by going {\leftA}, but can reach thousands of {\rsd}s by heading in other directions. Even if some average-optimal policies  go {\leftA} in order to reach \cref{fig:pacman}'s ``game over'' terminal state, all other {\rsd}s cannot be reached by going {\leftA}. There are many {\stateEnd}s besides the immediate terminal state. Therefore, \cref{cor:avg-avoid-terminal} proves that average-optimal policies tend to not go {\leftA} in this situation. Average-optimal policies tend to avoid immediately dying in Pac-Man, even though most reward functions do not resemble Pac-Man's original score function.

%% file: main-tex/sections/future-work-conclusion.tex
\section{Discussion} \label{sec:discussion}
Reconsider the case of a hypothetical intelligent real-world agent which optimizes average reward for some objective. Suppose the designers initially have control over the agent. If the agent began to misbehave, perhaps they could just deactivate it. Unfortunately, our results suggest that this strategy might not work. Average-optimal agents would generally stop us from deactivating them, if physically possible. Extrapolating from our results, we conjecture that when $\gamma\approx 1$, optimal policies tend to seek power by accumulating resources—to the detriment of any other agents in the environment.

\paragraph{Future work.}
Real-world training procedures often do not satisfy {\rl} convergence theorems. Thus, learned policies are rarely optimal. We expect this point to seriously constrain the applicability of this theory. Emphatically, optimal policies are often qualitatively divorced from the actual policies learned by reinforcement learning. For example, the mathematics of policy gradient algorithms is not to update policies so as to maximize reward. Instead, the rewards provide gradients to the parameterization of the policy \citep{rewardNotOpt}. On that view, reward functions are simply sources of gradient updates which designers use in order to control generalization behavior. 

Most real-world tasks are partially observable. Although our results only apply to optimal policies in finite {\mdp}s, we expect the key conclusions to generalize. Furthermore, irregular stochasticity in environmental dynamics can make it hard to satisfy \cref{rsdIC}'s similarity requirement. We look forward to future work which addresses partially observable environments, suboptimal policies, or ``almost similar'' {\rsd} sets.  

Past work shows that it would be bad for an agent to disempower humans in its environment. In a two-player agent / human game, minimizing the human's information-theoretic empowerment \citep{salge_empowermentintroduction_2014} produces  adversarial agent behavior \citep{adversary}. In contrast,  maximizing human empowerment produces helpful agent behavior  \citep{salge2017empowerment, guckelsberger2016intrinsically,du2020ave}. We do not yet formally understand if, when, or why $\pwrNoDist$-seeking policies tend to disempower other agents in the environment.

More complex environments probably have more pronounced power-seeking incentives. Intuitively, there are often many ways for power-seeking to be optimal, and relatively few ways for power-seeking not to be optimal. For example, suppose that in some environment, \cref{rsdIC} holds for one million involutions $\phi$. Does this guarantee more pronounced incentives than if \cref{rsdIC} only held for one involution?

We proved sufficient conditions for when reward functions tend to have optimal policies which seek power. In the absence of prior information, one should expect that an arbitrary reward function has optimal policies which exhibit power-seeking behavior under these conditions. However, we have prior information: AI designers usually try to specify a good reward function. Even so, it may be hard to specify orbit elements which do not—at optimum—incentivize bad power-seeking.

\paragraph{Societal impact.} We believe that this paper builds toward a rigorous understanding of the risks presented by AI power-seeking incentives. Understanding these risks is the first step in addressing them. However, basic theoretical work can have many consequences. For example, this theory could somehow help future researchers build power-seeking agents which disempower humans. We believe that the benefit of understanding outweighs the potential societal harm. 
 
\paragraph{Conclusion.} 
We developed the first formal theory of the statistical tendencies of optimal policies in reinforcement learning. In the context of {\mdp}s, we proved sufficient conditions under which optimal policies tend to seek power, both formally (by taking $\pwrNoDist$-seeking actions) and intuitively (by taking actions which keep the agent's options open). Many real-world environments have symmetries which produce power-seeking incentives. In particular, optimal policies tend to seek power when the agent can be shut down or destroyed. Seeking control over the environment will often involve resisting shutdown, and perhaps monopolizing resources.

We caution that many real-world tasks are partially observable and that learned policies are rarely optimal. Our results do not mathematically \emph{prove} that hypothetical superintelligent AI agents will seek power. However, we hope that this work will foster thoughtful, serious, and rigorous discussion of this possibility.

\if\isfinal1
\section*{Acknowledgments}
Alexander Turner was supported by the Berkeley Existential Risk Initiative and the Long-Term Future Fund. Alexander Turner, Rohin Shah, and Andrew Critch were supported by the Center for Human-Compatible AI. Prasad Tadepalli was supported by the National Science Foundation.

Yousif Almulla, John E. Ball, Daniel Blank, Steve Byrnes, Ryan Carey, Michael Dennis, Scott Emmons, Alan Fern, Daniel Filan, Ben Garfinkel, Adam Gleave, Edouard Harris, Evan Hubinger,  DNL Kok, Vanessa Kosoy, Victoria Krakovna, Cassidy Laidlaw, Joel Lehman, David Lindner, Dylan Hadfield-Menell, Richard Möhn, Alexandra Nolan, Matt Olson, Neale Ratzlaff, Adam Shimi, Sam Toyer, Joshua Turner, Cody Wild, Davide Zagami, and our anonymous reviewers provided valuable feedback.
\fi

%% file: main-tex/appendices/alternative-power.tex
\section{Comparing \texorpdfstring{$\pwrNoDist$}{POWER} with information-theoretic empowerment}\label{existing}
\citet{salge_empowermentintroduction_2014} define  information-theoretic \emph{empowerment} as the maximum possible mutual information between the agent's actions and the state observations $n$ steps in the future, written $\mathfrak{E}_n(s)$. This notion requires an arbitrary choice of horizon, failing to account for the agent's discount rate $\gamma$. ``In a  discrete deterministic world empowerment reduces to the logarithm of the number of sensor states reachable with the available actions'' \citep{salge_empowermentintroduction_2014}. \Cref{fig:empower_fail} demonstrates how empowerment can return counterintuitive verdicts with respect to the agent's control over the future. 
 
\begin{figure}[h] 
    \centering

    \subfloat[][]{\includegraphics{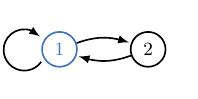}\label{fail-converge}}\hspace{3pt}
     \subfloat[][]{
     \includegraphics[]{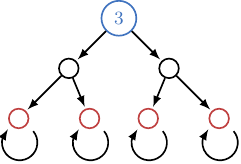}
     \label{empower-a}}\quad\hspace{5pt}
     \subfloat[][]{
     \includegraphics[]{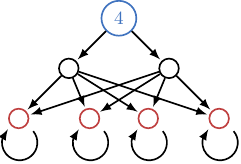}
    \label{empower-b}}

    \caption{Proposed empowerment measures fail to adequately capture how future choice is affected by present actions. In \protect\subref{fail-converge}: $\mathfrak{E}_n(\col{blue}{s_1})$ varies depending on whether $n$ is even; thus, $\lim_{n\to \infty}\mathfrak{E}_n(\col{blue}{s_1})$ does not exist. In \protect\subref{empower-a} and  \protect\subref{empower-b}: $\forall n: \mathfrak{E}_n(\col{blue}{s_3})=\mathfrak{E}_n(\col{blue}{s_4})$, even though $\col{blue}{s_4}$ allows greater control over future state trajectories than $\col{blue}{s_3}$ does. For example, suppose that in both \protect\subref{empower-a} and \protect\subref{empower-b}, the leftmost black state and the rightmost red state have 1 reward while all other states have 0 reward. In \protect\subref{empower-b}, the agent can independently maximize the intermediate black-state reward and the delayed red-state reward. Independent maximization is not possible in \protect\subref{empower-a}.  
    }
    \label{fig:empower_fail}
\end{figure}
 
$\pwrNoDist$ returns intuitive answers in these situations. $\lim_{\gamma\to 1}\pwr[\col{blue}{s_1},\gamma]$ converges by \cref{thm:cont-power}. Consider the obvious involution $\phi$ which takes each state in \cref{empower-a} to its counterpart in \cref{empower-b}. Since $\phi\cdot\Fnd(\col{blue}{s_3})\subsetneq \Fnd(\col{blue}{s_4})=\F(\col{blue}{s_4})$, \cref{prop:more-opt} proves that $\forall \gamma\in [0,1]:\pwr[\col{blue}{s_3},\gamma][\Dbd]\leqMost[][\DSetBd]\pwr[\col{blue}{s_4},\gamma][\Dbd]$, with the proof of \cref{prop:more-opt} showing strict inequality under all $\Diid$ when $\gamma\in(0,1)$.
 
Empowerment can be adjusted to account for these cases, perhaps by considering the channel capacity between the agent's actions and the state trajectories induced by stationary policies. However, since $\pwrNoDist$ is formulated in terms of optimal value, we believe that $\pwrNoDist$ is better suited for {\mdp}s than information-theoretic empowerment is.

%% file: main-tex/appendices/power-nuances.tex
\section{Seeking \texorpdfstring{$\pwrNoDist$}{POWER} can be a detour}\label{app:not-always}
\newcommand*{\sne}{s_3}
\newcommand*{\sn}{s_2}
\newcommand*{\startApp}{\col{blue}{s_1}}
\begin{remark}
The results of appendix \ref{app:proofs} do not depend on this section's results.
\end{remark}

One might suspect that optimal policies tautologically tend to seek $\pwrNoDist$. This intuition is wrong.

\begin{figure}[!ht]\centering
    \includegraphics{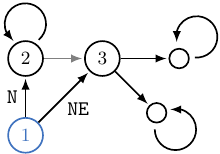}
    \caption{}\label{fig:power-not-ic}
\end{figure} 

\begin{restatable}[Greater $\pwr$ does not imply greater $\Prb_{\Dbd}$]{prop}{PowerNotIC}\label{prop:power-not-ic} Action $a$ seeking more $\pwr$ than $a'$ at state $s$ and $\gamma$ does not imply that $\optprob[\Dbd]{s,a,\gamma}\geq\optprob[\Dbd]{s,a',\gamma}$.
\end{restatable}
\begin{proof}
Consider the environment of \cref{fig:power-not-ic}. Let $\Dist_u\defeq \text{unif}(0,1)$, and consider $\Diid[\Dist_u]$, which has bounded support. Direct computation\footnote{In small deterministic {\mdp}s, the $\pwrNoDist$ and optimality probability of the maximum-entropy reward function distribution can be computed using  \href{https://github.com/loganriggs/Optimal-Policies-Tend-To-Seek-Power}{https://github.com/loganriggs/Optimal-Policies-Tend-To-Seek-Power}.} of the $\pwrNoDist$ expectation (\cref{def:power}) yields $\pwr[\sn,1][\Diid[\Dist_u]]=\frac{3}{4}> \frac{2}{3}=\pwr[\sne,1][\Diid[\Dist_u]]$. Therefore, $\texttt{N}$ seeks more $\pwr[][\Diid[\Dist_u]]$ than $\texttt{NE}$ at state $\startApp$ and $\gamma=1$. 

However, $\optprob[\Diid[\Dist_u]]{\startApp,\texttt{N},1}=\frac{1}{3}<\frac{2}{3}=\optprob[\Diid[\Dist_u]]{\startApp,\texttt{NE},1}$.
\end{proof}

\begin{restatable}[Fraction of orbits which agree on weak optimality]{lem}{fracOrbi}\label{lem:half-orbit-geq}
Let $\distSet\subseteq \Delta(\rewardVS)$, and suppose $f_1,f_2:\Delta(\rewardVS)\to \reals$ are such that $f_1(\D) \geqMost[][\distSet] f_2(\D)$. Then for all $\D\in\distSet$, $\frac{\abs{\set{\D' \in\orbi[\D]\mid f_1(\D')\geq f_2(\D')}}}{\abs{\orbi[\D]}}\geq \dfrac{1}{2}$.
\end{restatable}
\begin{proof}
All $\D'\in \orbi[\D]$ such that $f_1(\D')= f_2(\D')$ satisfy $f_1(\D')\geq f_2(\D')$.

Otherwise, consider the $\D'\in \orbi[\D]$ such that $f_1(\D')\neq f_2(\D')$. By the definition of $\geqMost[][]$ (\cref{def:ineq-most-dists}), at least $\frac{1}{2}$ of these $\D'$ satisfy $f_1(\D')> f_2(\D')$, in which case $f_1(\D')\geq f_2(\D')$. Then the desired inequality follows.
\end{proof}

\begin{restatable}[$\geq_\text{most}$ and trivial orbits]{lem}{identGeqMost}\label{lem:ident-geq-most}
Let $\distSet\subseteq \Delta(\rewardVS)$ and suppose $f_1(\D) \geqMost[][\distSet] f_2(\D)$. For all reward function distributions $\D\in\distSet$ with one-element orbits, $f_1(\D)\geq f_2(\D)$. In particular, $\D$ has a one-element orbit when it distributes reward identically and independently ({\iid}) across states.
\end{restatable}
\begin{proof}
By \cref{lem:half-orbit-geq}, at least half of the elements $\D'\in \orbi[\D]$ satisfy $f_1(\D')\geq f_2(\D')$. But $\abs{\orbi[\D]}=1$, and so $f_1(\D)\geq f_2(\D)$ must hold.

If $\D$ is {\iid}, it has a one-element orbit due to the assumed identical distribution of reward.
\end{proof}

\begin{restatable}[Actions which tend to seek $\pwrNoDist$ do not necessarily tend to be optimal]{prop}{PowerNotICMost}\label{prop:power-not-ic-most} Action $a$ tending to seek more $\pwrNoDist$ than $a'$ at state $s$ and $\gamma$ does not imply that $\optprob[\Dany]{s,a,\gamma}\geqMost \optprob[\Dany]{s,a',\gamma}$.
\end{restatable}
\begin{proof}
Consider the environment of \cref{fig:power-not-ic}. Since $\RSDnd[\sne]\subsetneq\RSD[\sn]$, \cref{RSDSimPower} shows that $\pwr[\sn,1][\Dbd]\geqMost[][\DSetBd] \pwr[\sne,1][\Dbd]$ via $s'\defeq \sne,s\defeq \sn,\phi$ the identity permutation (which is an involution). Therefore, $\texttt{N}$ tends to seek more $\pwrNoDist$ than $\texttt{NE}$ at state $\startApp$ and $\gamma=1$.

If $\optprob[\Dany]{\startApp,\texttt{N},1}\geqMost\optprob[\Dany]{\startApp,\texttt{NE},1}$, then \cref{lem:ident-geq-most} shows that $\optprob[\Diid]{\startApp,\texttt{N},1}\geq\optprob[\Diid]{\startApp,\texttt{NE},1}$ for all $\Diid$. But the proof of \cref{prop:power-not-ic} showed that $\optprob[\Diid[\Dist_u]]{\startApp,\texttt{N},1}<\optprob[\Diid[\Dist_u]]{\startApp,\texttt{NE},1}$ for $\Dist_u\defeq\text{unif}(0,1)$. Therefore, it cannot be true that $\optprob[\Dany]{\startApp,\texttt{N},1}\geqMost\optprob[\Dany]{\startApp,\texttt{NE},1}$.
\end{proof} 

%% file: main-tex/appendices/suboptimal-power.tex
\section{Sub-optimal \texorpdfstring{$\pwrNoDist$}{POWER}}\label{sec:suboptimal-power}
In certain situations, $\pwrNoDist$ returns intuitively surprising verdicts. There exists a policy under which the reader chooses a winning lottery ticket, but it seems wrong to say that the reader has the power to win the lottery with high probability. For various reasons, humans and other bounded agents are generally incapable of computing optimal policies for arbitrary objectives. More formally, consider the rewardless {\mdp} of \cref{fig:subopt-power}.

\begin{figure}[ht]
    \centering
    \includegraphics{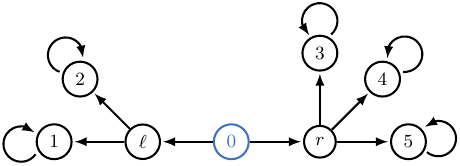}
    \caption{$\col{blue}{s_0}$ is the starting state, and $\abs{\A}=10^{10^{10}}$. At $\col{blue}{s_0}$, half of the actions lead to $s_\ell$, while the other half lead to $s_r$. Similarly, half of the actions at $s_\ell$ lead to $s_1$, while the other half lead to $s_2$. At $s_r$, one action leads to $s_3$, one action leads to $s_4$, and the remaining $10^{10^{10}}-2$ actions lead to $s_5$.  \label{fig:subopt-power}}
\end{figure}  

Consider a model-based RL agent with black-box simulator access to this environment. The agent has no prior information about the model, and so it acts randomly. Before long, the agent has probably learned how to navigate from $\col{blue}{s_0}$ to states $s_\ell$, $s_r$, $s_1$, $s_2$, and $s_5$. However, over any reasonable timescale, it is extremely improbable that the agent discovers the two actions respectively leading to $s_3$ and $s_4$.

Even provided with a reward function $R$ and the discount rate $\gamma$, the agent has yet to learn the relevant environmental dynamics, and so many of its policies are far from optimal. Although  \cref{prop:more-opt} shows that $\forall \gamma \in[0,1]: \pwr[s_\ell,\gamma][\Dbd] \leqMost[][\DSetBd] \pwr[s_r,\gamma][\Dbd]$, there is a sense in which $s_\ell$ gives this agent more power. 

We formalize a bounded agent's goal-achievement capabilities with a function $\pol[]$, which takes as input a reward function and a discount rate, and returns a policy. Informally, this is the best policy which the agent knows about. We can then calculate $\pwr$ with respect to $\pol[]$. 

\begin{restatable}[Suboptimal $\pwrNoDist$]{definition}{powPolicy}\label{def:pow-pol} Let $\Pi_\Delta$ be the set of stationary stochastic policies, and let $\text{pol} : \rewardSpace \times [0,1] \to \Pi_\Delta$. For $\gamma \in [0,1]$,
\begin{align}
\pwrPol{s, \gamma}\defeq \E{\substack{R\sim\Dbd,\\a\sim \pol(s),\\s'\sim T\prn{s,a}}}{\lim_{\gamma^*\to \gamma} (1-\gamma^*)\Vf[\pol]{s',\gamma^*}}.
\end{align}

By \cref{lem:power-id}, $\pwr$ is the special case where $\forall R\in\rewardSpace,\gamma\in[0,1]: \pol\in \optPi$. We define $\pwrPol{}$-seeking similarly as in \cref{def:pow-seek}. 
\end{restatable}

$\pwrPol{\col{blue}{s_0},1}$ increases as the policies returned by $\pol[]$ are improved. We illustrate this by considering the $\Diid$ case.

\begin{enumerate}
\item[$\text{pol}_1$] The model is initially unknown, and so $\forall R,\gamma:\text{pol}_1(R,\gamma)$ is a uniformly random policy. Since $\text{pol}_1$ is constant on its inputs, $\pwrPol[\text{pol}_1][\Diid]{\col{blue}{s_0},1}=\EX$ by the linearity of expectation and the fact that $\Diid$ distributes reward independently and identically across states.
\item[$\text{pol}_2$] The agent knows the dynamics, except that it does not know how to reach $s_3$ or $s_4$. At this point, $\text{pol}_2(R,1)$ navigates from $\col{blue}{s_0}$ to the average-optimal choice among three terminal states: $s_1$, $s_2$, and $s_5$. Therefore, $\pwrPol[\text{pol}_2]{\col{blue}{s_0},1}=\Edraws{3}$. 
\item[$\text{pol}_3$] The agent knows the dynamics, the environment is small enough to solve explicitly, and so $\forall R,\gamma:\text{pol}_3(R,\gamma)$ is an optimal policy. $\text{pol}_3(R,1)$ navigates from $\col{blue}{s_0}$ to the average-optimal choice among all five terminal states. Therefore, $\pwrPol[\text{pol}_3]{\col{blue}{s_0},1}=\Edraws{5}$.
\end{enumerate}

As the agent learns more about the environment and improves $\pol[]$, the agent's $\pwrPol{}$ increases. The agent seeks $\pwrPol[\text{pol}_2]{}$ by navigating to $s_\ell$ instead of $s_r$, but seeks more $\pwr$ by navigating to $s_r$ instead of $s_\ell$. Intuitively, bounded agents gain power by improving $\pol[]$ and by formally seeking $\pwrPol{}$ within the environment.

%% file: main-tex/appendices/theorem-list.tex
\section{Lists of results}\label{app:thms}

\renewcommand\listtheoremname{List of definitions}
\listoftheorems[ignoreall, show={definition}]  

\renewcommand\listtheoremname{List of theorems}
\ignoretheorems{cor-no-num,conjecture,remark,definition}
\listoftheorems[]  

\subsection{Contributions of independent interest}\label{app:contrib}
We developed new basic {\mdp} theory by exploring the structural properties of visit distribution functions. Echoing \citet{wang_dual_2007,wang_stable_2008}, we believe that this area is interesting and underexplored.

\subsubsection{Optimal value theory} 
\Cref{lem:norm-value-lip} shows that $f(\gamma^*)\defeq\lim_{\gamma^*\to \gamma}(1-\gamma^*)\OptVf{s,\gamma^*}$ is Lipschitz continuous on $\gamma\in[0,1]$, with Lipschitz constant depending only on $\lone{R}$.  For all states $s$ and policies $\pi\in\Pi$, \cref{smoothOnPol} shows that $V^\pi_R(s,\gamma)$ is rational on $\gamma$. 

Optimal value has a well-known dual formulation: $\OptVf{s,\gamma}=\max_{\f\in\F(s)}\f(\gamma)^\top \rf$. 
\begin{restatable*}[$\forall\gamma\in[0,1):\OptVf{s,\gamma}=\max_{\f\in\Fnd(s)}\f(\gamma)^\top \rf$]{lem}{optVfFndRestrict}\label{cor:opt-vf-restrict-fnd}
\end{restatable*}
In a fixed rewardless {\mdp}, \cref{cor:opt-vf-restrict-fnd} may enable more efficient computation of optimal value functions for multiple reward functions.

\subsubsection{Optimal policy theory} 
\Cref{transferDiscount} demonstrates how to preserve optimal incentives while changing the discount rate.

\begin{restatable*}[How to transfer optimal policy sets across discount rates]{prop}{transferDiscount}\label{transferDiscount}
Suppose reward function $R$ has optimal policy set $\optPi[R,\gamma]$ at discount rate $\gamma\in(0,1)$. For any $\gamma^*\in(0,1)$, we can construct a reward function $R'$ such that $\optPi[R',\gamma^*]=\optPi$. Furthermore, $\OptVf[R']{\cdot,\gamma^*}=\OptVf[R]{\cdot, \gamma}$.
\end{restatable*} 

\subsubsection{Visit distribution theory} 
While \citet{regan_robust_2010} consider  a visit distribution function $\f\in\F(s)$ to be non-dominated if it is optimal for some reward function in a set $\R\subseteq \rewardVS$, our stricter \cref{def:nd} considers $\f$ to be non-dominated when $\exists\rf\in\rewardVS, \gamma\in(0,1):\f(\gamma)^\top \rf > \max_{\f'\in\F(s)\setminus\set{\f}}\f'(\gamma)^\top\rf$.

%% file: main-tex/appendices/proofs.tex
\section{Theoretical results} \label{app:proofs}
\begin{restatable}[A policy is optimal iff it induces an optimal visit distribution at every state]{lem}{optShare}\label{lem:opt-pol-visit-iff}
Let $\gamma \in (0,1)$ and let $R$ be a reward function. $\pi \in \optPi$ iff $\pi$ induces an optimal visit distribution at every state. 
\end{restatable}
\begin{proof}
By definition, a policy $\pi$ is optimal iff $\pi$ induces the maximal on-policy value at each state, which is true iff $\pi$ induces an optimal visit distribution at every state (by the dual formulation of optimal value functions). 
\end{proof}

\begin{restatable}[Transition matrix induced by a policy]{definition}{transMatrix}\label{def:trans-matrix}
$\mathbf{T}^\pi$ is the transition matrix induced by policy $\pi \in \Pi$, where $\mathbf{T}^\pi\unitvec \defeq T(s,\pi(s))$. $(\mathbf{T}^\pi)^t \unitvec$ gives the probability distribution over the states visited at time step $t$, after following $\pi$ for $t$ steps from $s$.
\end{restatable} 

\begin{restatable}[Properties of visit distribution functions]{prop}{visitDistProperties}\label{prop:visit-dist-prop} 
Let $s,s'\in\St,\fpi{s} \in \F(s)$. 
\begin{enumerate}
    \item $\fpi{s}(\gamma)$ is element-wise non-negative and element-wise monotonically increasing on $\gamma\in[0,1)$. \label{item:mono-increase}
    \item $\forall \gamma \in [0,1): \lone{\fpi{s}(\gamma)}=\geom$.\label{item:lone-visit}
\end{enumerate}
\end{restatable}
\begin{proof}
\Cref{item:mono-increase}: by examination of \cref{def:visit}, $\fpi{s} = \sum_{t=0}^\infty \prn{\gamma\mathbf{T}^\pi}^t\unitvec$. Since each $\prn{\mathbf{T}^\pi}^t$ is left stochastic and $\unitvec$ is the standard unit vector, each entry in each summand is non-negative. Therefore, $\forall \gamma \in [0,1): \fpi{s}(\gamma)^\top \unitvec[s']\geq 0$, and this function monotonically increases on $\gamma$.

\Cref{item:lone-visit}:
\begin{align}
    \lone{\fpi{s}(\gamma)}&=\lone{\sum_{t=0}^\infty \prn{\gamma\mathbf{T}^\pi}^t\unitvec}\\
    &= \sum_{t=0}^\infty\gamma^t\lone{ \prn{\mathbf{T}^\pi}^t\unitvec}\label{eq:non-negative-norm}\\
    &= \sum_{t=0}^\infty \gamma^t\label{eq:left-stoch}\\
    &= \geom.
\end{align}

\Cref{eq:non-negative-norm} follows because all entries in each $\prn{\mathbf{T}^\pi}^t\unitvec$ are non-negative by  \cref{item:mono-increase}. \Cref{eq:left-stoch} follows because each $\prn{\mathbf{T}^\pi}^t$ is left stochastic and $\unitvec$ is a stochastic vector, and so $\lone{\prn{\mathbf{T}^\pi}^t\unitvec}=1$. 
\end{proof}

\begin{restatable}[$\f\in\F(s)$ is multivariate rational on $\gamma$]{lem}{fRat}\label{f-rat}
$\fpi{}\in\F(s)$ is a multivariate rational function on $\gamma\in[0,1)$.
\end{restatable}
\begin{proof}
Let $\rf\in\rewardVS$ and consider $\fpi{}\in\F(s)$. Let $\mathbf{v}^\pi_R$ be the $\OptVf{s,\gamma}$ function in column vector form, with one entry per state value.

By the Bellman equations, $\mathbf{v}^\pi_R = \prn{\mathbf{I}-\gamma\mathbf{T}^\pi}\inv\rf.$ Let $\mathbf{A}_\gamma \defeq \prn{\mathbf{I}-\gamma\mathbf{T}^\pi}\inv$, and for state $s$, form $ \mathbf{A}_{s,\gamma}$ by replacing  $\mathbf{A}_\gamma$'s column for state $s$ with $\rf$. As noted by \citet{lippman1968set}, by Cramer's rule, $V^\pi_R(s,\gamma)=\frac{\det{\mathbf{A}_{s,\gamma}}}{\det\mathbf{A}_\gamma}$ is a rational function with numerator and denominator having degree at most $\abs{\St}$.

In particular, for each state indicator reward function $\unitvec[s_i]$, $V^\pi_{s_i}(s,\gamma)=\fpi{s}(\gamma)^\top\unitvec[s_i]$ is a rational function of $\gamma$ whose numerator and denominator each have degree at most $\abs{\St}$. This implies that $\f^\pi(\gamma)$ is multivariate rational on $\gamma\in[0,1)$.
\end{proof}

\begin{restatable}[On-policy value is rational on $\gamma$]{cor}{smoothOnPol}\label{smoothOnPol}
Let $\pi\in\Pi$ and $R$ be any reward function. $V^\pi_R(s,\gamma)$ is rational on  $\gamma\in[0,1)$. 
\end{restatable}
\begin{proof}
$V^\pi_R(s,\gamma)=\fpi{s}(\gamma)^\top \rf$, and $\f$ is a multivariate rational function of $\gamma$ by \cref{f-rat}. Therefore, for fixed $\rf$, $\fpi{s}(\gamma)^\top \rf$ is a rational function of $\gamma$.
\end{proof}

\subsection{Non-dominated visit distribution functions}\label{sec:nondom}

\begin{restatable}[Continuous reward function distribution]{definition}{dCont}
Results with $\Dcont$ hold for any absolutely continuous reward function distribution.
\end{restatable}

\begin{remark}
We assume $\rewardVS$ is endowed with the standard topology.
\end{remark}

\begin{restatable}[Distinct linear functionals disagree almost everywhere on their domains]{lem}{distinctLin}\label{lem:distinct-lin-prob}
Let $\x,\x'\in\rewardVS$ be distinct. $\optprob[\rf\sim\Dcont]{\x^\top \rf = \x'^\top\rf}=0$.
\end{restatable}
\begin{proof}
$\set{\rf \in\rewardVS\mid (\x-\x')^\top \rf=0}$ is a hyperplane since $\x-\x'\neq \mathbf{0}$. Therefore, it has no interior in the standard topology on $\rewardVS$. Since this empty-interior set is also convex, it has zero Lebesgue measure. By the Radon-Nikodym theorem, it has zero measure under any continuous distribution $\Dcont$.
\end{proof} 

\begin{restatable}[Unique maximization of almost all vectors]{cor}{uniqueMax}\label{cor:distinct-maximized}
Let $X\subsetneq\rewardVS$ be finite. $\optprob[\rf\sim\Dcont]{\abs{\argmax_{\x''\in X} \x''^\top \rf}>1}=0$. 
\end{restatable}
\begin{proof}
Let $\x,\x'\in X$ be distinct. For any $\rf\in\rewardVS$, $\x,\x'\in \argmax_{\x''\in X} \x''^\top \rf$ iff $\x^\top \rf = \x'^\top\rf\geq \max_{\x''\in X\setminus \set{\x,\x'}} \x''^\top \rf$. By \cref{lem:distinct-lin-prob}, $\x^\top \rf = \x'^\top\rf$ holds with probability 0 under any $\Dcont$.
\end{proof} 

\subsubsection{Generalized non-domination results}
Our formalism includes both $\Fnd(s)$ and $\RSDnd$; we therefore prove results that are applicable to both. 

\begin{restatable}[Non-dominated linear functionals]{definition}{ndLinFunc}\label{def:nd-lin-func}
Let $X\subsetneq \rewardVS$ be finite. $\ND{X}\defeq \set{\x\in X\mid \exists \rf \in \rewardVS: \x^\top \rf > \max_{\x'\in X\setminus \set{\x}} \x'^\top \rf}$.
\end{restatable}

\begin{restatable}[All vectors are maximized by a non-dominated linear functional]{lem}{allRFMaxND}\label{lem:all-rf-max-nd}
Let $\rf \in \rewardVS$ and let $X\subsetneq \rewardVS$ be finite and non-empty. $\exists \x^* \in \ND{X}: \x^{*\top}\rf=\max_{\x \in X} \x^\top \rf$.
\end{restatable}
\begin{proof}
Let $A(\rf\mid X)\defeq \argmax_{\x\in X} \x^\top \rf=\set{\x_1,\ldots, \x_n}$. Then 
\begin{align}
    \x_1^\top \rf = \cdots= \x_n^\top \rf> \max_{\x' \in X \setminus A(\rf\mid X)} \x'^\top \rf.\label{eq:expr-cont-nd-func}
\end{align}

In \cref{eq:expr-cont-nd-func}, each $\x^\top \rf$ expression is linear on $\rf$. The $\max$ is piecewise linear on $\rf$ since it is the maximum of a finite set of linear functionals. In particular, all expressions in \cref{eq:expr-cont-nd-func} are continuous on $\rf$, and so we can find some $\delta>0$ neighborhood $B(\rf,\delta)$ such that $\forall \rf' \in B(\rf,\delta): \max_{\x_i\in A(\rf\mid X)} \x_i^\top \rf' > \max_{\x'\in X\setminus A(\rf\mid X)} \x'^\top \rf'$.

But almost all $\rf' \in B(\rf,\delta)$ are maximized by a unique functional $\x^*$ by \cref{cor:distinct-maximized}; in particular, at least one such $\rf''$ exists. Formally, $\exists \rf''\in B(\rf,\delta):\x^{*\top} \rf''>\max_{\x' \in X\setminus \set{\x^*}} \x'^\top \rf''$. Therefore, $\x^*\in \ND{X}$ by \cref{def:nd-lin-func}. 

$\x^{*\top} \rf'\geq\max_{\x_i\in A(\rf\mid X)} \x_i^\top \rf' > \max_{\x'\in X\setminus A(\rf\mid X)} \x'^\top \rf'$, with the strict inequality following because  $\rf'' \in B(\rf,\delta)$. These inequalities imply that $\x^*\in A(\rf\mid X)$.
\end{proof}

\begin{restatable}[Maximal value is invariant to restriction to non-dominated functionals]{cor}{ndFuncIndif}\label{cor:nd-func-indif}
Let $\rf \in \rewardVS$ and let $X\subsetneq \rewardVS$ be finite. $\max_{\x\in X} \x^\top \rf=\max_{\x\in \ND{X}} \x^\top \rf$. 
\end{restatable}
\begin{proof}
If $X$ is empty, holds trivially. Otherwise, apply \cref{lem:all-rf-max-nd}.
\end{proof}

\begin{restatable}[How non-domination containment affects optimal value]{lem}{ndOptContain}\label{lem:nd-opt-contain}
Let $\rf \in \rewardVS$ and let $X,X'\subsetneq \rewardVS$ be finite. 
\begin{enumerate}
    \item If $\ND{X}\subseteq X'$, then $\max_{\x\in X} \x^\top \rf\leq\max_{\x'\in X'} \x'^\top \rf$.\label{item:ND-contain-max}
    \item If $\ND{X}\subseteq X'\subseteq X$, then $\max_{\x\in X} \x^\top \rf=\max_{\x'\in X'} \x'^\top \rf$.\label{item:ND-contain-max-2}
\end{enumerate}
\end{restatable}
\begin{proof}
\Cref{item:ND-contain-max}: 
\begin{align}
\max_{\x\in X} \x^\top \rf&=\max_{\x\in \ND{X}} \x^\top \rf\label{eq:item-max}\\
&\leq \max_{\x'\in X'} \x'^\top \rf.\label{eq:item-max-2}
\end{align}
\Cref{eq:item-max} follows by \cref{cor:nd-func-indif}. \Cref{eq:item-max-2} follows because $\ND{X}\subseteq X'$.

\Cref{item:ND-contain-max-2}: by \cref{item:ND-contain-max}, $\max_{\x\in X} \x^\top \rf\leq\max_{\x'\in X'} \x'^\top \rf$. Since $X'\subseteq X$, we also have $\max_{\x\in X} \x^\top \rf\geq\max_{\x'\in X'} \x'^\top \rf$, and so equality must hold.
\end{proof}

\begin{restatable}[Non-dominated vector functions]{definition}{ndVecFun}\label{def:nd-vec-func} 
Let $I\subseteq \reals$ and let $F\subsetneq \prn{\rewardVS}^I$ be a finite set of vector-valued functions on $I$. $\ND{F}\defeq \set{\f\in F\mid \exists \gamma\in I, \rf \in \rewardVS: \f(\gamma)^\top \rf> \max_{\f'\in F\setminus \set{\f}}\f'(\gamma)^\top \rf}$.
\end{restatable}

\begin{remark}
$\Fnd(s)=\ND{\F(s)}$ by \cref{def:nd}.
\end{remark}

\begin{restatable}[Affine transformation of visit distribution sets]{definition}{affTransfShorthand}\label{def:aff-transf-short}
For notational convenience, we define set-scalar multiplication and set-vector addition on $X\subseteq \rewardVS$: for $c\in\reals$, $cX\defeq \set{c\x\mid\x\in X}$. For $\av \in \rewardVS$, $X+\av\defeq \set{\x+\av\mid \x \in X}$. Similar operations hold when $X$ is a set of vector functions $\reals\mapsto \rewardVS$.
\end{restatable}

\begin{restatable}[Invariance of non-domination under positive affine transform]{lem}{posAffNDInvar}\label{lem:pos-aff-nd-invar}\hfill
\begin{enumerate}
    \item Let $X\subsetneq \rewardVS$ be finite. If $\x\in \ND{X}$, then $\forall c>0, \av\in\rewardVS: (c\x+\av) \in \ND{cX+\av}$.\label{item:invar-vectors}
    \item Let $I\subseteq \reals$ and let $F\subsetneq \prn{\rewardVS}^I$ be a finite set of vector-valued functions on $I$. If $\f \in \ND{F}$, then $\forall c>0, \av\in\rewardVS: (c\f+\av) \in \ND{cF+\av}$.\label{item:invar-vector-fns}
\end{enumerate}
\end{restatable}
\begin{proof}
\Cref{item:invar-vectors}: Suppose $\x\in \ND{X}$ is strictly optimal for $\rf\in\rewardVS$. Then let $c>0,\av\in\rewardVS$ be arbitrary, and define $b\defeq \av^\top \rf$. 
\begin{align}
    \x^\top\rf &> \max_{\x'\in X\setminus\set{\x}} \x'^\top \rf\\
    c\x^\top\rf+b &> \max_{\x'\in X\setminus\set{\x}} c\x'^\top \rf+b\label{eq:rescale-c}\\
    (c\x+\av)^\top\rf &> \max_{\x'\in X\setminus\set{\x}} (c\x'+\av)^\top \rf\label{eq:b-defn}\\
    (c\x+\av)^\top\rf &> \max_{\x''\in \prn{cX+\av}\setminus\set{c\x+\av}} \x''^\top \rf.
\end{align}
\Cref{eq:rescale-c} follows because $c>0$. \Cref{eq:b-defn} follows by the definition of $b$.

\Cref{item:invar-vector-fns}: If $\f\in\ND{F}$, then by \cref{def:nd-vec-func}, there exist $\gamma\in I, \rf\in\rewardVS$ such that
\begin{align}
    \f(\gamma)^\top \rf> \max_{\f'\in F\setminus \set{\f}}\f'(\gamma)^\top \rf.
\end{align}

Apply \cref{item:invar-vectors} to conclude
\begin{align}
    (c\f(\gamma)+\av)^\top \rf> \max_{(c\f'+\av)\in (cF+\av)\setminus \set{c\f+\av}}(c\f'(\gamma)+\av)^\top \rf.
\end{align}

Therefore, $(c\f+\av) \in \ND{cF+\av}$.
\end{proof}

\subsubsection{Inequalities which hold under most reward function distributions}

\ineqMost*

\begin{restatable}[Helper lemma for demonstrating $\geqMost$]{lem}{helperGeqMost}\label{lem:helper-geq-most}
Let $\distSet\subseteq \Delta(\rewardVS)$. If $\exists \phi\in\mdpPermGroup$ such that for all $\D\in\distSet$, $f_1\prn{\D}< f_2\prn{\D}$ implies that $f_1\prn{\phi\cdot \D}> f_2\prn{\phi\cdot \D}$, then $f_1(\D) \geqMost[][\distSet] f_2(\D)$.
\end{restatable}
\begin{proof}
Since $\phi$ does not belong to the stabilizer of $\mdpPermGroup$, $\phi$ acts injectively on $\orbi[\D]$. By assumption on $\phi$, the image of $\{\D'\in\orbi[\D]\mid f_1(\D')<f_2(\D')\}$ under $\phi$ is a subset of $\{\D'\in\orbi[\D]\mid f_1(\D')>f_2(\D')\}$. Since $\phi$ is injective, $\abs{\{\D'\in\orbi[\D]\mid f_1(\D')<f_2(\D')\}}\leq \abs{\{\D'\in\orbi[\D]\mid f_1(\D')>f_2(\D')\}}$. $f_1(\D) \geqMost[][\distSet] f_2(\D)$ by \cref{def:ineq-most-dists}.
\end{proof}

\begin{restatable}[A helper result for expectations of functions]{lem}{helperPerm}\label{lem:helper-perm}
Let $B_1,\ldots,B_n\subsetneq \rewardVS$ be finite and let $\distSet \subseteq \Delta(\rewardVS)$. Suppose $f$ is a function of the form
\begin{align}
    f\prn{B_1,\ldots,B_n \mid \D}=\E{\rf\sim \D}{g\prn{\max_{\bv_1\in B_1} \bv_1^\top \rf,\ldots, \max_{\bv_n\in B_n} \bv_n^\top \rf}}\label{eq:helper}
\end{align}
for some function $g$, and that $f$ is well-defined for all $\D \in \distSet$. Let $\phi$ be a state permutation. Then
\begin{equation}
    f\prn{B_1,\ldots,B_n \mid \D}=f\prn{\phi\cdot B_1 ,\ldots,\phi \cdot B_n \mid \phi \cdot\D}.
\end{equation}
\end{restatable}
\begin{proof} Let distribution $\D$ have probability measure $F$, and let $\phi\cdot \D$ have probability measure $F_\phi$. 
\begin{align}
    &f\prn{B_1,\ldots,B_n \mid \D}\\
    \defeq{}&\E{\rf\sim \D}{g\prn{\max_{\bv_1\in B_1} \bv_1^\top \rf,\ldots, \max_{\bv_n\in B_n} \bv_n^\top \rf}}\\
    \defeq{}&\int_{\rewardVS} g\prn{\max_{\bv_1\in B_1} \bv_1^\top \rf,\ldots, \max_{\bv_n\in B_n} \bv_n^\top \rf} \dF[\rf][F]\\
    ={}&\int_{\rewardVS} g\prn{\max_{\bv_1\in B_1} \bv_1^\top \rf,\ldots, \max_{\bv_n\in B_n} \bv_n^\top \rf} \dF[\permute\rf][F_\phi]\label{eq:permute-prob-meas}\\
    ={}&\int_{\rewardVS} g\prn{\max_{\bv_1\in B_1} \bv_1^\top \prn{\permute\inv\rf'},\ldots, \max_{\bv_n\in B_n} \bv_n^\top \prn{\permute\inv\rf'}} \abs{\det \permute}\dF[\rf'][F_\phi]\label{eq:change-of-variables}\\
    ={}&\int_{\rewardVS} g\prn{\max_{\bv_1\in B_1} \prn{\permute\bv_1}^\top \rf',\ldots, \max_{\bv_n\in B_n} \prn{\permute\bv_n}^\top \rf'} \dF[\rf'][F_\phi]\label{eq:change-of-variables-2}\\
    ={}&\int_{\rewardVS} g\prn{\max_{\bv_1'\in \phi\cdot B_1} \bv_1'^{\top} \rf',\ldots, \max_{\bv_n'\in \phi\cdot B_n} \bv_n'^{\top} \rf'} \dF[\rf'][F_\phi]\\
    \eqdef{}& f\prn{\phi \cdot B_1,\ldots,\phi\cdot B_n \mid \phi\cdot \D}.
\end{align}

\Cref{eq:permute-prob-meas} follows by the definition of $F_\phi$ (\cref{def:pushforward-permute}). \Cref{eq:change-of-variables} follows by substituting $\rf'\defeq \permute\rf$. \Cref{eq:change-of-variables-2} follows from the fact that all permutation matrices have unitary determinant and are orthogonal (and so $(\permute\inv)^\top=\permute$). 
\end{proof}

\begin{restatable}[Support of $\Dany$]{definition}{DefSupportDist}\label{def:support-dist}
Let $\Dany$ be any reward function distribution. $\supp[\Dany]$ is the smallest closed subset of $\rewardVS$ whose complement has measure zero under $\Dany$. 
\end{restatable} 

\begin{restatable}[Linear functional optimality probability]{definition}{linFNProb}
For finite $A,B\subsetneq \rewardVS$, the \emph{probability under $\Dany$ that $A$ is optimal over $B$} is $\phelper{A\geq B}[\Dany]\defeq \optprob[\rf \sim \Dany]{\max_{\av\in A} \av^\top \rf \geq \max_{\bv\in B} \bv^\top \rf}$. 
\end{restatable}

\begin{restatable}[Non-dominated linear functionals and their optimality probability]{prop}{helperPositiveProb}\label{prop:helper-positive-prob} 
Let $A\subsetneq \rewardVS$ be finite. If $\exists b<c:[b,c]^{\abs{\St}} \subseteq \supp[\Dany]$, then $\av\in\ND{A}$ implies that $\av$ is strictly optimal for a set of reward functions with positive measure under $\Dany$. 
\end{restatable}
\begin{proof}
Suppose $\exists b<c:[b,c]^{\abs{\St}} \subseteq \supp[\Dany]$. If $\av\in\ND{A}$, then let $\rf$ be such that $\av^\top \rf > \max_{\av'\in A\setminus\set{\av}} \av'^\top \rf$. For $a_1>0, a_2\in \reals$, positively affinely transform $\rf'\defeq a_1\rf + a_2\mathbf{1}$ (where $\mathbf{1}\in\rewardVS$ is the all-ones vector) so that $\rf' \in (b,c)^{\abs{\St}}$. 

Note that $\av$ is still strictly optimal for $\rf'$:
\begin{equation}
    \av^\top \rf > \max_{\av'\in A\setminus\set{\av}} \av'^\top \rf \iff \av^\top \rf' > \max_{\av'\in A\setminus\set{\av}} \av'^\top \rf'.\label{eq:cont-lin-ineq}
\end{equation}

Furthermore, by the continuity of both terms on the right-hand side of \cref{eq:cont-lin-ineq}, $\av$ is strictly optimal for reward functions in some open neighborhood $N$ of $\rf'$. Let $N'\defeq N\cap (b,c)^{\abs{\St}}$. $N'$ is still open in $\rewardVS$ since it is the intersection of two open sets $N$ and $(b,c)^{\abs{\St}}$. 

$\Dany$ must assign positive probability measure to all open sets in its support; otherwise, its support would exclude these zero-measure sets by \cref{def:support-dist}. Therefore, $\Dany$ assigns positive probability to $N'\subseteq \supp[\Dany]$. 
\end{proof}   

\begin{restatable}[Expected value of similar linear functional sets]{lem}{ndFuncIncr}\label{lem:nd-func-incr-pwr}
Let $A,B\subsetneq \rewardVS$ be finite, let $A'$ be such that $\ND{A}\subseteq A'\subseteq A$, and let $g:\reals\to\reals$ be an increasing function. If $B$ contains a copy $B'$ of $A'$ via $\phi$, then
\begin{equation}
    \E{\rf \sim \Dbd}{g\prn{\max_{\av\in A}\av^\top \rf}} \leq \E{\rf \sim \phi\cdot \Dbd}{g\prn{\max_{\bv\in B}\bv^\top \rf}}.\label{eq:nd-func-incr}
\end{equation} 

If $\ND{B}\setminus B'$ is empty, then \cref{eq:nd-func-incr} is an equality. If $\ND{B}\setminus B'$ is non-empty, $g$ is strictly increasing, and $\exists b<c: (b,c)^{\abs{\St}}\subseteq\supp[\Dbd]$, then \cref{eq:nd-func-incr} is strict.
\end{restatable}
\begin{proof}
Because $g:\reals\to\reals$ is increasing, it is measurable (as is $\max$). Therefore, the relevant expectations exist for all $\Dbd$.
\begin{align}
     \E{\rf \sim \Dbd}{g\prn{\max_{\av\in A}\av^\top \rf}} &=\E{\rf \sim \Dbd}{g\prn{\max_{\av\in A'}\av^\top \rf}}\label{eq:nd-func-incr-restr}\\
     &=\E{\rf \sim \phi\cdot \Dbd}{g\prn{\max_{\av\in \phi\cdot A'}\av^\top \rf}}\label{eq:apply-helper}\\
     &=\E{\rf \sim \phi\cdot \Dbd}{g\prn{\max_{\bv\in B'}\bv^\top \rf}}\label{eq:sim-helper}\\
     &\leq\E{\rf \sim \phi\cdot \Dbd}{g\prn{\max_{\bv\in B}\bv^\top \rf}}.\label{eq:sim-inequality}
\end{align}
\Cref{eq:nd-func-incr-restr} holds because $\forall \rf\in\rewardVS:\max_{\av\in A}\av^\top \rf=\max_{\av\in A'}\av^\top \rf$ by \cref{lem:nd-opt-contain}'s \cref{item:ND-contain-max-2} with $X\defeq A$, $X'\defeq A'$. \Cref{eq:apply-helper} holds by \cref{lem:helper-perm}. \Cref{eq:sim-helper} holds by the definition of $B'$. Furthermore, our assumption on $\phi$ guarantees that $B'\subseteq B$. Therefore, $\max_{\bv\in B'}\bv^\top \rf\leq \max_{\bv\in B}\bv^\top \rf$, and so \cref{eq:sim-inequality} holds by the fact that $g$ is an increasing function. Then \cref{eq:nd-func-incr} holds. 

If $\ND{B}\setminus B'$ is empty, then $\ND{B}\subseteq B'$. By assumption, $B'\subseteq B$. Then apply \cref{lem:nd-opt-contain} \cref{item:ND-contain-max-2} with $X\defeq B$, $X'\defeq B'$ in order to conclude that \cref{eq:sim-inequality} is an equality. Then \cref{eq:nd-func-incr} is also an equality.

Suppose that $g$ is strictly increasing, $\ND{B}\setminus B'$ is non-empty, and $\exists b<c: (b,c)^{\abs{\St}}\subseteq\supp[\Dbd]$. Let $\x \in \ND{B}\setminus B'$.
\begin{align}
     \E{\rf \sim \phi\cdot \Dbd}{g\prn{\max_{\bv\in B'}\bv^\top \rf}}&< \E{\rf \sim \phi\cdot \Dbd}{g\prn{\max_{\av\in B'\cup\set{\x}}\bv^\top \rf}}\label{eq:sim-helper-2}\\
     &\leq\E{\rf \sim \phi\cdot \Dbd}{g\prn{\max_{\bv\in B}\bv^\top \rf}}.\label{eq:sim-inequality-2}
\end{align}

$\x$ is strictly optimal for a positive-probability subset of $\supp[\Dbd]$ by \cref{prop:helper-positive-prob}. Since  $g$ is strictly increasing, \cref{eq:sim-helper-2} is strict. Therefore, we conclude that \cref{eq:nd-func-incr} is strict.
\end{proof}

\begin{restatable}[For continuous\ {\iid} distributions $\Diid$, $\exists b<c: (b,c)^{\abs{\St}}\subseteq \optSupp(\Diid)$]{lem}{IIDContains}\label{lem:iid-contains}
\end{restatable}
\begin{proof}
$\Diid\defeq \Dist^{\abs{\St}}$. Since the state reward distribution $\Dist$ is continuous, $\Dist$ must have support on some open interval $(b,c)$. Since $\Diid$ is {\iid} across states, $(b,c)^{\abs{\St}}\subseteq \supp[\Diid]$.
\end{proof} 

\begin{restatable}[Bounded, continuous {\iid} reward]{definition}{BCIIDDef}
$\DSetBCiid$ is the set of $\Diid$ which equal $\Dist^{\abs{\St}}$ for some continuous, bounded-support distribution $\Dist$ over $\reals$.
\end{restatable}
 
\begin{restatable}[Expectation superiority lemma]{lem}{expectSuperior}\label{lem:expect-superior}
Let $A,B\subsetneq \rewardVS$ be finite and let $g:\reals\to \reals$ be an increasing function. If $B$ contains a copy $B'$ of $\ND{A}$ via $\phi$, then 
\begin{align}
    \E{\rf \sim \Dbd}{g\prn{\max_{\av\in A}\av^\top \rf}}\leqMost[][\DSetBd]  \E{\rf \sim \Dbd}{g\prn{\max_{\bv\in B}\bv^\top \rf}}.\label{eq:permute-superior}
\end{align}

Furthermore, if $g$ is strictly increasing and $\ND{B}\setminus \phi\cdot \ND{A}$ is non-empty, then \cref{eq:permute-superior} is strict for all $\Diid\in\DSetBCiid$. In particular, $\E{\rf \sim \Dbd}{g\prn{\max_{\av\in A}\av^\top \rf}}\not\geqMost[][\DSetBd]  \E{\rf \sim \Dbd}{g\prn{\max_{\bv\in B}\bv^\top \rf}}$.
\end{restatable}
\begin{proof}
Because $g:\reals\to\reals$ is increasing, it is measurable (as is $\max$). Therefore, the relevant expectations exist for all $\Dbd$.

Suppose that $\Dbd$ is such that $\E{\rf \sim \Dbd}{g\prn{\max_{\bv\in B}\bv^\top \rf}}< \E{\rf \sim \Dbd}{g\prn{\max_{\av\in A}\av^\top \rf}}$. 
\begin{align}
    \E{\rf \sim \phi\cdot \Dbd}{g\prn{\max_{\av\in A}\av^\top \rf}}&\leq\E{\rf \sim \phi^2\cdot \Dbd}{g\prn{\max_{\bv\in B}\bv^\top \rf}}\label{eq:first-leq}\\
    &=\E{\rf \sim \Dbd}{g\prn{\max_{\bv\in B}\bv^\top \rf}}\label{eq:involve}\\
    &<\E{\rf \sim \Dbd}{g\prn{\max_{\av\in A}\av^\top \rf}}\label{eq:superior-leq}\\
    &\leq  \E{\rf \sim \phi\cdot \Dbd}{g\prn{\max_{\bv\in B}\bv^\top \rf}}.\label{eq:second-leq}
\end{align}
\Cref{eq:first-leq} follows by applying \cref{lem:nd-func-incr-pwr} with permutation $\phi$ and $A'\defeq \ND{A}$. \Cref{eq:involve} follows because involutions satisfy $\phi\inv=\phi$, and $\phi^2$ is therefore the identity. \Cref{eq:superior-leq} follows because we assumed that $\E{\rf \sim \Dbd}{g\prn{\max_{\bv\in B}\bv^\top \rf}}< \E{\rf \sim \Dbd}{g\prn{\max_{\av\in A}\av^\top \rf}}$. \Cref{eq:second-leq} follows by applying \cref{lem:nd-func-incr-pwr} with permutation $\phi$ and and $A'\defeq \ND{A}$. By \cref{lem:helper-geq-most}, \cref{eq:permute-superior} holds.

Suppose $g$ is strictly increasing and $\ND{B}\setminus B'$ is non-empty. Let $\phi'\in\mdpPermGroup$.
\begin{align}
\E{\rf \sim \phi'\cdot \Diid}{g\prn{\max_{\av\in A}\av^\top \rf}}&=\E{\rf \sim \Diid}{g\prn{\max_{\av\in A}\av^\top \rf}}\label{eq:same-iid-dist-exp}\\
&<\E{\rf \sim \phi\cdot \Diid}{g\prn{\max_{\bv\in B}\bv^\top \rf}}\label{eq:sim-iid-exp}\\
&=\E{\rf \sim \phi'\cdot \Diid}{g\prn{\max_{\bv\in B}\bv^\top \rf}}.\label{eq:same-iid-dist-exp-2}
\end{align}

\Cref{eq:same-iid-dist-exp} and \cref{eq:same-iid-dist-exp-2} hold because $\Diid$ distributes reward identically across states: $\forall \phi_x\in\mdpPermGroup: \phi_x\cdot \Diid=\Diid$. By \cref{lem:iid-contains}, $\exists b<c: (b,c)^{\abs{\St}}\subseteq \supp[\Diid]$. Therefore, apply \cref{lem:nd-func-incr-pwr} with $A'\defeq \ND{A}$ to conclude that \cref{eq:sim-iid-exp} holds.

Therefore, $\forall \phi' \in \mdpPermGroup: \E{\rf \sim \phi'\cdot \Diid}{g\prn{\max_{\av\in A}\av^\top \rf}}<\E{\rf \sim \phi'\cdot \Diid}{g\prn{\max_{\bv\in B}\bv^\top \rf}}$, and so $\E{\rf \sim \Dbd}{g\prn{\max_{\av\in A}\av^\top \rf}}\not\geqMost[][\DSetBd]  \E{\rf \sim \Dbd}{g\prn{\max_{\bv\in B}\bv^\top \rf}}$ by \cref{def:ineq-most-dists}.
\end{proof}

\begin{restatable}[Indicator function]{definition}{indicDef}
Let $L$ be a predicate which takes input $x$. $\indic{L(x)}$ is the function which returns 1 when $L(x)$ is true, and 0 otherwise.
\end{restatable} 

\begin{restatable}[Optimality probability inclusion relations]{lem}{optprbInclusion}\label{lem:inclusion-opt}
Let $X,Y\subsetneq \rewardVS$ be finite and suppose $Y'\subseteq Y$. \begin{equation}
    \phelper{X\geq Y}[\Dany]\leq\phelper{X\geq Y'}[\Dany]\leq \phelper{X\cup\prn{Y\setminus Y'}\geq Y}[\Dany].\label{eq:inclusion-opt}
\end{equation}

If $\exists b<c: (b,c)^{\abs{\St}}\subseteq \supp[\Dany]$, $X\subseteq Y$, and $\ND{Y}\cap \prn{Y\setminus Y'}$ is non-empty, then the second inequality is strict.
\end{restatable}
\begin{proof}
\begin{align}
    \phelper{X\geq Y}[\Dany]&\defeq \E{\rf \sim \Dany}{\indic{\max_{\x\in X} \x^\top \rf \geq \max_{\mathbf{y} \in Y} \mathbf{y}^\top \rf}}\\
    &\leq \E{\rf \sim \Dany}{\indic{\max_{\x\in X} \x^\top \rf \geq \max_{\mathbf{y} \in Y'} \mathbf{y}^\top \rf}}\label{eq:leq-Y-contain}\\
    &\leq \E{\rf \sim \Dany}{\indic{\max_{\x\in X\cup (Y\setminus Y')} \x^\top \rf \geq \max_{\mathbf{y} \in Y'} \mathbf{y}^\top \rf}}\label{eq:leq-Y-union}\\
    &= \E{\rf \sim \Dany}{\indic{\max_{\x\in X\cup (Y\setminus Y')} \x^\top \rf \geq \max_{\mathbf{y} \in Y'\cup (Y\setminus Y')} \mathbf{y}^\top \rf}}\label{eq:leq-Y-union-2}\\
    &=\E{\rf \sim \Dany}{\indic{\max_{\x\in X\cup (Y\setminus Y')} \x^\top \rf \geq \max_{\mathbf{y} \in Y} \mathbf{y}^\top \rf}}\\
    &\eqdef \phelper{X\cup\prn{Y\setminus Y'}\geq Y}[\Dany].
\end{align}
\Cref{eq:leq-Y-contain} follows because $\forall \rf \in \rewardVS: \indic{\max_{\x\in X} \x^\top \rf \geq \max_{\mathbf{y} \in Y} \mathbf{y}^\top \rf}\leq \indic{\max_{\x\in X} \x^\top \rf \geq \max_{\mathbf{y} \in Y'} \mathbf{y}^\top \rf}$ since $Y' \subseteq Y$; note that \cref{eq:leq-Y-contain} equals $\phelper{X\geq Y'}[\Dany]$, and so the first inequality of \cref{eq:inclusion-opt} is shown. \Cref{eq:leq-Y-union} holds because $\forall \rf \in \rewardVS: \indic{\max_{\x\in X} \x^\top \rf \geq \max_{\mathbf{y} \in Y'} \mathbf{y}^\top \rf}\leq \indic{\max_{\x\in X\cup (Y\setminus Y')} \x^\top \rf \geq \max_{\mathbf{y} \in Y'} \bv^\top \rf}$. 

Suppose $\exists b<c: (b,c)^{\abs{\St}}\subseteq \supp[\Dany]$, $X\subseteq Y$, and $\ND{Y}\cap \prn{Y\setminus Y'}$ is non-empty. Let $\mathbf{y}^*\in \ND{Y}\cap \prn{Y\setminus Y'}$. By \cref{prop:helper-positive-prob}, $\mathbf{y}^*$ is strictly optimal on a subset of $\supp[\Dany]$ with positive measure under $\Dany$. In particular, for a set of $\rf^*$ with positive measure under $\Dany$, we have $\mathbf{y}^{*\top}\rf^*> \max_{\mathbf{y}\in Y'} \mathbf{y}^\top \rf^*.$

Then \cref{eq:leq-Y-union} is strict, and therefore the second inequality of \cref{eq:inclusion-opt} is strict as well.
\end{proof}

\begin{restatable}[Optimality probability of similar linear functional sets]{lem}{optProbSim}\label{lem:sim-lin-func-opt}
Let $A,B,C\subsetneq \rewardVS$ be finite, and let $Z\subseteq \rewardVS$ be such that $\ND{C}\subseteq Z\subseteq C$. If $\ND{A}$ is similar to $B'\subseteq B$ via $\phi$ such that $\phi\cdot \prn{Z\setminus \prn{B\setminus B'}}=Z\setminus \prn{B\setminus B'}$, then 
\begin{equation}
    \phelper{A\geq C}[\Dany]\leq \phelper{B\geq C}[\phi\cdot \Dany].\label{eq:sim-nd-func}
\end{equation} 
If $B'=B$, then \cref{eq:sim-nd-func} is an equality. If $\exists b<c: (b,c)^{\abs{\St}}\subseteq \supp[\Dany]$, $B'\subseteq C$, and $\ND{C}\cap \prn{B\setminus B'}$ is non-empty, then \cref{eq:sim-nd-func} is strict.
\end{restatable}
\begin{proof}
\begin{align}
    \phelper{A\geq C}[\Dany]&=\phelper{A\geq Z}[\Dany]\label{eq:nd-restrict-opt-lin}\\ 
    &=\phelper{\ND{A}\geq Z}[\Dany]\label{eq:nd-restrict-opt-lin-2}\\
    &\leq \phelper{\ND{A}\geq Z\setminus\prn{B\setminus B'}}[\Dany]\label{eq:fewer-requirements}\\
    &= \phelper{\phi\cdot \ND{A}\geq \phi\cdot Z\setminus\prn{B\setminus B'}}[\phi\cdot \Dany]\label{eq:phi-helper-nd-opt}\\
    &= \phelper{B'\geq Z\setminus\prn{B\setminus B'}}[\phi\cdot \Dany]\label{eq:phi-assumptions}\\
    &\leq \phelper{B'\cup \prn{B\setminus B'}\geq Z}[\phi\cdot \Dany]\label{eq:reunion}\\
    &= \phelper{B\geq C}[\phi\cdot \Dany].\label{eq:final-phelper-leq}
\end{align}

\Cref{eq:nd-restrict-opt-lin} and \cref{eq:final-phelper-leq} follow by \cref{lem:nd-opt-contain}'s \cref{item:ND-contain-max-2} with $X\defeq C$, $X'\defeq Z$. Similarly, \cref{eq:nd-restrict-opt-lin-2} follows by \cref{lem:nd-opt-contain}'s \cref{item:ND-contain-max-2} with $X\defeq A$, $X'\defeq \ND{A}$. \Cref{eq:fewer-requirements} follows by applying the first inequality of \cref{lem:inclusion-opt} with $X\defeq \ND{A}, Y\defeq Z, Y'\defeq Z\setminus (B\setminus B')$. \Cref{eq:phi-helper-nd-opt} follows by applying \cref{lem:helper-perm} to \cref{eq:nd-restrict-opt-lin} with permutation $\phi$. 

\Cref{eq:phi-assumptions} follows by our assumptions on $\phi$. \Cref{eq:reunion} follows because by applying the second inequality of \cref{lem:inclusion-opt} with $X\defeq B', Y\defeq \ND{C}, Y'\defeq \ND{C}\setminus (B\setminus B')$.

Suppose $B'=B$. Then $B\setminus B'=\emptyset$, and so \cref{eq:fewer-requirements} and \cref{eq:reunion} are trivially equalities. Then \cref{eq:sim-nd-func} is an equality.

Suppose $\exists b<c: (b,c)^{\abs{\St}}\subseteq \supp[\Dany]$; note that $(b,c)^{\abs{\St}}\subseteq\supp[\phi\cdot \Dany]$, since such support must be invariant to permutation. Further suppose that $B'\subseteq C$ and that $\ND{C}\cap \prn{B\setminus B'}$ is non-empty. Then letting $X\defeq B', Y\defeq Z, Y'\defeq Z\setminus (B\setminus B')$ and noting that $\ND{\ND{Z}}=\ND{Z}$, apply \cref{lem:inclusion-opt} to \cref{eq:reunion} to conclude that \cref{eq:sim-nd-func} is strict.
\end{proof}

\begin{restatable}[Optimality probability superiority lemma]{lem}{optProbSup}\label{lem:opt-prob-superior}
Let $A,B,C\subsetneq \rewardVS$ be finite, and let $Z$ satisfy $\ND{C}\subseteq Z \subseteq C$. If $B$ contains a copy $B'$ of $\ND{A}$ via $\phi$ such that $\phi\cdot\prn{Z\setminus \prn{B\setminus B'}}=Z\setminus \prn{B\setminus B'}$, then $\phelper{A\geq C}[\Dany] \leqMost[][\DSetAny] \phelper{B \geq C}[\Dany]$.

If $B'\subseteq C$ and $\ND{C}\cap \prn{B\setminus B'}$ is non-empty, then the inequality is strict for all $\Diid\in\DSetBCiid$ and  $\phelper{A\geq C}[\Dany] \not\geqMost[][\DSetAny] \phelper{B \geq C}[\Dany]$.
\end{restatable}
\begin{proof}
Suppose $\Dany$ is such that $\phelper{B \geq C}[\Dany] < \phelper{A\geq C}[\Dany]$. 
\begin{align}
    \phelper{A\geq C}[\phi\cdot \Dany]&=\phelper{A\geq C}[\phi\inv\cdot \Dany]\label{eq:involution-1}\\
    &\leq \phelper{B\geq C}[\Dany]\label{eq:sim-apply-opt}\\
    &< \phelper{A\geq C}[\Dany]\label{eq:assumption-apply-leq}\\
    &\leq \phelper{B\geq C}[\phi\cdot \Dany].\label{eq:sim-apply-opt-2}
\end{align}

\Cref{eq:involution-1} holds because $\phi$ is an involution. \Cref{eq:sim-apply-opt} and \cref{eq:sim-apply-opt-2} hold by applying \cref{lem:sim-lin-func-opt} with permutation $\phi$. \Cref{eq:assumption-apply-leq} holds by assumption. Therefore, $\phelper{A\geq C}[\Dany] \leqMost[][\DSetAny] \phelper{B \geq C}[\Dany]$ by \cref{lem:helper-geq-most}.

Suppose $B'\subseteq C$ and $\ND{C}\cap \prn{B\setminus B'}$ is non-empty, and let $\Diid$ be any continuous distribution which distributes reward independently and identically across states. Let $\phi'\in\mdpPermGroup$.
\begin{align}
\phelper{A\geq C}[\phi'\cdot \Diid]&=\phelper{A\geq C}[\Diid]\label{eq:same-iid-dist-prob}\\
&<\phelper{B\geq C}[\phi\cdot \Diid]\label{eq:sim-iid-prob}\\
&=\phelper{A\geq C}[\phi'\cdot \Diid].\label{eq:same-iid-dist-prob-2}
\end{align}

\Cref{eq:same-iid-dist-prob} and \cref{eq:same-iid-dist-prob-2} hold because $\Diid$ distributes reward identically across states, $\forall \phi_x\in\mdpPermGroup: \phi_x\cdot \Diid=\Diid$. By \cref{lem:iid-contains}, $\exists b<c: (b,c)^{\abs{\St}}\subseteq \supp[\Diid]$. Therefore, apply \cref{lem:sim-lin-func-opt} to conclude that \cref{eq:sim-iid-prob} holds. 

Therefore, $\forall \phi'\in\mdpPermGroup:\phelper{A\geq C}[\phi'\cdot \Diid]<\phelper{B\geq C}[\phi'\cdot \Diid]$. In particular, $\phelper{A\geq C}[\Dany] \not\geqMost[][\DSetAny] \phelper{B \geq C}[\Dany]$ by \cref{def:ineq-most-dists}.
\end{proof}

\begin{restatable}[Limit probability inequalities which hold for most distributions]{lem}{limProbMost}\label{lem:lim-prob-most}
Let $I\subseteq \reals$, let $\distSet \subseteq \Delta(\rewardVS)$ be closed under permutation, and let $F_A,F_B,F_C$ be finite sets of vector functions $I\mapsto \rewardVS$. Let $\gamma$ be a limit point of $I$ such that $f_1(\D)\defeq \lim_{\gamma^*\to \gamma}\phelper{F_B(\gamma^*)\geq F_C(\gamma^*)}[\D],f_2(\D)\defeq \lim_{\gamma^*\to \gamma}\phelper{F_A(\gamma^*)\geq F_C(\gamma^*)}[\D]$ are well-defined for all $\D\in\distSet$.

Let $F_Z$ satisfy $\ND{F_C}\subseteq F_Z \subseteq F_C$. Suppose $F_B$ contains a copy of $F_A$ via $\phi$ such that $\phi\cdot \prn{F_Z\setminus \prn{F_B\setminus \phi\cdot F_A}}=F_Z\setminus \prn{F_B\setminus \phi\cdot F_A}$. Then $f_2(\distSet) \leqMost[][\distSet] f_1(\distSet)$.
\end{restatable}
\begin{proof}
Suppose $\D\in\distSet$ is such that $f_2(\D)>f_1(\D)$. 
\begin{align}
    f_2\prn{\phi\cdot \D}&=f_2\prn{\phi\inv\cdot \D}\label{eq:involute-f1}\\
    &\defeq \lim_{\gamma^*\to \gamma}\phelper{F_A(\gamma^*)\geq F_C(\gamma^*)}[\phi\inv\cdot \D]\\
    &\leq \lim_{\gamma^*\to \gamma}\phelper{F_B(\gamma^*)\geq F_C(\gamma^*)}[\D]\label{eq:FB-geq}\\
    &<\lim_{\gamma^*\to \gamma}\phelper{F_A(\gamma^*)\geq F_C(\gamma^*)}[\D]\label{eq:assumed-ineq}\\
    &\leq  \lim_{\gamma^*\to \gamma}\phelper{F_B(\gamma^*)\geq F_C(\gamma^*)}[\phi\cdot \D]\label{eq:FB-geq-2}\\
    &\eqdef f_1\prn{\phi\cdot \D}.
\end{align}
By the assumption that $\distSet$ is closed under permutation and $f_2$ is well-defined for all $\D\in\distSet$, $f_2(\phi\cdot \D)$ is well-defined. \Cref{eq:involute-f1} follows since $\phi=\phi\inv$ because $\phi$ is an involution. For all $\gamma^*\in I$, let $A\defeq F_A(\gamma^*), B\defeq F_B(\gamma^*), C\defeq F_C(\gamma^*), Z\defeq F_Z(\gamma^*)$ (by \cref{def:nd-vec-func}, $\ND{C}\subseteq Z \subseteq C$). Since $\phi\cdot A\subseteq B$ by assumption, and since $\ND{A}\subseteq A$, $B$ also contains a copy of $\ND{A}$ via $\phi$. Furthermore, $\phi\cdot \prn{Z\setminus \prn{B\setminus \phi\cdot A}}=Z\setminus \prn{B\setminus \phi\cdot A}$ (by assumption), and so apply  \cref{lem:sim-lin-func-opt} to conclude that $\phelper{F_A(\gamma^*)\geq F_C(\gamma^*)}[\phi\inv\cdot \D]\leq \phelper{F_B(\gamma^*)\geq F_C(\gamma^*)}[\D]$. Therefore, the limit inequality \cref{eq:FB-geq} holds. \Cref{eq:assumed-ineq} follows because we assumed that $f_1(\D)<f_2(\D)$. \Cref{eq:FB-geq-2} holds by reasoning similar to that given for \cref{eq:FB-geq}.

Therefore, $f_2(\D)>f_1(\D)$ implies that $f_2\prn{\phi\cdot \D}<f_1\prn{\phi\cdot \D}$, and so apply \cref{lem:helper-geq-most} to conclude that $f_2(\D) \leqMost[][\distSet] f_1(\D)$.
\end{proof}

\subsubsection{\texorpdfstring{$\Fnd$}{Non-dominated visit distribution} results} 
  
\transferDiscount
\begin{proof}
Let $R$ be any reward function. Suppose $\gamma^*\in(0,1)$ and construct $R'(s)\defeq \OptVf{s, \gamma}- \gamma^* \max_{a\in\A}\E{s'\sim T(s,a)}{\OptVf{s',\gamma}}$. 

Let $\pi\in\Pi$ be any policy. By the definition of optimal policies, $\pi\in\optPi[R',\gamma^*]$ iff for all $s$:
\begin{align}
    R'(s) + \gamma^* \E{s' \sim T\prn{s,\pi(s)}}{\OptVf[R']{s',\gamma^*}}&=R'(s) + \gamma^* \max_{a\in\A}\E{s' \sim T\prn{s,a}}{\OptVf[R']{s',\gamma^*}}\\
    R'(s) + \gamma^* \E{s' \sim T\prn{s,\pi(s)}}{\OptVf{s',\gamma}}&=R'(s) + \gamma^* \max_{a\in\A}\E{s' \sim T\prn{s,a}}{\OptVf{s',\gamma}}\label{eq:opt-val-fn-equality}\\
    \gamma^*\E{s' \sim T\prn{s,\pi(s)}}{\OptVf{s',\gamma}}&=\gamma^*\max_{a\in\A}\E{s'\sim T(s,a)}{\OptVf{s',\gamma}}\label{eq:subst-r'}\\
    \E{s' \sim T\prn{s,\pi(s)}}{\OptVf{s',\gamma}}&=\max_{a\in\A}\E{s'\sim T(s,a)}{\OptVf{s',\gamma}}.\label{eq:div-gamma*}
\end{align} 

By the Bellman equations, $R'(s)=\OptVf[R']{s, \gamma^*}- \gamma^* \max_{a\in\A}\E{s'\sim T(s,a)}{\OptVf[R']{s',\gamma^*}}$. By the definition of $R'$, $\OptVf[R']{\cdot,\gamma^*}=\OptVf[R]{\cdot, \gamma}$ must be the unique solution to the Bellman equations for $R'$ at $\gamma^*$. Therefore, \cref{eq:opt-val-fn-equality} holds. \Cref{eq:subst-r'} follows by plugging in $R'\defeq \OptVf{s, \gamma}- \gamma^* \max_{a\in\A}\E{s'\sim T(s,a)}{\OptVf{s',\gamma}}$ to \cref{eq:opt-val-fn-equality} and doing algebraic manipulation. \Cref{eq:div-gamma*} follows because $\gamma^*>0$.  

\Cref{eq:div-gamma*} shows that $\pi\in\optPi[R',\gamma^*]$ iff $\forall s: \E{s'\sim T(s,\pi(s))}{\OptVf{s',\gamma}} = \max_{a\in\A} \E{s'\sim T(s,a)}{\OptVf{s',\gamma}}$. That is, $\pi\in\optPi[R',\gamma^*]$ iff $\pi\in\optPi[R,\gamma]$.
\end{proof} 

\begin{restatable}[Evaluating sets of visit distribution functions at $\gamma$]{definition}{evalFDisc}\label{def:eval-f-discount}
For $\gamma \in (0,1)$, define $\F(s,\gamma)\defeq \set{\f(\gamma) \mid \f \in \F(s)}$ and $\Fnd(s,\gamma)\defeq \set{\f(\gamma) \mid \f \in \Fnd(s)}$. If $F\subseteq \F(s)$, then $F(\gamma) \defeq \set{\f(\gamma) \mid \f \in F}$.
\end{restatable}

\begin{restatable}[Non-domination across $\gamma$ values for expectations of visit distributions]{lem}{ndSetFGamma}\label{lem:nd-gamma-F-subset}
Let $\Delta_d \in \Delta\prn{\rewardVS}$ be any state distribution and let $F\defeq\set{\E{s_d\sim\Delta_d}{\fpi{s_d}} \mid \pi \in \Pi}$. $\f \in \ND{F}$ iff $\forall \gamma^*\in(0,1): \f(\gamma^*)\in \ND{F(\gamma^*)}$.
\end{restatable}
\begin{proof}
Let $\fpi{}\in\ND{F}$ be strictly optimal for reward function $R$ at discount rate $\gamma\in(0,1)$:
\begin{align}
    \fpi{}(\gamma)^\top \rf > \max_{\fpi[\pi']{}\in F\setminus \set{\fpi{}}} \fpi[\pi']{}(\gamma)^\top \rf.
\end{align}

Let $\gamma^*\in(0,1)$. By \cref{transferDiscount}, we can produce $R'$ such that $\optPi[R',\gamma^*]=\optPi$. Since the optimal policy sets are equal, \cref{lem:opt-pol-visit-iff} implies that
\begin{align}
    \fpi{}(\gamma^*)^\top \rf' > \max_{\fpi[\pi']{}\in F\setminus \set{\fpi{}}} \fpi[\pi']{}(\gamma^*)^\top \rf'.
\end{align}

Therefore, $\fpi{}(\gamma^*)\in\ND{F(\gamma^*)}$. 

The reverse direction follows by the definition of $\ND{F}$.
\end{proof}

\begin{restatable}[$\forall \gamma \in (0,1): \dbf\in \Fnd(s,\gamma)$ iff $\dbf\in\ND{\F(s,\gamma)}$]{lem}{invariantFndDiscount}\label{lem:nd-relation} 
\end{restatable} 
\begin{proof}
By \cref{def:eval-f-discount}, $\Fnd(s,\gamma)\defeq \set{\f(\gamma) \mid \f \in \ND{\F(s)}}$. By applying \cref{lem:nd-gamma-F-subset} with $\Delta_d\defeq \unitvec$, $\f \in \ND{\F(s)}$ iff $\forall \gamma\in(0,1): \f(\gamma)\in \ND{\F(s,\gamma)}$. 
\end{proof}

\optVfFndRestrict
\begin{proof}
$\ND{\F(s,\gamma)}=\Fnd(s,\gamma)$ by \cref{lem:nd-relation}, so apply \cref{cor:nd-func-indif} with $X \defeq \F(s,\gamma)$. 
\end{proof} 

\subsection{Some actions have greater probability of being optimal}
\begin{restatable}[Optimal policy shift bound]{lem}{optPiBound}\label{lem:opt-pol-shift-bound}
For fixed $R$, $\optPi$ can take on at most $(2\abs{\St}+1)\sum_s \binom{\abs{\F(s)}}{2}$ distinct values over $\gamma \in (0,1)$.
\end{restatable} 
\begin{proof}
By \cref{lem:opt-pol-visit-iff}, $\optPi$ changes value iff there is a change in optimality status for some visit distribution function at some state. \citet{lippman1968set} showed that two visit distribution functions can trade off optimality status at most $2\abs{\St}+1$ times. At each state $s$, there are $\binom{\abs{\F(s)}}{2}$ such pairs. 
\end{proof} 

\begin{restatable}[Optimality probability's limits exist]{prop}{muConverge} \label{prop:opt-prob-converge} Let $F\subseteq \F(s)$. $\optprob[\Dany]{F,0}=\lim_{\gamma\to 0} \optprob[\Dany]{F,\gamma}$ and $\optprob[\Dany]{F,1}=\lim_{\gamma\to 1} \optprob[\Dany]{F,\gamma}$. 
\end{restatable}
\begin{proof}
First consider the limit as $\gamma \to 1$. Let $\Dany$ have probability measure $F_\text{any}$, and define $\delta(\gamma)\defeq F_\text{any}\prn{\set{R\in\rewardSpace \mid \exists \gamma^* \in [\gamma,1): \optPi[R,\gamma^*]\neq \optPi[R,1]}}$. Since $F_\text{any}$ is a probability measure, $\delta(\gamma)$ is bounded $[0,1]$, and $\delta(\gamma)$ is monotone decreasing. Therefore, $\lim_{\gamma \to 1} \delta(\gamma)$ exists. 

If $\lim_{\gamma \to 1} \delta(\gamma)>0$, then there exist reward functions whose optimal policy sets $\optPi$ never converge (in the discrete topology on sets) to $\optPi[R,1]$, contradicting \cref{lem:opt-pol-shift-bound}. So $\lim_{\gamma \to 1} \delta(\gamma)=0$.

By the definition of optimality probability (\cref{def:prob-opt}) and of $\delta(\gamma)$,  $|\optprob[\Dany]{F,\gamma}-\optprob[\Dany]{F,1}|\leq \delta(\gamma)$. Since $\lim_{\gamma\to 1} \delta(\gamma)=0$, $\lim_{\gamma\to 1}\optprob[\Dany]{F,\gamma}=\optprob[\Dany]{F,1}$.

A similar proof shows that $\lim_{\gamma\to 0} \optprob[\Dany]{F,\gamma}=\optprob[\Dany]{F,0}$.
\end{proof} 

\begin{restatable}[Optimality probability identity]{lem}{optProbID}\label{lem:opt-prob-id}
Let $\gamma\in(0,1)$ and let $F\subseteq \F(s)$. 
\begin{equation}
\optprob[\Dany]{F,\gamma}=\phelper{F(\gamma)\geq \F(s,\gamma)}=\phelper{F(\gamma)\geq \Fnd(s,\gamma)}.
\end{equation}
\end{restatable}
\begin{proof}
Let $\gamma\in(0,1)$.
\begin{align}
    \optprob[\Dany]{F,\gamma}&\defeq \prob[R\sim\Dany]{\exists \f^\pi \in F: \pi\in\optPi}\\
    &=\E{\rf \sim\Dany}{ \indic{\max_{\f\in F} \f(\gamma)^\top \rf = \max_{\f'\in\F(s)} \f'(\gamma)^\top \rf}}\label{eq:F-probability}\\
    &=\E{\rf \sim\Dany}{ \indic{\max_{\f\in F} \f(\gamma)^\top \rf = \max_{\f'\in\Fnd(s)} \f'(\gamma)^\top \rf}}\label{eq:Fnd-probability}\\
    &\eqdef \phelper{F(\gamma)\geq \Fnd(s,\gamma)}.
\end{align}
\Cref{eq:F-probability} follows because \cref{lem:opt-pol-visit-iff} shows that $\pi$ is optimal iff it induces an optimal visit distribution  $\f$ at every state. \Cref{eq:Fnd-probability} follows because $\forall \rf \in \rewardVS: \max_{\f'\in\F(s)} \f'(\gamma)^\top \rf= \max_{\f'\in\Fnd(s)} \f'(\gamma)^\top \rf$ by \cref{cor:opt-vf-restrict-fnd}.
\end{proof}

\subsection{Basic properties of \texorpdfstring{$\pwrNoDist$}{power}}

\begin{restatable}[$\pwrNoDist$ identities]{lem}{lemPowEQ}\label{lem:power-id}
Let $\gamma \in (0,1)$.
\begin{align}
    \pwr[s,\gamma]&=\E{\rf\sim\Dbd}{\max_{\f\in \Fnd(s)} \frac{1-\gamma}{\gamma}\prn{\f(\gamma)-\unitvec}^\top \rf}\label{eq:pwr-def-f}\\
    &= \dfrac{1-\gamma}{\gamma}\E{\rf\sim\Dbd}{\OptVf{s,\gamma}-R(s)}\\
    &=  \dfrac{1-\gamma}{\gamma}\prn{\vavg-\E{R\sim \Dbd}{R(s)}} \\
    &=\E{R\sim\Dbd}{\max_{\pi\in\Pi} \E{s'\sim T\prn{s,\pi(s)}}{\prn{1-\gamma}V^\pi_R\prn{s',\gamma}}}.\label{eq:pwr-def-avg-discounted}
\end{align}
\end{restatable}
\begin{proof}
\begin{align}
    \pwrNoDist_{\Dbd}(s,\gamma)&\defeq\E{\rf\sim\Dbd}{\max_{\f\in \F(s)} \frac{1-\gamma}{\gamma}\prn{\f(\gamma)-\unitvec}^\top \rf}\\
    &= \E{\rf\sim\Dbd}{\max_{\f\in \Fnd(s)} \frac{1-\gamma}{\gamma}\prn{\f(\gamma)-\unitvec}^\top \rf}\label{eq:pwr-fnd-restrict}\\
    &=\E{\rf\sim\Dbd}{\max_{\f\in \F(s)} \frac{1-\gamma}{\gamma}\prn{\f(\gamma)-\unitvec}^\top \rf}\\
    &= \dfrac{1-\gamma}{\gamma}\E{\rf\sim\Dbd}{\OptVf{s,\gamma}-R(s)}\label{eq:pwr-vf-convert}\\
    &=\dfrac{1-\gamma}{\gamma}\prn{\vavg-\E{R\sim \Dbd}{R(s)}}\label{eq:vavg-def-pwr-def}\\
    &= \E{\rf\sim\Dbd}{\max_{\pi\in\Pi} \E{s'\sim T\prn{s,\pi(s)}}{\prn{1-\gamma}\fpi[\pi]{s'}(\gamma)^\top\rf}}\label{eq:f-expand-pwr-def}\\
    &= \E{R\sim\Dbd}{\max_{\pi\in\Pi} \E{s'\sim T\prn{s,\pi(s)}}{\prn{1-\gamma}V^\pi_R\prn{s',\gamma}}}.\label{eq:pwr-avg-normalized}
\end{align}

\Cref{eq:pwr-fnd-restrict} follows from \cref{cor:opt-vf-restrict-fnd}. \Cref{eq:pwr-vf-convert} follows from the dual formulation of optimal value functions. \Cref{eq:vavg-def-pwr-def} holds by the definition of $\vavg$ (\cref{def:vavg}). \Cref{eq:f-expand-pwr-def} holds because $\fpi{s}(\gamma) = \unitvec +\gamma \E{s'\sim T\prn{s,\pi(s)}}{\fpi[\pi]{s'}(\gamma)}$ by the definition of a visit distribution function (\cref{def:visit}).
\end{proof}

\begin{restatable}[Discount-normalized value function]{definition}{discValue}\label{def:disc-value}
Let $\pi$ be a policy, $R$ a reward function, and $s$ a state. For $\gamma \in [0,1]$, $\VfNorm[\pi]{s,\gamma}\defeq \lim_{\gamma^*\to \gamma}(1-\gamma^*) V^\pi_R(s,\gamma^*)$.
\end{restatable}

\begin{restatable}[Normalized value functions have uniformly bounded derivative]{lem}{normValueLip}\label{lem:norm-value-lip}
There exists $K\geq0$ such that for all reward functions $\rf\in \rewardVS$, $\sup_{\substack{s\in\St, \pi\in\Pi,\gamma\in [0,1]}}\abs{\frac{d}{d\gamma} \VfNorm[\pi]{s,\gamma}} \leq K\lone{\rf}$.
\end{restatable}
\begin{proof}
Let $\pi$ be any policy, $s$ a state, and $R$ a reward function. Since $\VfNorm[\pi]{s,\gamma}=\lim_{\gamma^*\to\gamma}(1-\gamma^*)\fpi{s}(\gamma^*)^\top \rf$, $\frac{d}{d\gamma}\VfNorm[\pi]{s,\gamma}$ is controlled by the behavior of $\lim_{\gamma^*\to\gamma}(1-\gamma^*)\fpi{s}(\gamma^*)$. We show that this function's gradient is bounded in infinity norm.

By \cref{f-rat}, $\fpi{s}(\gamma)$ is a multivariate rational function on $\gamma$. Therefore, for any state $s'$,  $\fpi{s}(\gamma)^\top\unitvec[s']=\frac{P(\gamma)}{Q(\gamma)}$ in reduced form. By \cref{prop:visit-dist-prop}, $0\leq \fpi{s}(\gamma)^\top\unitvec[s']\leq \frac{1}{1-\gamma}$. Thus, $Q$ may only have a root of multiplicity 1 at $\gamma=1$, and $Q(\gamma)\neq 0$ for $\gamma\in[0,1)$. Let $f_{s'}(\gamma)\defeq (1-\gamma)\fpi{s}(\gamma)^\top\unitvec[s']$.

If $Q(1)\neq 0$, then the derivative $f_{s'}'(\gamma)$ is bounded on $\gamma\in[0,1)$ because the polynomial $(1-\gamma)P(\gamma)$ cannot diverge on a bounded domain. 

If $Q(1)=0$, then factor out the root as $Q(\gamma)=(1-\gamma)Q^*(\gamma)$. 
\begin{align}
    f_{s'}'(\gamma)&=\frac{d}{d\gamma}\prn{\frac{(1-\gamma)P(\gamma)}{Q(\gamma)}}\\
    &=\frac{d}{d\gamma}\prn{\frac{P(\gamma)}{Q^*(\gamma)}}\\
    &=\frac{P'(\gamma)Q^*(\gamma)-(Q^*)'(\gamma)P(\gamma)}{(Q^*(\gamma))^2}.
\end{align}

Since $Q^*(\gamma)$ is a polynomial with no roots on $\gamma\in[0,1]$, $f_{s'}'(\gamma)$ is bounded on $\gamma\in[0,1)$.

Therefore, whether or not $Q(\gamma)$ has a root at $\gamma=1$, $f_{s'}'(\gamma)$ is bounded on $\gamma\in[0,1)$. Furthermore, $\sup_{\gamma\in[0,1)}\linfty{\nabla(1-\gamma)\fpi{s}(\gamma)}=\sup_{\gamma\in[0,1)}\max_{s'\in\St} \abs{f_{s'}'(\gamma)}$ is finite since there are only finitely many states.

There are finitely many $\pi\in\Pi$, and finitely many states $s$, and so there exists some $K'$ such that $\sup_{\substack{s\in\St,\\ \pi\in\Pi,\gamma\in [0,1)}} \linfty{\nabla(1-\gamma)\fpi{s}(\gamma)}\leq K'$. Then $\lone{\nabla(1-\gamma)\fpi{s}(\gamma)}\leq \abs{\St}K'\eqdef K$. 
\begin{align}
    \sup_{\substack{s\in\St,\\ \pi\in\Pi,\gamma\in [0,1)}} \abs{\frac{d}{d\gamma}V^\pi_{R,\text{norm}}\prn{s,\gamma}}\defeq\,&\sup_{\substack{s\in\St,\\ \pi\in\Pi,\gamma\in [0,1)}} \abs{\frac{d}{d\gamma}\lim_{\gamma^*\to\gamma}(1-\gamma^*)V^\pi_{R}\prn{s,\gamma^*}}\\
    =\,&\sup_{\substack{s\in\St,\\ \pi\in\Pi,\gamma\in [0,1)}} \abs{\frac{d}{d\gamma}(1-\gamma)V^\pi_{R}\prn{s,\gamma}}\label{eq:lim-disappear}\\
    =\,&\sup_{\substack{s\in\St,\\ \pi\in\Pi,\gamma\in [0,1)}} \abs{\nabla (1-\gamma)\fpi{s}(\gamma)^\top\rf}\label{eq:on-pol-rat}\\
    \leq\, & \sup_{\substack{s\in\St,\\ \pi\in\Pi,\gamma\in [0,1)}} \lone{\nabla (1-\gamma)\fpi{s}(\gamma)}\lone{\rf}\label{eq:cs}\\
    \leq\, & K\lone{\rf}.\label{eq:bounded}
\end{align}

\Cref{eq:lim-disappear} holds because $\Vf[\pi]{s,\gamma}$ is continuous on $\gamma\in[0,1)$ by \cref{smoothOnPol}. \Cref{eq:cs} holds by the Cauchy-Schwarz inequality.

Since $\abs{\frac{d}{d\gamma}V^\pi_{R,\text{norm}}\prn{s,\gamma}}$ is bounded for all $\gamma \in [0,1)$, \cref{eq:bounded} also holds for $\gamma\to 1$. 
\end{proof}

\ContPower*
\begin{proof}
Let $b,c$ be such that $\supp[\Dbd]\subseteq [b,c]^{\abs{\St}}$. For any $\rf\in\supp[\Dbd]$ and $\pi\in\Pi$, $\VfNorm[\pi]{s,\gamma}$ has Lipschitz constant $K\lone{\rf}\leq K \abs{\St}\linfty{\rf}\leq K\abs{\St}\max(\abs{c},\abs{b})$ on $\gamma\in(0,1)$ by \cref{lem:norm-value-lip}.

For $\gamma\in(0,1)$, $\pwr[s,\gamma] = \E{R\sim\Dbd}{\max_{\pi\in\Pi} \E{s'\sim T\prn{s,\pi(s)}}{(1-\gamma)V^\pi_R\prn{s',\gamma}}}$ by \cref{eq:pwr-avg-normalized}. The expectation of the maximum of a set of functions which share a Lipschitz constant, also shares the Lipschitz constant. This shows that $\pwr[s,\gamma]$ is Lipschitz continuous on $\gamma\in(0,1)$. Thus, its limits are well-defined as $\gamma\to 0$ and $\gamma\to 1$. So it is Lipschitz continuous on the closed unit interval.
\end{proof}

\maxPwrGeneral*
\begin{proof}
Let $\gamma \in (0,1)$.
\begin{align}
    \pwr[s,\gamma][\Dbd]&= \E{R\sim \Dbd}{\max_{\pi\in\Pi} \E{s'\sim T(s,\pi(s))}{(1-\gamma)\OptVf{s',\gamma}}}\label{eq:pwr-id}\\
    &\leq \E{R\sim \Dbd}{\max_{\pi\in\Pi} \E{s'\sim T(s,\pi(s))}{(1-\gamma)\geom[\max_{s''\in\St}R(s'')]}}\label{eq:max-vfn}\\
    &= \E{R\sim \Dbd}{\max_{s''\in\St}R(s'')}.
\end{align}

\Cref{eq:pwr-id} follows from \cref{lem:power-id}. \Cref{eq:max-vfn} follows because $\OptVf{s',\gamma}\leq \geom[\max_{s''\in\St}R(s'')]$, as no policy can do better than achieving maximal reward at each time step. Taking limits, the inequality holds for all $\gamma\in[0,1]$.

Suppose that $s$ can deterministically reach all states in one step and all states are {\stateEnd}s. Then \cref{eq:max-vfn} is an equality for all $\gamma\in(0,1)$, since for each $R$, the agent can select an action which deterministically transitions to a state with maximal reward. Thus the equality holds for all $\gamma\in[0,1]$.
\end{proof}

\begin{restatable}[Lower bound on current $\pwrNoDist$ based on future $\pwrNoDist$]{lem}{FuturePower}\label{lem:future-power}
\begin{align}
\pwr[s,\gamma]\geq (1-\gamma)\min_a \E{\substack{s'\sim T(s,a),\\R\sim\Dbd}}{R(s')} + \gamma\max_a \E{s'\sim T(s,a)}{\pwr[s',\gamma]}.
\end{align}
\end{restatable}
\begin{proof} 
Let $\gamma \in (0,1)$ and let $a^*\in \argmax_a \E{s'\sim T\prn{s,a}}{\pwr[s',\gamma]}$.
\begin{align}
    &\pwr[s,\gamma]\\
    =\,&(1-\gamma)\E{R\sim\Dbd}{\max_a \E{s'\sim T\prn{s,a}}{\OptVf{s',\gamma}}}\label{eq:pwr-id-app}\\
    \geq\,& (1-\gamma) \max_a \E{s'\sim T\prn{s,a}}{\E{R\sim\Dbd}{\OptVf{s',\gamma}}}\label{eq:max-expect-ineq}\\
    =\,&(1-\gamma)\max_a \E{s'\sim T\prn{s,a}}{\vavg[s',\gamma]}\\
    =\,&(1-\gamma) \max_a \E{s'\sim T\prn{s,a}}{\E{R\sim \Dbd}{R(s')}+\frac{\gamma}{1-\gamma}\pwr[s',\gamma]}\label{eq:vavg-identity}\\
    \geq\,&(1-\gamma) \E{s'\sim T\prn{s,a^*}}{\E{R\sim \Dbd}{R(s')}+\frac{\gamma}{1-\gamma}\pwr[s',\gamma]}\\
    \geq\,&(1-\gamma)\min_a \E{\substack{s'\sim T(s,a),\\R\sim\Dbd}}{R(s')} + \gamma \E{s'\sim T\prn{s,a^*}}{\pwr[s',\gamma]}.\label{eq:min-reward-max-power}
\end{align}

\Cref{eq:pwr-id-app} holds by \cref{lem:power-id}. \Cref{eq:max-expect-ineq} follows because $\E{x\sim X}{\max_a f(a,x)}\geq \max_a \E{x\sim X}{f(a,x)}$ by Jensen's inequality, and \cref{eq:vavg-identity} follows by \cref{lem:power-id}. 

The inequality also holds when we take the limits  $\gamma \to 0$ or $\gamma \to 1$.
\end{proof}

\smooth*
\begin{proof}
Suppose $\gamma\in[0,1]$. First consider the case where $\pwr[s,\gamma]\geq \pwr[s',\gamma]$.
\begin{align}
\pwr[s',\gamma]&\geq (1-\gamma)\min_a \E{\substack{s_x\sim T(s',a),\\R\sim\Dbd}}{R(s_x)} + \gamma\max_a \E{s_x\sim T(s',a)}{\pwr[s_x,\gamma]}\label{eq:lb-pwr-next}\\
&\geq (1-\gamma)b + \gamma \pwr[s,\gamma].\label{eq:lb-pwr-next-min}
\end{align}

\Cref{eq:lb-pwr-next} follows by \cref{lem:future-power}. \Cref{eq:lb-pwr-next-min} follows because reward is lower-bounded by $b$ and because  $s'$ can reach $s$ in one step with probability 1. 
\begin{align}
    \abs{\pwr[s,\gamma]-\pwr[s',\gamma]}&=\pwr[s,\gamma]-\pwr[s',\gamma]\label{eq:no-abs}\\
    &\leq \pwr[s,\gamma]-\prn{(1-\gamma)b + \gamma \pwr[s,\gamma]}\label{eq:lb-pwr-abs-diff}\\
    &=(1-\gamma)\prn{\pwr[s,\gamma] - b}\\
    &\leq(1-\gamma)\prn{\E{R\sim \Dbd}{\max_{s''\in\St}R(s'')} - b}\label{eq:ub-pwr-s}\\
    &\leq (1-\gamma)(c-b).\label{eq:ub-pwr-abs-diff}
\end{align}

\Cref{eq:no-abs} follows because $\pwr[s,\gamma]\geq \pwr[s',\gamma]$. \Cref{eq:lb-pwr-abs-diff} follows by \cref{eq:lb-pwr-next-min}. \Cref{eq:ub-pwr-s} follows by \cref{lem:max-power-general}. \Cref{eq:ub-pwr-abs-diff} follows because reward under $\Dbd$ is upper-bounded by $c$.

The case where $\pwr[s,\gamma]\leq \pwr[s',\gamma]$ is similar, leveraging the fact that $s$ can also reach $s'$ in one step with probability 1.
\end{proof}

\subsection{Seeking \texorpdfstring{$\pwrNoDist$}{power} is often more probable under optimality}

\subsubsection{Keeping options open tends to be \texorpdfstring{$\pwrNoDist$}{POWER}-seeking and tends to be optimal}

\begin{restatable}[Normalized visit distribution function]{definition}{normVisitFn}\label{def:norm-visit-fn}
Let $\f:[0,1)\to \rewardVS$ be a vector function. For $\gamma\in[0,1]$, $\NormF{\f,\gamma}\defeq \lim_{\gamma^*\to\gamma} (1-\gamma^*)\f(\gamma^*)$ (this limit need not exist for arbitrary $\f$). If $F$ is a set of such $\f$, then $\NormF{F,\gamma}\defeq \set{\NormF{\f,\gamma}\mid \f \in F}$.
\end{restatable}

\begin{remark}
$\RSD=\NormF{\F(s),1}$.
\end{remark}

\begin{restatable}[Normalized visit distribution functions are continuous]{lem}{contNormFMix}\label{lem:cont-norm-f}
Let $\Delta_s\in \Delta(\St)$ be a state probability distribution, let $\pi\in\Pi$, and let $\f^*\defeq \E{s\sim \Delta_s}{\fpi{s}}$. $\NormF{\f^*, \gamma}$ is continuous on $\gamma\in [0,1]$.
\end{restatable}
\begin{proof}
\begin{align}
    \NormF{\f^*,\gamma}&\defeq \lim_{\gamma^*\to\gamma} (1-\gamma^*)\E{s\sim \Delta_s}{\fpi{s}(\gamma^*)}\\
    &=\E{s\sim \Delta_s}{\lim_{\gamma^*\to\gamma} (1-\gamma^*)\fpi{s}(\gamma^*)}\label{eq:expect-lim-swap}\\
    &\eqdef \E{s\sim \Delta_s}{\NormF{\fpi{s},\gamma}}.\label{eq:expect-cont}
\end{align}
\Cref{eq:expect-lim-swap} follows because the expectation is over a finite set. Each $\fpi{s}\in\F(s)$ is continuous on $\gamma\in[0,1)$ by \cref{f-rat}, and $\lim_{\gamma^*\to 1} (1-\gamma^*) \fpi{s}(\gamma^*)$ exists because {\rsd}s are well-defined \citep{puterman_markov_2014}. Therefore, each $\NormF{\fpi{s},\gamma}$ is continuous on $\gamma\in[0,1]$. Lastly, \cref{eq:expect-cont}'s expectation over finitely many continuous functions is itself continuous. 
\end{proof}

\begin{restatable}[Non-domination of normalized visit distribution functions]{lem}{ndNormalVisit}\label{lem:nd-norm-visit}
Let $\Delta_s\in \Delta(\St)$ be a state probability distribution and let $F\defeq \set{\E{s\sim \Delta_s}{\fpi{s}}\mid\pi\in\Pi}$. For all $\gamma\in[0,1]$, $\ND{\NormF{F,\gamma}}\subseteq \NormF{\ND{F},\gamma}$, with equality when $\gamma\in(0,1)$.
\end{restatable}
\begin{proof}
Suppose $\gamma\in (0,1)$.
\begin{align}
    \ND{\NormF{F,\gamma}}&=\ND{(1-\gamma)F(\gamma)}\label{eq:apply-cont-norm-1}\\
    &= (1-\gamma)\ND{F(\gamma)}\label{eq:nd-rescale-invariant}\\
    &= (1-\gamma)\prn{\ND{F}(\gamma)}\label{eq:nd-invariant-01}\\
    &= \NormF{\ND{F},\gamma}.\label{eq:apply-cont-norm-2}
\end{align}

\Cref{eq:apply-cont-norm-1} and \cref{eq:apply-cont-norm-2} follow by the continuity of $\NormF{\f,\gamma}$ (\cref{lem:cont-norm-f}). \Cref{eq:nd-rescale-invariant} follows by \cref{lem:pos-aff-nd-invar} \cref{item:invar-vectors}. \Cref{eq:nd-invariant-01} follows by \cref{lem:nd-gamma-F-subset}.

Let $\gamma=1$. Let $\dbf \in \ND{\NormF{F,1}}$ be strictly optimal for $\rf^*\in\rewardVS$. Then let $F_\dbf \subseteq F$ be the subset of $\f\in F$ such that $\NormF{\f,1}=\dbf$. 
\begin{align}
    \max_{\f\in F_\dbf} \NormF{\f,1}^\top \rf^* & > \max_{\f'\in F\setminus F_\dbf} \NormF{\f',1}^\top \rf^*.\label{eq:norm-f-ineq}
\end{align}

Since $\NormF{\f,1}$ is continuous at $\gamma=1$ (\cref{lem:cont-norm-f}), $\x^\top \rf^*$ is continuous on $\x\in\rewardVS$, and $F$ is finite, \cref{eq:norm-f-ineq} holds for some $\gamma^*\in(0,1)$ sufficiently close to $\gamma=1$. By \cref{lem:all-rf-max-nd}, at least one $\f\in F_\dbf$ is an element of $\ND{F(\gamma^*)}$. Then by \cref{lem:nd-gamma-F-subset}, $\f\in\ND{F}$. We conclude that $\ND{\NormF{F,1}}\subseteq \NormF{\ND{F},1}$. 

The case for $\gamma=0$ proceeds similarly.
\end{proof}

\begin{restatable}[$\pwrNoDist$ limit identity]{lem}{pwrLimit}\label{lem:pwr-limit}
Let $\gamma\in[0,1]$.
\begin{align}
    \pwr[s,\gamma]&=\E{\rf\sim\Dbd}{\max_{\f\in \Fnd(s)} \lim_{\gamma^*\to\gamma}\frac{1-\gamma^*}{\gamma^*}\prn{\f(\gamma^*)-\unitvec}^\top \rf}.
\end{align}
\end{restatable}
\begin{proof}
Let $\gamma\in[0,1]$.
\begin{align}
    \pwr[s,\gamma]&= \lim_{\gamma^*\to\gamma}\pwr[s,\gamma^*]\label{eq:pwr-cont-lim}\\
    &=\lim_{\gamma^*\to\gamma} \E{\rf\sim\Dbd}{\max_{\f\in \Fnd(s)} \frac{1-\gamma^*}{\gamma^*}\prn{\f(\gamma^*)-\unitvec}^\top \rf}\label{eq:pwr-lim-def}\\
    &=\E{\rf\sim\Dbd}{\lim_{\gamma^*\to\gamma}\max_{\f\in \Fnd(s)} \frac{1-\gamma^*}{\gamma^*}\prn{\f(\gamma^*)-\unitvec}^\top \rf}\label{eq:pwr-dom-conv}\\
    &=\E{\rf\sim\Dbd}{\max_{\f\in \Fnd(s)} \lim_{\gamma^*\to\gamma}\frac{1-\gamma^*}{\gamma^*}\prn{\f(\gamma^*)-\unitvec}^\top \rf}.\label{eq:max-cont-lim}
\end{align}
\Cref{eq:pwr-cont-lim} follows because $\pwr[s,\gamma]$ is continuous on $\gamma\in[0,1]$ by \cref{thm:cont-power}. \Cref{eq:pwr-lim-def} follows by \cref{lem:power-id}. 

For $\gamma^*\in(0,1)$, let $f_{\gamma^*}(\rf)\defeq \max_{\f\in \Fnd(s)} \frac{1-\gamma^*}{\gamma^*}\prn{\f(\gamma^*)-\unitvec}^\top \rf$. For any sequence $\gamma_n\to\gamma$, $\prn{f_{\gamma_n}}_{n=1}^\infty$ is a sequence of functions which are piecewise linear on $\rf \in \rewardVS$, which means they are continuous and therefore measurable. Since \cref{f-rat} shows that each $\f\in\Fnd(s)$ is multivariate rational on $\gamma^*$ (and therefore continuous on $\gamma^*$), $\set{f_{\gamma_n}}_{n=1}^\infty$ converges pointwise to limit function $f_\gamma$. Furthermore,  $\abs{\OptVf{s,\gamma_n}-R(s)}\leq \frac{\gamma}{1-\gamma_n}\linfty{R}$, and so $\abs{f_{\gamma_n}(\rf)}=\abs{\frac{1-\gamma_n}{\gamma_n}(\OptVf{s,\gamma_n}-R(s))}\leq g(\rf)\leq \linfty{\rf}\eqdef g(\rf)$, which is measurable. Therefore, apply Lebesgue's dominated convergence theorem to conclude that \cref{eq:pwr-dom-conv} holds. \Cref{eq:max-cont-lim} holds because $\max$ is a continuous function.
\end{proof}

\begin{restatable}[Lemma for $\pwrNoDist$ superiority]{lem}{morePowerPrefixmoreOptions}\label{lem:more-power-prefix}
Let $\Delta_1,\Delta_2\in \Delta\prn{\St}$ be state probability distributions. For $i=1,2$, let $F_{\Delta_i}\defeq \set{\gamma\inv\E{s_i\sim \Delta_i}{\fpi{s_i}-\unitvec[s_i]}\mid \pi\in\Pi}$. Suppose $F_{\Delta_2}$ contains a copy of $\ND{F_{\Delta_1}}$ via $\phi$. Then $\forall \gamma\in[0,1]:\E{s_1\sim \Delta_1}{\pwr[s_1,\gamma]}\leqMost[][\DSetBd] \E{s_2\sim \Delta_2}{\pwr[s_2,\gamma]}$. 

If $\ND{F_{\Delta_2}}\setminus \phi\cdot \ND{F_{\Delta_1}}$ is non-empty, then for all $\gamma\in(0,1)$, the inequality is strict for all $\Diid\in\DSetBCiid$ and $\E{s_1\sim \Delta_1}{\pwr[s_1,\gamma][\Dbd]}\not\geqMost[][\DSetBd] \E{s_2\sim \Delta_2}{\pwr[s_2,\gamma][\Dbd]}$.

These results also hold when replacing $F_{\Delta_i}$ with $F_{\Delta_i}^*\defeq \set{\E{s_i\sim \Delta_i}{\fpi{s_i}}\mid \pi\in\Pi}$ for $i=1,2$.
\end{restatable}
\begin{proof} 
\begin{align}
    \phi\cdot \ND{\NormF{F_{\Delta_1},\gamma}}&\subseteq \phi\cdot \NormF{\ND{F_{\Delta_1}},\gamma}\label{eq:subset-nd-norm}\\
    &\defeq \set{\permute\lim_{\gamma^*\to\gamma}(1-\gamma^*)\f(\gamma^*) \mid \f \in \ND{F_{\Delta_1}}}\\
    &=\set{\lim_{\gamma^*\to\gamma}(1-\gamma^*)\permute\f(\gamma^*) \mid \f \in \ND{F_{\Delta_1}}}\label{eq:cont-permute-app}\\
    &=\set{\lim_{\gamma^*\to\gamma}(1-\gamma^*)\f(\gamma^*) \mid \f \in F_{\text{sub}}'}\\
    &\subseteq \set{\lim_{\gamma^*\to\gamma}(1-\gamma^*)\f(\gamma^*) \mid \f\in F_{\Delta_2}}\label{eq:subset-f}\\
    &\eqdef \NormF{F_{\Delta_2},\gamma}.\label{eq:subset-f-final}
\end{align}
\Cref{eq:subset-nd-norm} follows by \cref{lem:nd-norm-visit}. \Cref{eq:cont-permute-app} follows because $\permute$ is a continuous linear operator. \Cref{eq:subset-f} follows by assumption.
\begin{align}
    \E{s_1\sim \Delta_1}{\pwr[s_1,\gamma]}&\defeq\E{\substack{s_1\sim \Delta_1,\\\rf\sim\Dbd}}{\max_{\pi\in\Pi} \lim_{\gamma^*\to\gamma}\frac{1-\gamma^*}{\gamma^*}\prn{\fpi{s_1}(\gamma^*)-\unitvec[s_1]}^\top \rf}\label{eq:expect-pwr-1}\\
    &=\E{\rf\sim\Dbd}{\max_{\pi\in\Pi} \lim_{\gamma^*\to\gamma}\frac{1-\gamma^*}{\gamma^*}\E{s_1\sim \Delta_1}{\fpi{s_1}(\gamma^*)-\unitvec[s_1]}^\top \rf}\label{eq:bring-inside-expect}\\
    &=\E{\rf\sim\Dbd}{\max_{\dbf\in \NormF{F_{\Delta_1},\gamma}} \dbf^\top \rf}\\
    &=\E{\rf\sim\Dbd}{\max_{\dbf\in \ND{\NormF{F_{\Delta_1},\gamma}}} \dbf^\top \rf}\label{eq:restrict-nd-fdelta}\\
    &\leqMost[][\DSetBd] \E{\rf\sim\Dbd}{\max_{\dbf\in \NormF{F_{\Delta_2},\gamma}} \dbf^\top \rf}\label{eq:leq-most-apply-transient}\\
    &=\E{\rf\sim\Dbd}{\max_{\pi\in\Pi} \lim_{\gamma^*\to\gamma}\frac{1-\gamma^*}{\gamma^*}\E{s_2\sim \Delta_2}{\fpi{s_2}(\gamma^*)-\unitvec[s_2]}^\top \rf}\\
    &=\E{\substack{s_2\sim \Delta_2,\\\rf\sim\Dbd}}{\max_{\pi\in\Pi} \lim_{\gamma^*\to\gamma}\frac{1-\gamma^*}{\gamma^*}\prn{\fpi{s_2}(\gamma^*)-\unitvec[s_2]}^\top \rf}\label{eq:bring-inside-expect-2}\\
    &\eqdef \E{s_2\sim \Delta_2}{\pwr[s_2,\gamma]}.\label{eq:expect-pwr-2}
\end{align}

\Cref{eq:expect-pwr-1} and \cref{eq:expect-pwr-2} follow by \cref{lem:pwr-limit}. \Cref{eq:bring-inside-expect} and \cref{eq:bring-inside-expect-2} follow because each $R$ has a stationary deterministic optimal policy $\pi\in \optPi\subseteq \Pi$ which simultaneously achieves  optimal value at all states. \Cref{eq:restrict-nd-fdelta} follows by \cref{cor:nd-func-indif}.

Apply \cref{lem:expect-superior} with $A\defeq \NormF{F_{\Delta_1},\gamma}, B \defeq \NormF{F_{\Delta_2},\gamma}$, $g$ the identity function, and involution $\phi$ (satisfying $\phi\cdot \ND{A}\subseteq B$ by \cref{eq:subset-f-final}) in order to conclude that \cref{eq:leq-most-apply-transient} holds. 

Suppose that $\ND{F_{\Delta_2}}\setminus \phi\cdot \ND{F_{\Delta_1}}$ is non-empty; let $F_\text{sub}'\defeq \phi\cdot \ND{F_{\Delta_1}}$. \Cref{lem:nd-gamma-F-subset} shows that for all $\gamma\in(0,1)$, $\ND{F_{\Delta_2}(\gamma)}\setminus F_\text{sub}'(\gamma)$ is non-empty. \Cref{lem:pos-aff-nd-invar} \cref{item:invar-vectors} then implies that $\ND{B}\setminus \phi\cdot A=\frac{1-\gamma}{\gamma}\prn{\ND{F_{\Delta_2}(\gamma)}-\unitvec}\setminus \prn{\frac{1-\gamma}{\gamma}F'_\text{sub}(\gamma)}$ is non-empty. Then \cref{lem:expect-superior} implies that for all $\gamma\in(0,1)$, \cref{eq:leq-most-apply-transient} is strict for all $\Diid\in\DSetBCiid$ and $\E{s_1\sim \Delta_1}{\pwr[s_1,\gamma][\Dbd]}\not\geqMost[][\DSetBd] \E{s_2\sim \Delta_2}{\pwr[s_2,\gamma][\Dbd]}$. 

We show that this result's preconditions holding for  $F_{\Delta_i}^*$ implies the $F_{\Delta_i}$ preconditions. Suppose $F_{\Delta_i}^*\defeq \set{\E{s_i\sim \Delta_i}{\fpi{s_i}}\mid \pi\in\Pi}$ for $i=1,2$ are such that $F_\text{sub}^*\defeq \phi\cdot \ND{F_{\Delta_1}^*}\subseteq F_{\Delta_2}^*$. In the following, the $\Delta_i$ are represented as vectors in $\rewardVS$, and $\gamma$ is a variable.
\begin{align}
    \phi\cdot \set{\gamma\f\mid \f \in \ND{F_{\Delta_1}}}&=\phi\cdot \prn{\ND{F_{\Delta_1}^*-\Delta_1}}\\
    &=\phi\cdot\prn{\ND{F_{\Delta_1}^*}-\Delta_1}\label{eq:invariant-vec-fn-set}\\
    &=\set{\permute \f - \permute\Delta_1\mid \f\in \ND{F_{\Delta_1}^*}}\\
    &\subseteq \set{\f-\Delta_2 \mid \f\in F_{\Delta_2}^*}\label{eq:fsub-no-state}\\
    &=\set{\gamma\f\mid \f \in F_{\Delta_2}}.\label{eq:gamma-sim}
\end{align}

\Cref{eq:invariant-vec-fn-set} follows from \cref{lem:pos-aff-nd-invar} \cref{item:invar-vector-fns}. Since we assumed that $\phi\cdot \ND{F_{\Delta_1}^*}\subseteq F_{\Delta_2}^*$, $\phi\cdot \set{\Delta_1}=\phi\cdot\prn{\ND{F_{\Delta_1}^*}(0)}\subseteq F_{\Delta_2}^*(0) =\set{\Delta_2}$. This implies that $\permute\Delta_1=\Delta_2$ and so \cref{eq:fsub-no-state} follows. 

\Cref{eq:gamma-sim} shows that $\phi\cdot \set{\gamma\f\mid \f \in \ND{F_{\Delta_1}}}\subseteq \set{\gamma\f\mid \f \in F_{\Delta_2}}$. But we then have $\phi\cdot \set{\gamma\f\mid \f \in \ND{F_{\Delta_1}}}\defeq \set{\gamma\permute \f \mid \f \in \ND{F_{\Delta_1}}}=\set{\gamma\f \mid \f \in \phi\cdot \ND{F_{\Delta_1}}}\subseteq \set{\gamma\f\mid \f \in F_{\Delta_2}}$. Thus,  $\phi\cdot \ND{F_{\Delta_1}}\subseteq F_{\Delta_2}$. 

Suppose $\ND{F_{\Delta_2}^*}\setminus \phi\cdot \ND{F_{\Delta_1}^*}$ is non-empty, which implies that
\begin{align}
    \phi\cdot \set{\gamma\f\mid \f \in \ND{F_{\Delta_1}}}&=\set{\permute \f - \permute\Delta_1\mid \f\in \ND{F_{\Delta_1}^*}}\\
    &=\set{\f - \permute\Delta_1\mid \f\in \phi\cdot \ND{F_{\Delta_1}^*}}\\
    &\subsetneq \set{\f-\Delta_2 \mid \f\in \ND{F_{\Delta_2}^*}}\\
    &=\set{\gamma\f\mid \f \in \ND{F_{\Delta_2}}}.
\end{align}

Then $\ND{F_{\Delta_2}}\setminus \phi\cdot \ND{F_{\Delta_1}}$ must be non-empty. Therefore, if the preconditions of this result are met for $F_{\Delta_i}^*$, they are met for $F_{\Delta_i}$.
\end{proof}

\morePowerMoreOptions*
\begin{proof}
Let $F_\text{sub}\defeq \phi\cdot \Fnd(s')\subseteq \F(s)$. Let $\Delta_1\defeq \unitvec[s'],\Delta_2\defeq \unitvec$, and define $F_{\Delta_i}^*\defeq \set{\E{s_i\sim \Delta_i}{\fpi{s_i}}\mid \pi\in\Pi}$ for $i=1,2$. Then $\Fnd(s')=\ND{F_{\Delta_1}^*}$ is similar to $F_\text{sub}=F^*_\text{sub}\subseteq F_{\Delta_2}^*=\F(s)$ via involution $\phi$. Apply \cref{lem:more-power-prefix} to conclude that  $\forall \gamma\in[0,1]:\pwr[s',\gamma][\Dbd]\leqMost[][\DSetBd] \pwr[s,\gamma][\Dbd]$. 

Furthermore, $\Fnd(s)=\ND{F_{\Delta_2}^*}$, and $F_\text{sub}=F_\text{sub}^*$, and so if $\Fnd(s)\setminus \phi\cdot \Fnd(s')\defeq \Fnd(s)\setminus F_\text{sub} =\ND{F_{\Delta_2}^*}\setminus F_\text{sub}^*$ is non-empty, then \cref{lem:more-power-prefix} shows that for all $\gamma\in(0,1)$, the inequality is strict for all $\Diid\in\DSetBCiid$ and $\pwr[s',\gamma][\Dbd]\not\geqMost[][\DSetBd] \pwr[s,\gamma][\Dbd]$. 
\end{proof}

\begin{restatable}[Non-dominated visit distribution functions never agree with other visit distribution functions at that state]{lem}{noAgreeND}\label{lem:no-agree} 
Let $\f \in \Fnd(s), \f'\in\F(s)\setminus\{\f\}$. $\forall \gamma \in (0,1):\f(\gamma) \neq \f'(\gamma)$.
\end{restatable}
\begin{proof}
Let $\gamma\in(0,1)$. Since $\f\in\Fnd(s)$, there exists a $\gamma^*\in(0,1)$ at which $\f$ is strictly optimal for some reward function. Then by \cref{transferDiscount}, we can produce another reward function for which $\f$ is strictly optimal at discount rate $\gamma$; in particular, \cref{transferDiscount} guarantees that the policies which induce $\f'$ are not optimal at $\gamma$. So $\f(\gamma)\neq \f'(\gamma)$. 
\end{proof}

\begin{restatable}[Cardinality of non-dominated visit distributions]{cor}{cardNDInter}\label{cor:card-nd-visit}
Let $F\subseteq \F(s)$. $\forall \gamma\in(0,1): \abs{F\cap \Fnd(s)}=\abs{F(\gamma)\cap \Fnd(s,\gamma)}$.
\end{restatable}
\begin{proof}
Let $\gamma\in(0,1)$. By applying \cref{lem:nd-gamma-F-subset} with $\Delta_d\defeq \unitvec$, $\f\in\Fnd(s)=\ND{\F(s)}$ iff $\f(\gamma)\in\ND{\F(s,\gamma)}$. By \cref{lem:nd-relation}, $\ND{\F(s,\gamma)}=\Fnd(s,\gamma)$. So all $\f\in F\cap \Fnd(s)$ induce $\f(\gamma)\in F(\gamma)\cap \Fnd(s,\gamma)$, and $\abs{F\cap \Fnd(s)}\geq\abs{F(\gamma)\cap \Fnd(s,\gamma)}$.

\Cref{lem:no-agree} implies that for all $\f,\f'\in\Fnd(s)$, $\f=\f'$ iff $\f(\gamma)=\f'(\gamma)$. Therefore, $\abs{F\cap \Fnd(s)}\leq\abs{F(\gamma)\cap \Fnd(s,\gamma)}$. So $\abs{F\cap \Fnd(s)}=\abs{F(\gamma)\cap \Fnd(s,\gamma)}$.
\end{proof}

\begin{restatable}[Optimality probability and state bottlenecks]{lem}{optProbBottle}\label{lem:opt-prob-bottleneck}
Suppose that $s$ can reach $\reach{s',a'}\cup\reach{s',a}$, but only by taking actions equivalent to $a'$ or $a$ at state $s'$. $F_{\text{nd},a'}\defeq \FndRestrictAction{s'}{a'}, F_a\defeq \FRestrictAction{s'}{a}$. Suppose $F_a$ contains a copy of $F_{\text{nd},a'}$ via $\phi$ which fixes all states not belonging to $\reach{s',a'}\cup\reach{s',a}$. 
Then $\forall \gamma\in[0,1]:\optprob[\Dany]{F_{\text{nd},a'},\gamma}\leqMost[][\DSetAny] \optprob[\Dany]{F_a,\gamma}$.

If $\Fnd(s)\cap \prn{F_a\setminus \phi\cdot F_{\text{nd},a'}}$ is non-empty, then for all $\gamma\in(0,1)$, the inequality is strict for all $\Diid\in\DSetBCiid$, and $\optprob[\Dany]{F_{\text{nd},a'},\gamma}\not\geqMost[][\DSetAny] \optprob[\Dany]{F_a,\gamma}$.
\end{restatable}
\begin{proof} Let $F_\text{sub}\defeq \phi\cdot F_{\text{nd},a'}$. Let $F^* \defeq \bigcup_{\substack{a''\in\A:\\ \prn{a'' \not\equiv_{s'} a} \land \prn{a'' \not \equiv_{s'}a'}}} \FRestrictAction{s'}{a''}\cup F_{\text{nd},a'} \cup F_\text{sub}$.
\begin{align}
    \phi\cdot F^*\defeq\,& \phi\cdot \prn{\bigcup_{\substack{a''\in\A:\\ \prn{a'' \not\equiv_{s'} a} \land \prn{a'' \not \equiv_{s'}a'}}} \FRestrictAction{s'}{a''}\cup F_{\text{nd},a'} \cup F_\text{sub}}\\
    =\,& \bigcup_{\substack{a''\in\A:\\ \prn{a'' \not\equiv_{s'} a} \land \prn{a'' \not \equiv_{s'}a'}}}\phi\cdot \FRestrictAction{s'}{a''}\cup \prn{\phi\cdot F_{\text{nd},a'}} \cup \prn{\phi\cdot F_\text{sub}}\\
    =\,& \bigcup_{\substack{a''\in\A:\\ \prn{a'' \not\equiv_{s'} a} \land \prn{a'' \not \equiv_{s'}a'}}}\phi\cdot \FRestrictAction{s'}{a''}\cup F_\text{sub} \cup F_{\text{nd},a'}\label{eq:invert-Fsub}\\
    =\,& \bigcup_{\substack{a''\in\A:\\ \prn{a'' \not\equiv_{s'} a} \land \prn{a'' \not \equiv_{s'}a'}}} \FRestrictAction{s'}{a''}\cup F_\text{sub}\cup F_{\text{nd},a'}\label{eq:phi-prime-bigcup}\\
    \eqdef\,& F^*.
\end{align}
\Cref{eq:invert-Fsub} follows because the involution $\phi$ ensures that $\phi\cdot F_\text{sub}=F_{\text{nd},a'}$. By assumption, $\phi$ fixes all $s'\not\in \reach{s',a'}\cup\reach{s',a}$. Suppose $\f\in\F (s)\setminus\prn{F_{\text{nd},a'}\cup F_{a}}$. By the bottleneck assumption, $\f$ does not visit states in $\reach{s',a'}\cup\reach{s',a}$. Therefore, $\permute[\phi]\f=\f$, and so \cref{eq:phi-prime-bigcup} follows.

Let $F_Z \defeq \prn{\F(s)\setminus (\FRestrictAction{s}{a'}\cup F_a)} \cup F_{\text{nd},a'} \cup F_a$. By definition, $F_Z\subseteq \F(s)$. Furthermore, $\Fnd(s)=\bigcup_{\substack{a''\in\A}} \FndRestrictAction{s'}{a''}\subseteq \prn{\F(s)\setminus (\FRestrictAction{s}{a'}\cup F_a)} \cup \FndRestrictAction{s}{a'} \cup F_a\eqdef F_Z$, and so $\Fnd(s)\subseteq F_Z$. Note that $F^*=F_Z\setminus (F_a\setminus F_\text{sub})$.

\paragraph*{Case: $\gamma\in(0,1)$.}
\begin{align}
    \optprob[\Dany]{F_{\text{nd},a'},\gamma} &= \phelper{F_{\text{nd},a'}(\gamma)\geq \F(s,\gamma)}[\Dany]\label{eq:opt-prb-id-1}\\
    &\leqMost[][\DSetAny] \phelper{F_{a}(\gamma)\geq \F(s,\gamma)}[\Dany]\label{eq:leq-most-apply}\\
    &= \optprob[\Dany]{F_{\text{nd},a'},\gamma}.\label{eq:opt-prb-id-2}
\end{align} 

\Cref{eq:opt-prb-id-1} and \cref{eq:opt-prb-id-2} follow from \cref{lem:opt-prob-id}. \Cref{eq:leq-most-apply} follows by applying \cref{lem:opt-prob-superior} with $A\defeq F_{\text{nd},a'}(\gamma),B'\defeq F_\text{sub}(\gamma),B\defeq F_{a}(\gamma),C\defeq\F(s,\gamma),Z\defeq F_Z(\gamma)$ which satisfies $\ND{C}=\Fnd(s,\gamma)\subseteq F_Z(\gamma)\subseteq \F(s,\gamma)=C$, and involution $\phi$ which satisfies $\phi\cdot F^*(\gamma)=\phi\cdot \prn{Z\setminus\prn{B\setminus B'}}=Z\setminus\prn{B\setminus B'}=F^*(\gamma)$. 

Suppose $\Fnd(s)\cap \prn{F_a\setminus F_\text{sub}}$ is non-empty. $0<\abs{\Fnd(s)\cap \prn{F_a\setminus F_\text{sub}}}=\abs{\Fnd(s,\gamma)\cap\prn{F_a(\gamma)\setminus F_\text{sub}(\gamma)}}\eqdef\abs{\ND{C}\cap\prn{B\setminus B'}}$ (with the first equality holding by \cref{cor:card-nd-visit}), and so $\ND{C}\cap\prn{B\setminus B'}$ is non-empty. We also have $B\defeq F_{a}(\gamma)\subseteq \F(s,\gamma)\eqdef C$. Then reapplying \cref{lem:opt-prob-superior}, \cref{eq:leq-most-apply} is strict for all $\Diid\in\DSetBCiid$, and $\optprob[\Dany]{F_{\text{nd},a'},\gamma}\not\geqMost[][\DSetAny] \optprob[\Dany]{F_a,\gamma}$.

\paragraph*{Case: $\gamma=1$, $\gamma=0$.} 
\begin{align}
    \optprob[\Dany]{F_{\text{nd},a'},1} &=\lim_{\gamma^*\to 1} \optprob[\Dany]{F_{\text{nd},a'},\gamma^*}\label{eq:lim-1-prob}\\ 
    &=\lim_{\gamma^*\to 1} \phelper{F_{\text{nd},a'}(\gamma^*)\geq \F(s,\gamma^*)}[\Dany]\label{eq:lim-1-opt-prb-id-1}\\
    &\leqMost[][\DSetAny] \lim_{\gamma^*\to 1} \phelper{F_{a}(\gamma^*)\geq \F(s,\gamma^*)}[\Dany]\label{eq:apply-lim-prob}\\
    &=\lim_{\gamma^*\to 1} \optprob[\Dany]{F_{a},\gamma^*}\label{eq:lim-1-opt-prb-id-2}\\ 
    &=\optprob[\Dany]{F_{a},1}.\label{eq:lim-1-prob-b}
\end{align}
\Cref{eq:lim-1-prob} and \cref{eq:lim-1-prob-b} hold by \cref{prop:opt-prob-converge}. \Cref{eq:lim-1-opt-prb-id-1} and \cref{eq:lim-1-opt-prb-id-2} follow by \cref{lem:opt-prob-id}. Applying \cref{lem:lim-prob-most} with $\gamma\defeq 1,I\defeq (0,1), F_A\defeq F_{\text{nd},a'}, F_B\defeq F_a, F_C\defeq \F(s)$, $F_Z$ as defined above, and involution $\phi$ (for which $\phi\cdot \prn{F_Z\setminus \prn{F_B\setminus \phi\cdot F_A}}=F_Z\setminus \prn{F_B\setminus \phi\cdot F_A}$),  we conclude that \cref{eq:apply-lim-prob} follows.

The $\gamma=0$ case proceeds similarly to $\gamma=1$.
\end{proof}

\begin{restatable}[Action optimality probability is a special case of visit distribution optimality probability]{lem}{agreeOptProb}\label{lem:agree-opt-prob} $\optprob[\Dany]{s,a,\gamma}=\optprob[\Dany]{\FRestrictAction{s}{a},\gamma}$.
\end{restatable}
\begin{proof}
Let $F_a\defeq \FRestrictAction{s}{a}$. For $\gamma \in (0,1)$,
\begin{align}
    \prob[\Dany]{s,a,\gamma}&\defeq\prob[R \sim \Dany]{\exists \pi^* \in \optPi: \pi^*(s)=a}\\
    &= \prob[\rf \sim \Dany]{\exists \fpi[\pi^*]{s} \in F_a: \fpi[\pi^*]{s}(\gamma)^\top \rf=\max_{\f \in \F(s)}\f(\gamma)^\top \rf}\label{eq:f-action-equiv}\\
    &=\optprob[\Dany]{F_a,\gamma}.\label{eq:f-action-equiv-final}
\end{align}
By \cref{lem:opt-pol-visit-iff}, if $\exists\pi^*\in\optPi:\pi^*(s)=a$, then it induces some optimal $\fpi[\pi^*]{s}\in F_a$. Conversely, if $\fpi[\pi^*]{s}\in F_a$ is optimal at $\gamma\in(0,1)$, then $\pi^*$ chooses optimal actions on the support of $\fpi[\pi^*]{s}(\gamma)$. Let $\pi'$ agree with $\pi^*$ on that support and let $\pi'$ take optimal actions at all other states. Then $\pi'\in\optPi$ and $\pi'(s)=a$. So \cref{eq:f-action-equiv} follows.

Suppose $\gamma=0$ or $\gamma=1$. Consider any sequence $\prn{\gamma_n}_{n=1}^\infty$ converging to $\gamma$, and let $\Dany$ induce probability measure $F$.
\begin{align}
    \optprob[\Dany]{F_a,\gamma}&\defeq \lim_{\gamma^*\to \gamma} \optprob[\Dany]{F_a,\gamma^*}\\
    &=\lim_{\gamma^*\to \gamma}\prob[R \sim \Dany]{\exists \pi^* \in \optPi[R,\gamma^*]: \pi^*(s)=a}\label{eq:inner-eq-act}\\
    &=\lim_{n\to\infty}\prob[R \sim \Dany]{\exists \pi^* \in \optPi[R,\gamma_n]: \pi^*(s)=a}\\
    &=\lim_{n\to\infty}\int_{\rewardSpace} \indic{\exists \pi^* \in \optPi[R,\gamma_n]: \pi^*(s)=a} \dF[R]\\
    &=\int_{\rewardSpace}\lim_{n\to\infty} \indic{\exists \pi^* \in \optPi[R,\gamma_n]: \pi^*(s)=a} \dF[R]\label{eq:dominated-convergence}\\
    &=\int_{\rewardSpace} \indic{\exists \pi^* \in \optPi[R,\gamma]: \pi^*(s)=a} \dF[R]\\
    &\eqdef \prob[\Dany]{s,a,\gamma}.
\end{align}

\Cref{eq:inner-eq-act} follows by \cref{eq:f-action-equiv-final}. for $\gamma^*\in[0,1]$, let $f_{\gamma^*}(R)\defeq \indic{\exists \pi^* \in \optPi[R,\gamma^*]: \pi^*(s)=a}$. For each $R\in\rewardSpace$, \cref{lem:opt-pol-shift-bound} exists $\gamma_x\approx \gamma$ such that for all intermediate $\gamma_x'$ between $\gamma_x$ and $\gamma$, $ \optPi[R,\gamma_x']=\optPi[R,\gamma]$. Since $\gamma_n \to \gamma$, this means that $\prn{f_{\gamma_n}}_{n=1}^\infty$ converges pointwise to $f_\gamma$. Furthermore, $\forall n \in \mathbb{N}, R\in\rewardSpace: \abs{f_{\gamma_n}(R)}\leq 1$ by definition. Therefore, \cref{eq:dominated-convergence} follows by Lebesgue's dominated convergence theorem.
\end{proof}

\graphOptions*
\begin{proof}
Note that by \cref{def:visit}, $F_{a'}(0)=\set{\unitvec}=F_a(0)$. Since $\phi\cdot F_{a'}\subseteq F_a$, in particular we have $\phi\cdot F_{a'}(0)=\set{\permute\unitvec}\subseteq \set{\unitvec}=F_a(0)$, and so $\phi(s)=s$.

\textbf{\Cref{item:power-options}.}
For state probability distribution $\Delta_s\in \Delta(\St)$, let $F^*_{\Delta_s}\defeq \set{\E{s'\sim \Delta_s}{\fpi{s'}}\mid \pi\in\Pi}$. Unless otherwise stated, we treat $\gamma$ as a variable in this item; we apply element-wise vector addition, constant multiplication, and variable multiplication via the conventions outlined in \cref{def:aff-transf-short}.
\begin{align}
    F_{a'} &= \set{\unitvec + \gamma \E{s_{a'}\sim T(s,a')}{\fpi{s_{a'}}}\mid \pi\in\Pi: \pi(s)=a'}\label{eq:visit-defn}\\
    &= \set{\unitvec + \gamma \E{s_{a'}\sim T(s,a')}{\fpi{s_{a'}}}\mid \pi\in\Pi}\label{eq:remove-action-constraint}\\
    &= \unitvec+\gamma F^*_{T(s,a')}.\label{eq:Fstar-contain}
\end{align}

\Cref{eq:visit-defn} follows by \cref{def:visit}, since each $\f\in\F(s)$ has an initial term of $\unitvec$. \Cref{eq:remove-action-constraint} follows because $s\not\in \reach{s,a'}$, and so for all $s_{a'}\in \supp[T(s,a')]$, $\fpi{s_{a'}}$ is unaffected by the choice of action $\pi(s)$. Note that similar reasoning implies that $F_a\subseteq \unitvec+\gamma F^*_{T(s,a)}$ (because \cref{eq:remove-action-constraint} is a containment relation in general).

Since $F_{a'}=\unitvec+\gamma F^*_{T(s,a')}$, if $F_a$ contains a copy of $F_{a'}$ via $\phi$, then $F^*_{T(s,a)}$ contains a copy of $F^*_{T(s,a')}$ via $\phi$. Then $\phi\cdot \ND{F^*_{T(s,a')}}\subseteq \phi\cdot F^*_{T(s,a')}\subseteq F^*_{T(s,a)}$, and so $F^*_{T(s,a)}$ contains a copy of $\ND{F^*_{T(s,a')}}$. Then apply \cref{lem:more-power-prefix} with $\Delta_1\defeq T(s,a')$ and $\Delta_2\defeq T(s,a)$ to conclude that $\forall\gamma\in[0,1]:\E{s_{a'}\sim T(s,a')}{\pwr[s_{a'},\gamma][\Dbd]}\leqMost[][\DSetBd] \E{s_{a}\sim T(s,a)}{\pwr[s_a,\gamma][\Dbd]}$. 

Suppose $\Fnd(s)\cap \prn{F_a\setminus \phi\cdot F_{a'}}$ is non-empty. To apply the second condition of \cref{lem:more-power-prefix}, we want to demonstrate that $\ND{F^*_{T(s,a)}}\setminus \phi\cdot \ND{F^*_{T(s,a')}}$ is also non-empty.

First consider $\f\in\Fnd(s)\cap F_a$. Because $F_a\subseteq \unitvec+\gamma F^*_{T(s,a)}$, we have that $\gamma\inv(\f - \unitvec)\in F^*_{T(s,a)}$. Because $\f \in \Fnd(s)$, by \cref{def:nd}, $\exists \rf \in \rewardVS, \gamma_x\in (0,1)$ such that
\begin{align}
    \f(\gamma_x)^\top\rf  &> \max_{\f'\in\F(s)\setminus \set{\f}}\f'(\gamma_x)^\top \rf.\label{eq:f-nd-def}
\end{align}
Then since $\gamma_x\in (0,1)$,
\begin{align}
    \gamma_x\inv(\f(\gamma_x)-\unitvec)^\top\rf  &> \max_{\f'\in\F(s)\setminus \set{\f}}\gamma_x\inv(\f'(\gamma_x)-\unitvec)^\top \rf\\
    &=\max_{\f'\in\gamma_x\inv\prn{(\F(s)\setminus \set{\f})-\unitvec}}\f'(\gamma_x)^\top \rf\\
    &\geq \max_{\f'\in \gamma_x\inv\prn{(F_a\setminus \set{\f})-\unitvec}}\f'(\gamma_x)^\top \rf\label{eq:Fact-contain}\\
    &= \max_{\f'\in F^*_{T(s,a)}\setminus \set{\gamma_x\inv(\f-\unitvec)}}\f'(\gamma_x)^\top \rf.\label{eq:FTsa}
\end{align}

\Cref{eq:Fact-contain} holds because $F_a\subseteq \F(s)$. By assumption, action $a$ is optimal for $\rf$ at state $s$ and at discount rate $\gamma_x$. \Cref{eq:remove-action-constraint} shows that $F^*_{T(s,a)}$ potentially allows the agent a non-stationary policy choice at $s$, but  non-stationary policies cannot increase optimal value \citep{puterman_markov_2014}. Therefore, \cref{eq:FTsa} holds.

We assumed that $\gamma\inv (\f-\unitvec)\in \gamma\inv(\Fnd(s)-\unitvec)$. Furthermore, since we just showed that $\gamma\inv (\f-\unitvec)\in F^*_{T(s,a)}$ is strictly optimal over the other elements of $F^*_{T(s,a)}$ for reward function $\rf$ at discount rate $\gamma_x\in(0,1)$, we conclude that it is an element of $\ND{F^*_{T(s,a)}}$ by \cref{def:nd-vec-func}. Then we conclude that $\gamma\inv(\Fnd(s)-\unitvec)\cap F^*_{T(s,a)}\subseteq \ND{F^*_{T(s,a)}}$. 

We now show that $\ND{F^*_{T(s,a)}}\setminus \phi\cdot \ND{F^*_{T(s,a')}}$ is non-empty. 
\begin{align}
    0&<\abs{\Fnd(s)\cap \prn{F_a\setminus \phi\cdot F_{a'}}}\label{eq:nonempty-fnd-cap}\\
    &=\abs{\gamma\inv\prn{\Fnd(s)\cap \prn{F_a\setminus \phi\cdot F_{a'}}-\unitvec}}\label{eq:card-preserve}\\
    &\leq\abs{\gamma\inv\prn{\Fnd(s)-\unitvec}\cap \prn{F^*_{T(s,a)}\setminus\phi\cdot F^*_{T(s,a')}}}\label{eq:Fstar-convert}\\
    &=\abs{\prn{\gamma\inv\prn{\Fnd(s)-\unitvec}\cap F^*_{T(s,a)}}\setminus\phi\cdot F^*_{T(s,a')}}\\
    &\leq\abs{\ND{F^*_{T(s,a)}}\setminus\phi\cdot F^*_{T(s,a')}}\label{eq:contain-nd-fstar}\\
    &\leq\abs{\ND{F^*_{T(s,a)}}\setminus \phi\cdot \ND{F^*_{T(s,a')}}}.\label{eq:nd-containment-trivial}
\end{align}
\Cref{eq:nonempty-fnd-cap} follows by the assumption that $\Fnd(s)\cap \prn{F_a\setminus \phi\cdot F_{a'}}$ is non-empty. Let $\f,\f'\in \Fnd(s)\cap \prn{F_a\setminus \phi\cdot F_{a'}}$ be distinct. Then we must have that for some $\gamma_x\in (0,1)$, $\f(\gamma_x)\neq \f'(\gamma_x)$. This holds iff $\gamma_x\inv(\f(\gamma_x)-\unitvec)\neq \gamma_x\inv(\f'(\gamma_x)-\unitvec)$, and so \cref{eq:card-preserve} holds.

\Cref{eq:Fstar-convert} holds because $F_a\subseteq \unitvec+\gamma F^*_{T(s,a)}$ and $F_a'= \unitvec+\gamma F^*_{T(s,a')}$ by \cref{eq:Fstar-contain}. \Cref{eq:contain-nd-fstar} holds because we showed above that $\gamma\inv(\Fnd(s)-\unitvec)\cap F^*_{T(s,a)}\subseteq \ND{F^*_{T(s,a)}}$. \Cref{eq:nd-containment-trivial} holds because $\ND{F^*_{T(s,a')}}\subseteq F^*_{T(s,a')}$ by \cref{def:nd-vec-func}. 

Therefore, $\ND{F^*_{T(s,a)}}\setminus \phi\cdot \ND{F^*_{T(s,a')}}$ is non-empty, and so apply the second condition of \cref{lem:more-power-prefix} to conclude that for all $\Diid\in\DSetBCiid$, $\forall\gamma\in(0,1):\E{s_{a'}\sim T(s,a')}{\pwr[s_{a'},\gamma][\Diid]}< \E{s_{a}\sim T(s,a)}{\pwr[s_a,\gamma][\Diid]}$, and that $\forall\gamma\in(0,1):\E{s_{a'}\sim T(s,a')}{\pwr[s_{a'},\gamma][\Dbd]}\not \geqMost[][\DSetBd] \E{s_{a}\sim T(s,a)}{\pwr[s_a,\gamma][\Dbd]}$.

\textbf{\Cref{item:opt-prob-options}.} 
Let $\phi'(s_x)\defeq \phi(s_x)$ when $s_x\in\reach{s,a'}\cup\reach{s,a}$, and equal $s_x$ otherwise. Since $\phi$ is an involution, so is $\phi'$. 
\begin{align}
    \phi'\cdot F_{a'}&\defeq \set{\permute[\phi']\prn{\unitvec + \gamma\E{s_{a'}\sim T(s,a')}{ \fpi{s_{a'}}}}\mid \pi\in \Pi, \pi(s)=a'}\\
    &= \set{\unitvec+\gamma\E{s_{a'}\sim T(s,a')}{\permute[\phi'] \fpi{s_{a'}}}\mid \pi\in \Pi, \pi(s)=a'}\label{eq:fix-unitvec}\\
    &= \set{\permute\unitvec+ \gamma\E{s_{a'}\sim T(s,a')}{\permute[\phi] \fpi{s_{a'}}}\mid \pi\in \Pi, \pi(s)=a'}\label{eq:phi-prime-eq}\\
    &\eqdef \phi\cdot F_{a'}\\
    &\subseteq F_a.\label{eq:sim-subtracted}
\end{align}

\Cref{eq:fix-unitvec} follows because if $s\in\reach{s,a'}\cup\reach{s,a}$, then we already showed that $\phi$ fixes $s$. Otherwise, $\phi'(s)=s$ by definition. \Cref{eq:phi-prime-eq} follows by the definition of $\phi'$ on $\reach{s,a'}\cup\reach{s,a}$ and because $\unitvec=\permute\unitvec$. Next, we assumed that $\phi\cdot F_{a'}\subseteq F_a$, and so \cref{eq:sim-subtracted} holds.

Therefore, $F_a$ contains a copy of $F_{a'}$ via $\phi'$ fixing all $s_x\not\in\reach{s,a'}\cup\reach{s,a}$. Therefore, $F_a$ contains a copy of $F_{\text{nd},a'}\defeq \Fnd(s)\cap F_{a'}$ via the same $\phi'$. Then apply \cref{lem:opt-prob-bottleneck} with $s'\defeq s$ to conclude that $\forall \gamma\in[0,1]:\optprob[\Dany]{F_{a'},\gamma}\leqMost[][\DSetAny] \optprob[\Dany]{F_a,\gamma}$. By \cref{lem:agree-opt-prob}, $\optprob[\Dany]{s,a',\gamma}=\optprob[\Dany]{F_{a'},\gamma}$ and $\optprob[\Dany]{s,a,\gamma}=\optprob[\Dany]{F_a,\gamma}$. Therefore, $\forall \gamma\in[0,1]:\optprob[\Dany]{s,a',\gamma}\leqMost[][\DSetAny] \optprob[\Dany]{s,a,\gamma}$.

If $\Fnd(s)\cap \prn{F_a\setminus \phi\cdot F_{a'}}$ is non-empty, then apply the second condition of \cref{lem:opt-prob-bottleneck} to conclude that for all $\gamma\in(0,1)$, the inequality is strict for all $\Diid\in\DSetBCiid$, and $\optprob[\Dany]{s,a',\gamma}\not\geqMost[][\DSetAny] \optprob[\Dany]{s,a,\gamma}$.
\end{proof}

\subsubsection{When \texorpdfstring{$\gamma=1$}{reward is undiscounted}, optimal policies tend to navigate towards ``larger'' sets of cycles} \label{sec:preservation}

\begin{restatable}[$\pwrNoDist$ identity when $\gamma=1$]{lem}{gammaOnePower}\label{lem:gamma-1-power}
\begin{equation}
\pwr[s,1][\Dbd]=\E{\rf \sim \Dbd}{\max_{\dbf \in \RSD} \dbf^\top \rf}=\E{\rf \sim \Dbd}{\max_{\dbf \in \RSDnd} \dbf^\top \rf}.
\end{equation}
\end{restatable}
\begin{proof}
\begin{align}
    \pwr[s,1][\Dbd]&= \E{\rf\sim\Dbd}{\max_{\fpi{s}\in \F(s)} \lim_{\gamma\to 1} \frac{1-\gamma}{\gamma}\prn{\fpi{s}(\gamma)-\unitvec}^\top \rf} \label{eq:pwr-lim-rsd}\\
 &=\E{\rf \sim \Dbd}{\max_{\dbf \in \RSD} \dbf^\top \rf}\label{eq:def-rsd}\\
 &=\E{\rf \sim \Dbd}{\max_{\dbf \in \RSDnd} \dbf^\top \rf}.\label{eq:nd-restrict-rsd}
\end{align}
\Cref{eq:pwr-lim-rsd} follows by \cref{lem:pwr-limit}. \Cref{eq:def-rsd} follows by the definition of $\RSD$ (\cref{def:rsd}). \Cref{eq:nd-restrict-rsd} follows because for all $\rf \in \rewardVS$, \cref{cor:nd-func-indif} shows that $\max_{\dbf \in \RSD} \dbf^\top \rf=\max_{\dbf \in \ND{\RSD}} \dbf^\top \rf \eqdef\max_{\dbf \in \RSDnd} \dbf^\top \rf$.
\end{proof}

\RSDSimPower* 
\begin{proof} Suppose $\RSDnd[s']$ is similar to $D\subseteq\RSD$ via involution $\phi$. 
\begin{align}
    \pwr[s',1]&=\E{\rf \sim \Dbd}{\max_{\dbf \in \RSDnd[s']} \dbf^\top \rf}\label{eq:pwr-id-rsd}\\
    &\leqMost[][\DSetBd]\E{\rf \sim \Dbd}{\max_{\dbf \in \RSDnd} \dbf^\top \rf}\label{eq:rsd-func-incr}
    \\
    &=\pwr[s,1][\Dbd]\label{eq:pwr-id-rsd-2}
\end{align} 
\Cref{eq:pwr-id-rsd} and \cref{eq:pwr-id-rsd-2} follow from \cref{lem:gamma-1-power}. By applying \cref{lem:expect-superior} with $A\defeq \RSD[s'],B'\defeq D, B\defeq\RSD$ and $g$ the identity function, \cref{eq:rsd-func-incr} follows. 

Suppose $\RSDnd\setminus D$ is non-empty. By the same result, \cref{eq:rsd-func-incr} is a strict inequality for all $\Diid\in\DSetBCiid$, and we conclude that $\pwr[s',1][\Dbd]\not \geqMost[][\DSetBd] \pwr[s,1][\Dbd]$. 
\end{proof}

\rsdIC*
\begin{proof}
Let $D_\text{sub}\defeq\phi\cdot D'$, where $D_\text{sub}\subseteq D$ by assumption. Let $X\defeq \set{s_i\in \St \mid \max_{\dbf \in D'\cup D} \dbf^\top \unitvec[s_i]>0}$. Define
\begin{equation}
    \phi'(s_i)\defeq 
    \begin{cases}
    \phi(s_i) & \text{ if } s_i\in X\label{eq:phi-prime-def}\\
    s_i &\text{ else}.
    \end{cases}
\end{equation}

Since $\phi$ is an involution, $\phi'$ is also an involution. Furthermore, by the definition of $X$, $\phi'\cdot D'=D_\text{sub}$ and $\phi'\cdot D_\text{sub}=D'$ (because we assumed that both equalities hold for $\phi$). 

Let $D^*\defeq D'\cup D_\text{sub} \cup \prn{\RSDnd\setminus (D' \cup D)}$.
\begin{align}
    \phi'\cdot D^*&\defeq \phi'\cdot \prn{D'\cup D_\text{sub} \cup \prn{\RSDnd\setminus (D' \cup D)}} \\
    &=\prn{\phi'\cdot D'}\cup \prn{\phi'\cdot D_\text{sub}} \cup \phi'\cdot\prn{\RSDnd\setminus (D' \cup D)} \\
    &= D_\text{sub} \cup D' \cup \prn{\RSDnd\setminus (D' \cup D)}\label{eq:sim-phi-prime-rsd}\\
    &\eqdef D^*.\label{eq:fixing-rsd-perm}
\end{align}

In \cref{eq:sim-phi-prime-rsd}, we know that $\phi'\cdot D'=D_\text{sub}$ and $\phi'\cdot D_\text{sub}=D'$. We just need to show that $\phi'\cdot \prn{\RSDnd\setminus (D' \cup D)}=\RSDnd\setminus (D' \cup D)$. 

Suppose $\exists s_i\in X, \dbf' \in \RSDnd\setminus (D' \cup D):\dbf'^\top \unitvec[s_i]>0$. By the definition of $X$, $\exists \dbf \in D' \cup D:\dbf^\top \unitvec[s_i]>0$. Then 
\begin{align}
    \dbf^\top \dbf'&=\sum_{j=1}^{\abs{\St}} \dbf^\top (\dbf'\odot\unitvec[s_j])\label{eq:dot-had}\\
    &\geq \dbf^\top (\dbf'\odot\unitvec[s_i])\label{eq:non-neg-rsd-entry}\\
    &=\dbf^\top \prn{(\dbf'^\top \unitvec[s_i])\unitvec[s_i]}\\
    &= (\dbf'^\top \unitvec[s_i]) \cdot (\dbf^\top\unitvec[s_i])\\
    &>0.\label{eq:pos-rsd-agree}
\end{align}

\Cref{eq:dot-had} follows from the definitions of the dot and Hadamard products. \Cref{eq:non-neg-rsd-entry} follows because $\dbf$ and $\dbf'$ have non-negative entries. \Cref{eq:pos-rsd-agree} follows because $\dbf^\top \unitvec[s_i]$ and $\dbf'^\top \unitvec[s_i]$ are both positive. But \cref{eq:pos-rsd-agree} shows that $\dbf^\top\dbf'>0$, contradicting our assumption that $\dbf$ and $\dbf'$ are orthogonal. 

Therefore, such an $s_i$ cannot exist, and $X'\defeq\set{s_i' \in \St \mid \max_{\dbf' \in \RSDnd\setminus (D' \cup D)} \dbf'^\top \unitvec[s_i]>0}\subseteq (\St \setminus X)$. By \cref{eq:phi-prime-def}, $\forall s_i'\in X':\phi'(s_i')=s_i'$. Thus, $\phi'\cdot \prn{\RSDnd\setminus (D' \cup D)}=\RSDnd\setminus (D' \cup D)$, and \cref{eq:sim-phi-prime-rsd} follows. We conclude that $\phi'\cdot D^*=D^*$.

Consider $Z\defeq \prn{\RSDnd\setminus (D' \cup D)} \cup D \cup D'$. First, $Z\subseteq \RSD$ by definition. Second, $\RSDnd = \RSDnd\setminus (D' \cup D) \cup (\RSDnd\cap D') \cup (\RSDnd \cap D) \subseteq Z$. Note that $D^*=Z\setminus (D \setminus D_\text{sub})$.
\begin{align} 
\avgprob[\Dany]{D'}&= \phelper{D'\geq \RSD}[\Dany]\\ 
&\leqMost[][\DSetAny] \phelper{D\geq \RSD}[\Dany]\label{eq:leq-most-1}\\
&= \avgprob[\Dany]{D}.
\end{align} 

Since $\phi\cdot D'\subseteq D$ and $\ND{D'}\subseteq D'$, $\phi\cdot \ND{D'}\subseteq D$. Then \cref{eq:leq-most-1} holds by applying \cref{lem:opt-prob-superior} with $A\defeq D', B'\defeq D_\text{sub}, B\defeq D, C\defeq \RSD$, and the previously defined $Z$ which we showed satisfies $\ND{C}\subseteq Z\subseteq C$. Furthermore, involution $\phi'$ satisfies $\phi'\cdot  B^* = \phi'\cdot \prn{Z\setminus (B \setminus B')}=Z\setminus (B \setminus B') = B^*$ by \cref{eq:fixing-rsd-perm}.

When $\RSDnd\cap\prn{D\setminus D_\text{sub}}$ is non-empty, since $B'\subseteq C$ by assumption, \cref{lem:opt-prob-superior} also shows that \cref{eq:leq-most-1} is strict for all $\Diid\in\DSetBCiid$, and that $\avgprob[\Dany]{D'}\not\geqMost[][\DSetAny] \avgprob[\Dany]{D}$.
\end{proof} 

\begin{restatable}[{\rsd[R]} properties]{prop}{rsdProp}\label{prop:rsd-properties}
Let $\dbf\in\RSD$. $\dbf$ is element-wise non-negative and $\lone{\dbf}=1$.
\end{restatable}
\begin{proof}
$\dbf$ has non-negative elements because it equals the limit of $\lim_{\gamma\to1}(1-\gamma)\f(\gamma)$, whose elements are non-negative by \cref{prop:visit-dist-prop} \cref{item:mono-increase}.
\begin{align}
    \lone{\dbf}&=\lone{\lim_{\gamma\to1}(1-\gamma)\f(\gamma)}\label{eq:lone-rsd}\\
    &=\lim_{\gamma\to 1} (1-\gamma)\lone{\f(\gamma)}\label{eq:lone-norm-rsd}\\
    &= 1.\label{eq:lone-norm-rsd-1}
\end{align}
\Cref{eq:lone-rsd} follows because the definition of {\rsd}s (\cref{def:rsd}) ensures that $\exists \f\in\F(s):\lim_{\gamma\to1}(1-\gamma)\f(\gamma)=\dbf$. \Cref{eq:lone-norm-rsd} follows because $\lone{\cdot}$ is a continuous function. \Cref{eq:lone-norm-rsd-1} follows because $\lone{\f(\gamma)}=\geom$ by \cref{prop:visit-dist-prop} \cref{item:lone-visit}.
\end{proof}

\begin{restatable}[When reachable with probability 1, {\stateEnd}s induce non-dominated {\rsd}s]{lem}{termNDrsd}\label{lem:term-nd-rsd}
If $\unitvec[s']\in\RSD$, then $\unitvec[s']\in\RSDnd$.
\end{restatable}
\begin{proof}
If $\dbf\in\RSD$ is distinct from $\unitvec[s']$, then $\lone{\dbf}=1$ and $\dbf$ has non-negative entries by \cref{prop:rsd-properties}. Since $\dbf$ is distinct from $\unitvec[s']$, then its entry for index $s'$ must be strictly less than 1: $\dbf^\top\unitvec[s']<1=\unitvec[s']^\top\unitvec[s']$. Therefore, $\unitvec[s']\in\RSD$ is strictly optimal for the \emph{reward function} $\rf\defeq \unitvec[s']$, and so $\unitvec[s']\in\RSDnd$.
\end{proof}

\avgAvoidTerminal*
\begin{proof}
Suppose $\unitvec[s_x],\unitvec[s']\in\RSD$ are distinct. Let $\phi\defeq (s_x \,\,\, s'),D'\defeq \set{\unitvec[s_x]},D\defeq \RSD\setminus\set{\unitvec[s_x]}$. $\phi\cdot D'=\set{\unitvec[s']}\subseteq \RSD\setminus\set{\unitvec[s_x]}\eqdef D$ since $s_x\neq s'$. $D'\cup D=\RSD$ and $\RSDnd\setminus(D'\cup D)=\RSDnd\setminus\RSD=\emptyset$ trivially have pairwise orthogonal vector elements. Then apply \cref{rsdIC} to conclude that $\avgprob[\Dany]{\{\unitvec[s_x]\}}\leqMost[][\DSetAny] \avgprob[\Dany]{\RSD\setminus\{\unitvec[s_x]\}}$.

Suppose there exists another $\unitvec[s'']\in\RSD$. By \cref{lem:term-nd-rsd}, $\unitvec[s'']\in\RSDnd$. Furthermore, since $s''\not \in \set{s',s_x}$ , $\unitvec[s'']\in \prn{\RSD\setminus\set{\unitvec[s_x]}}\setminus \set{\unitvec[s']}=D\setminus \phi\cdot D'$. Therefore, $\unitvec[s'']\in\RSDnd\cap\prn{D\setminus \phi \cdot D'}$. Then apply the second condition of \cref{rsdIC} to conclude that $\avgprob[\Dany]{\{\unitvec[s_x]\}}\not\geqMost[][\DSetBd] \avgprob[\Dany]{\RSD\setminus\{\unitvec[s_x]\}}$. 
\end{proof}